\documentclass{article} 
\usepackage{iclr2025_conference,times}

\usepackage{amsmath,amsfonts,bm}

\def\eqref#1{equation~\ref{#1}}

\def\1{\bm{1}}

\def\rvx{{\mathbf{x}}}

\def\rvz{{\mathbf{z}}}

\DeclareMathAlphabet{\mathsfit}{\encodingdefault}{\sfdefault}{m}{sl}
\SetMathAlphabet{\mathsfit}{bold}{\encodingdefault}{\sfdefault}{bx}{n}

\usepackage{hyperref}
\usepackage{url}
\usepackage{booktabs}       
\usepackage{amsfonts}       
\usepackage{nicefrac}       
\usepackage{microtype}      
\usepackage{xcolor}

\usepackage{amsthm}
\usepackage{thmtools}
\usepackage{thm-restate}

\usepackage{adjustbox}
\usepackage{multirow}

\usepackage{amsmath}
\usepackage{amssymb}
\usepackage{pifont}
\newcommand{\cmark}{\ding{51}}%
\newcommand{\xmark}{\ding{55}}%
\usepackage{arydshln}
\usepackage{colortbl}

\usepackage{algpseudocode}
\usepackage{graphicx}
\usepackage{subcaption}
\usepackage{wrapfig}
\usepackage{float}
\usepackage{makecell}  
\usepackage[capitalize]{cleveref}
\usepackage[ruled,vlined]{algorithm2e}

\crefname{algorithm}{Algorithm}{Algorithms}
\Crefname{algorithm}{Algorithm}{Algorithms}
\usepackage{cancel}

\def\rvx{{\mathbf{x}}}
\def\rvz{{\mathbf{z}}}

\SetKwInput{KwInput}{Input}
\SetKwInput{KwOutput}{Output}

\title{Diffusion Bridge AutoEncoders \\for Unsupervised Representation Learning}
\author{Yeongmin Kim$^1$,\,  Kwanghyeon Lee$^1$,\,  Minsang Park$^1$,\, Byeonghu Na$^1$,\, Il-Chul Moon$^{1,2}$ \\
$^1$Korea Advanced Institute of Science and Technology (KAIST), $^2$summary.ai\\
\texttt{\{alsdudrla10,rhkdgus0414,pagemu,byeonghu.na,icmoon\}@kaist.ac.kr}
}

\iclrfinalcopy 
\begin{document}

\maketitle

\begin{abstract}
Diffusion-based representation learning has achieved substantial attention due to its promising capabilities in latent representation and sample generation. Recent studies have employed an auxiliary encoder to extract a corresponding representation from data and adjust the dimensionality of a latent variable $\rvz$. Meanwhile, this auxiliary structure invokes an \textit{information split problem}; the information of each data instance $\rvx_0$ is divided into diffusion endpoint $\rvx_T$ and encoded $\rvz$ because there exist two inference paths starting from the data. The latent variable modeled by the diffusion endpoint $\rvx_T$ has several disadvantages. The diffusion endpoint $\rvx_T$ is computationally expensive to obtain and inflexible in terms of dimensionality. To address this problem, we introduce Diffusion Bridge AutoEncoders (DBAE), which enables $\rvz$-dependent endpoint $\rvx_T$ inference through a feed-forward architecture. This structure creates an information bottleneck at $\rvz$, ensuring that $\rvx_T$ depends on $\rvz$ during its generation. This results in $\rvz$ holding the full information of the data. We propose an objective function for DBAE to enable both reconstruction and generative modeling, with theoretical justification. Empirical evidence demonstrates the effectiveness of the intended design in DBAE, which notably enhances downstream inference quality, reconstruction, and disentanglement. Additionally, DBAE generates high-fidelity samples in an unconditional generation. Our code is available at \url{https://github.com/aailab-kaist/DBAE}.

\end{abstract}

\section{Introduction}
Unsupervised representation learning is a fundamental topic within the latent variable generative models~\citep{hinton2006fast,KingmaW13,higgins2017betavae,chen2016infogan,jeff2017adversarial,alemi2018fixing}. Effective representation supports better downstream inference as well as realistic data synthesis. 
Variational autoencoders (VAEs)~\citep{KingmaW13} are frequently used because they inherently include latent representations with flexible dimensionality. Generative adversarial networks (GANs)~\citep{goodfellow2014generative} with inversion~\citep{abdal2019image2stylegan,abdal2020image2stylegan++} are another method to find latent representations. Additionally, diffusion probabilistic models (DPMs)~\citep{ho2020denoising,song2021scorebased} have achieved state-of-the-art performance in terms of generation quality~\citep{dhariwal2021diffusion}, naturally prompting efforts to explore unsupervised representation learning within the DPM framework~\citep{preechakul2022diffusion,zhang2022unsupervised,yue2024exploring}, which have recently dominated generative representation learning studies.

DPMs are a type of latent variable generative model, but inference on latent variables is not straightforward. DPMs progressively map from data $\rvx_0$ to a latent endpoint $\rvx_T$ via a predefined noise injection schedule, which does not facilitate learnable encoding. DDIM~\citep{song2021denoising} introduces an ODE-based deterministic encoding from the data $\rvx_0$ to the endpoint $\rvx_T$. However, this encoding is determined by the choice of the forward process~\citep{song2021scorebased}. Since the forward process with fixed noise injection is difficult to interpret as having semantic meaning, the ODE-based encoding remains challenging to consider as an effective semantic representation. Moreover, the encoding $\rvx_0$ into $\rvx_T$ is expensive because it requires solving the ODE, and its inflexible dimensionality poses disadvantages for downstream applications~\citep{sinha2021d2c}.

To tackle this issue, recent DPM-based representation learning studies~\citep{preechakul2022diffusion,zhang2022unsupervised,kim2022unsupervised,wang2023infodiffusion,yang2023disdiff,yue2024exploring,Wu_Zheng_2024} suggest an auxiliary latent variable $\rvz$ with an encoder used in VAEs, to combine the generation performance of diffusion models and the representation learning capabilities of VAEs. The encoder-generated latent variable $\rvz$ is obtained without solving the ODE, and the encoder also facilitates the learning of semantic representations with dimensionality reduction. The reconstruction capability from the extracted latent representation $\rvz$ is the primary focus of these studies, facilitating downstream inference, attribute manipulation, and interpolation. 

This paper points out the remaining problem in auxiliary encoder models, which we refer to as the \textit{information split problem}, hindering reconstruction capability. The information is not solely retained in the latent variable $\rvz$; rather, a portion is also distributed into the latent variable $\rvx_T$ as evidenced by \Cref{fig:1b}. If the auxiliary encoder models only infer $\rvz$ and reconstruct using a random $\rvx_T$, the facial details of the original image are not properly reconstructed, indicating that the missing information is contained within $\rvx_T$. Furthermore, the inference of $\rvx_T$ is computationally expensive and inflexible in dimensionality. To address this issue, we introduce Diffusion Bridge AutoEncoders (DBAE), which incorporate $\rvz$-dependent endpoint $\rvx_T$ inference using a feed-forward architecture.

The proposed model DBAE systematically resolves the \textit{information split problem}. Unlike the two split inference paths in the previous approach in \Cref{fig:1a}, DBAE encourages $\rvz$ to become an information bottleneck during inference (dotted line in \Cref{fig:1c}), making $\rvz$ more informative. DBAE establishes this bottleneck structure by defining a learnable forward process that starts from the data $\rvx_0$ and ends at the encoded endpoint $\rvx_T$ by utilizing Doob's $h$-transform.  Moreover, DBAE does not require solving an ODE to infer endpoint $\rvx_T$, thereby making endpoint inference more efficient, as shown in \Cref{fig:1d}. This efficient inference of $\rvx_T$ benefits interpolation and attribute manipulation tasks. In experiments, DBAE outperforms the previous works in downstream inference quality, reconstruction, disentanglement, and unconditional generation. DBAE also demonstrates satisfactory results in interpolation and attribute manipulation with its qualitative advantages.

\begin{figure}[t]
     \centering
     \begin{subfigure}[b]{0.33\textwidth}
         \centering
         \includegraphics[width=\textwidth, height=0.75in]{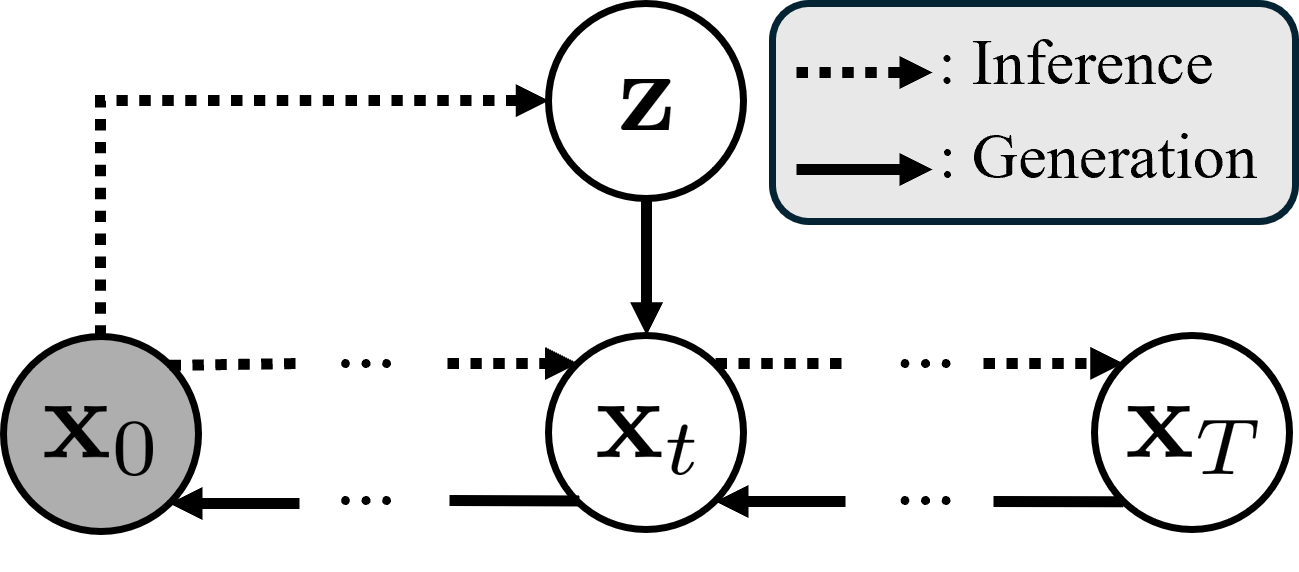}
         \caption{Bayesian network of DiffAE}
         \label{fig:1a}
     \end{subfigure}
     \hfill
     \begin{subfigure}[b]{0.65\textwidth}
         \centering
         \includegraphics[width=\textwidth, height=0.8in]{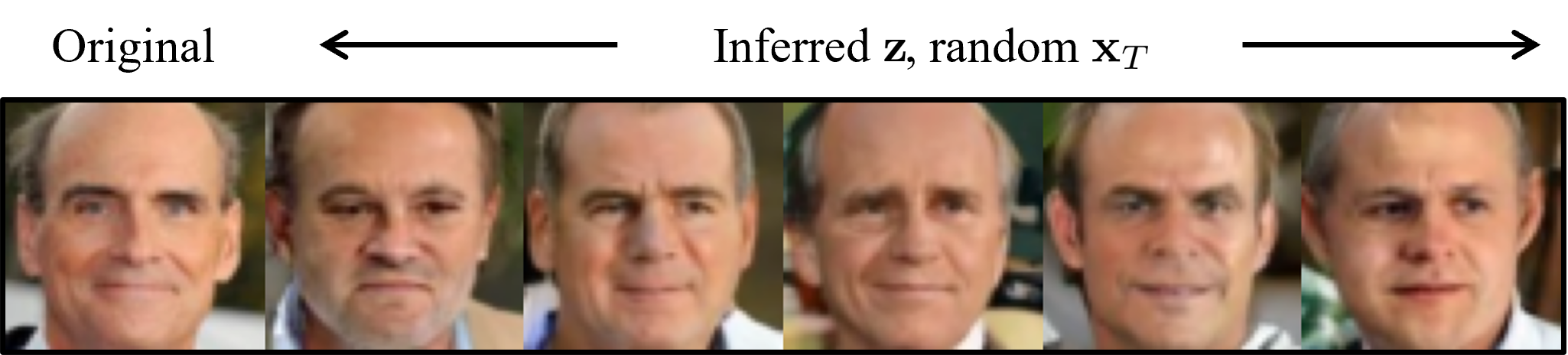}
         \caption{Reconstruction varying $\rvx_T\sim \mathcal{N}(\mathbf{0},\mathbf{I})$ in DiffAE} 
          \label{fig:1b}
     \end{subfigure}
        \label{fig:moredre}
    \begin{subfigure}[b]{0.33\textwidth}
         \centering
         \includegraphics[width=\textwidth, height=0.85in]{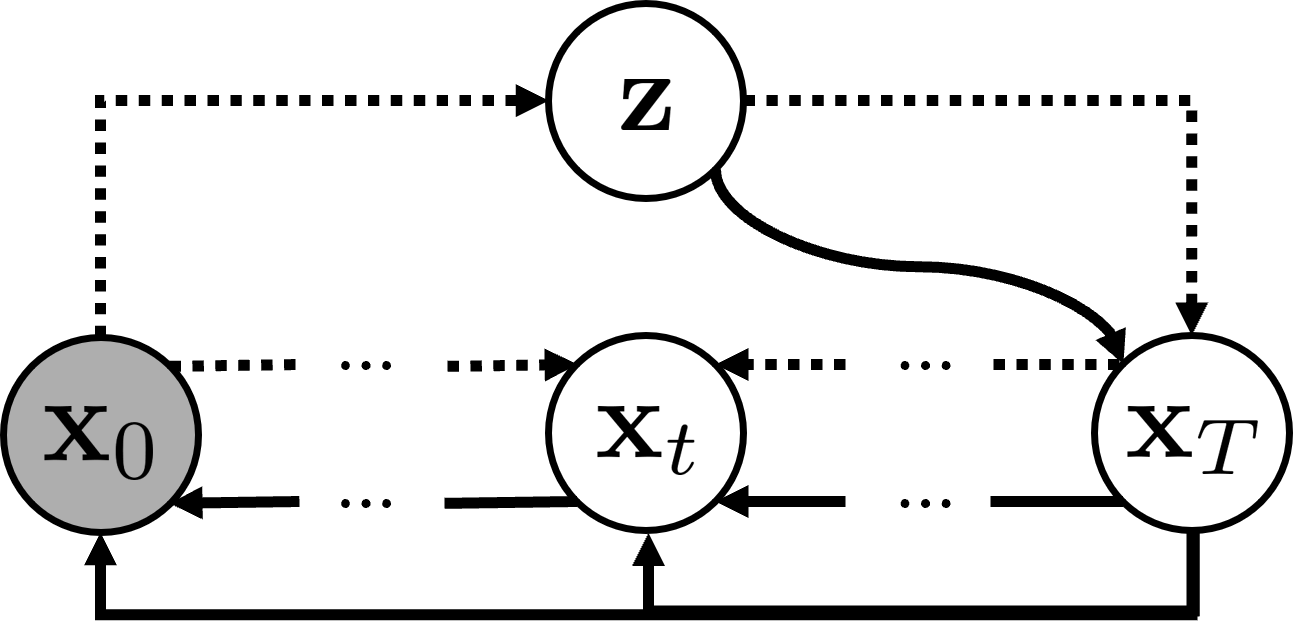}
         \caption{Bayesian network of DBAE}
          \label{fig:1c}
     \end{subfigure}
     \hfill
     \begin{subfigure}[b]{0.65\textwidth}
         \centering
         \includegraphics[width=\textwidth, height=0.8in]{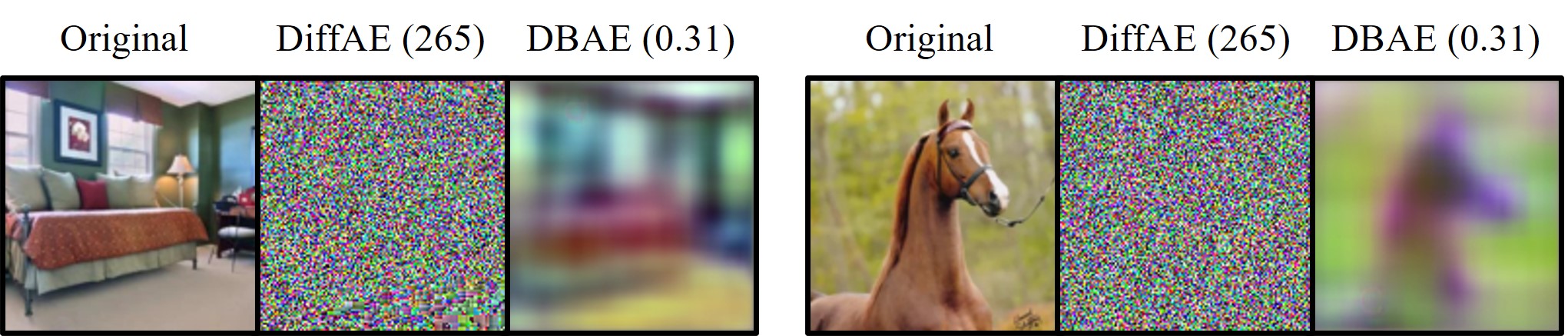}
         \caption{Inferred $\rvx_T$ from DiffAE and DBAE with inference time (ms).}
          \label{fig:1d}
     \end{subfigure}
         \caption{Comparison between DiffAE~\citep{preechakul2022diffusion} and DBAE. (a) depicts the simplified Bayesian network of DiffAE, illustrating two inference paths for the distinct latent variables $\rvx_T$ and $\rvz$. (b) shows the reconstruction using the inferred $\rvz$ in DiffAE on CelebA, where the reconstruction results perceptually vary depending on the selection of $\rvx_T$. (c) shows the simplified Bayesian network of DBAE with $\rvz$-dependent $\rvx_T$ inference. (d) shows the inferred $\rvx_T$ from DiffAE and DBAE.} 
        \label{fig:1}
        \vspace{-5mm}
\end{figure}

\section{Preliminaries}
\subsection{Diffusion Models}
\label{sec:2.1}
Diffusion probabilistic models (DPMs)~\citep{sohl2015deep,ho2020denoising} with a continuous time formulation~\citep{song2021scorebased} define a forward stochastic differential equation (SDE)
\begin{align}
& \mathrm{d}\rvx_t = \mathbf{f}(\rvx_t, t)\mathrm{d}t + g(t)\mathrm{d}\mathbf{w}_t, \quad \rvx_0 \sim q_{\textrm{data}}(\rvx_0),  \label{eq:for}
\end{align}
 where $\mathbf{w}_t$ denotes a standard Wiener process, $\mathbf{f}:\mathbb{R}^{d}\times [0,T] \rightarrow \mathbb{R}^{d}$ is a drift term, and $g:[0,T] \rightarrow \mathbb{R}$ is a volatility term. \cref{eq:for} starts from data distribution $q_{\textrm{data}}(\rvx_0)$ and gradually perturbs it into noise $\rvx_T$. Let the marginal distribution of \cref{eq:for} at time $t$ be denoted as $\tilde{q}_t(\rvx_t)$. There exists a unique reverse-time SDE~\citep{anderson1982reverse}
 \begin{align}
& \mathrm{d}\rvx_t = [\mathbf{f}(\rvx_t, t)-g^2(t)\nabla_{\rvx_t} \log \tilde{q}_t(\rvx_t)] \mathrm{d}t + g(t)\mathrm{d}\mathbf{\bar{w}}_t, \quad \rvx_T \sim p_{\textrm{prior}}(\rvx_T),   \label{eq:back}
\end{align}
where $\mathbf{\bar{w}}_t$ denotes a reverse-time Wiener process, $\nabla_{\rvx_t} \log \tilde{q}_t(\rvx_t)$ is the time-dependent score function, and $p_{\textrm{prior}}(\rvx_T)$ stands for the prior distribution, which closely resembles a Gaussian distribution with the specific form of $\mathbf{f}$ and $g$~\citep{song2021scorebased,ho2020denoising}. \cref{eq:back} traces back from noise $\rvx_T$ to data $\rvx_0$. The reverse-time ordinary differential equation (ODE) 
\begin{align}
& \mathrm{d}\rvx_t = [\mathbf{f}(\rvx_t, t)-\frac{1}{2}g^2(t)\nabla_{\rvx_t} \log \tilde{q}_t(\rvx_t)] \mathrm{d}t, \quad \rvx_T \sim p_{\textrm{prior}}(\rvx_T),\label{eq:backode}
\end{align}
produces a marginal distribution identical to \cref{eq:back} for all $t$, offering an alternative generative process while confining the stochasticity of the trajectory solely to $\rvx_T$.  
 To construct both reverse SDE and ODE, the diffusion model estimates a time-dependent score function $\nabla_{\rvx_t} \log \tilde{q}_t(\rvx_t) \approx \mathbf{s}_{\boldsymbol{\theta}}(\rvx_t,t)$ using a neural network and the score-matching objective~\citep{vincent2011connection,song2019generative}.  

\subsection{Latent Representation Learning with Diffusion Models}
\label{sec:2.2}
From the perspective of representation learning, the ODE in \cref{eq:backode} (a.k.a DDIM~\citep{song2021denoising} in discrete time diffusion formulation) provides a deterministic encoding from the data $\rvx_0$ to the latent $\rvx_T$. However, the latent representation $\rvx_T$ has some disadvantages. First, it is hard to learn its semantic meaning. This encoding is determined by the forward process $(\mathbf{f}, g)$ given a data distribution and assuming perfect optimization~\citep{song2021scorebased}. The forward process $(\mathbf{f}, g)$ is set to a fixed noise injection process, but the noise is hard to consider as a semantically meaningful encoding. Second, the dimension cannot be reduced. According to the definition of the diffusion process in \cref{eq:for}, the dimension of $\rvx_T$ must be the same as the data dimension. This hinders learning a compact representation, making it hard to facilitate downstream inference or attribute manipulation~\citep{sinha2021d2c}. Finally, $\rvx_T$ is computationally expensive to obtain. To infer $\rvx_T$ from the data point $\rvx_0$, it is necessary to numerically solve the ODE in \cref{eq:backode}. This results in high time complexity for inferring $\rvx_T$, which makes it inefficient to exploit latent representations.

To resolve the problem in the latent endpoint $\rvx_T$, some previous literature, e.g., DiffAE~\citep{preechakul2022diffusion}, proposes an auxiliary latent space utilizing a learnable encoder Enc$_{\boldsymbol{\phi}}:\mathbb{R}^{d} \rightarrow \mathbb{R}^{l}$, which maps from data $\rvx_0$ to an auxiliary latent variable $\rvz$. Unlike DDIM, these approaches tractably obtain $\rvz$ from $\rvx_0$ without solving the ODE, and the encoder can directly learn the latent space in reduced dimensionality. Consequently, the generative ODE 
\begin{align}
& \mathrm{d}\rvx_t = [\mathbf{f}(\rvx_t, t)-\frac{1}{2}g^2(t) \mathbf{s}_{\boldsymbol{\theta}}(\rvx_t,\rvz,t)] \mathrm{d}t,\label{eq:zbackode}
\end{align}
becomes associated with the $\rvz$-conditional score function $\mathbf{s}_{\boldsymbol{\theta}}(\rvx_t,\rvz,t)$, which approximates $\nabla_{\rvx_t} \log q^t_{\boldsymbol{\phi}}(\rvx_t|\rvz)$. The generation starts from two distinct latent variables $\rvz$ and $\rvx_T$, and defines the conditional probability $p^{\text{ODE}}_{\boldsymbol{\theta}}(\rvx_0|\rvz,\rvx_T)$. The ODE also provides an encoding from $\rvx_0$ and $\rvz$ to $\rvx_T$, which defines the conditional probability $q^{\text{ODE}}_{\boldsymbol{\theta}}(\rvx_T|\rvz,\rvx_0)$. However, the auxiliary encoder framework encounters an \textit{information split problem} which this paper raises in \Cref{sec:10}. This paper proposes a method to mitigate this problem.

\subsection{Diffusion Process with Fixed Endpoints}
\label{sec:2.4}
To control the information regarding the diffusion endpoint $\rvx_T$, it is imperative to specify a forward SDE that terminates at the desired endpoint. We employ Doob's $h$-transform~\citep{doob1984classical}, which facilitates the conversion of the original forward SDE in \cref{eq:for} into 
\begin{align}
& \mathrm{d}\rvx_t = [\mathbf{f}(\rvx_t, t)+g^2(t)\mathbf{h}(\rvx_t,t,\mathbf{y},T)]\mathrm{d}t + g(t)\mathrm{d}\mathbf{w}_t, \quad \rvx_0 \sim q_{\textrm{data}}(\rvx_0), \quad \rvx_T=\mathbf{y}  \label{eq:doobfor},
\end{align}
where $\mathbf{h}(\rvx_t,t,\mathbf{y},T):=\nabla_{\rvx_t}\log{\tilde{q}_t(\rvx_T|\rvx_t)}|_{\rvx_T=\mathbf{y}}$ is the score function of the perturbation kernel from the original forward SDE, and $\mathbf{y}$ denotes the desired endpoint. Let $q_t(\rvx_t)$ denote the marginal distribution of \cref{eq:doobfor} at $t$. It is noteworthy that when both $\rvx_0$ and $\rvx_T$ are given, the conditional probability of $\rvx_t$ becomes identical to that of the original forward SDE, i.e., $q_t(\rvx_t|\rvx_T,\rvx_0)=\tilde{q}_t(\rvx_t|\rvx_T,\rvx_0)$. If the original forward SDE in \cref{eq:for} is set to be a specific form (e.g., variance preserving SDE~\citep{ho2020denoising}), then $q_t(\rvx_t|\rvx_T,\rvx_0)$ follows a Gaussian distribution. This means that sampling of $\rvx_t \sim q_t(\rvx_t|\rvx_T,\rvx_0)$ at any time $t$ is tractable with an exact density function.

Corresponding to the $h$-transformed forward SDE of \cref{eq:doobfor}, there also exist unique reverse-time SDE and ODE~\citep{anderson1982reverse,zhou2024denoising} 
\begin{align}
& \mathrm{d}\rvx_t = [\mathbf{f}(\rvx_t, t)-g^2(t)\nabla_{\rvx_t} \log q_t(\rvx_t|\rvx_T)+g^2(t)\mathbf{h}(\rvx_t,t,\mathbf{y},T)] \mathrm{d}t + g(t)\mathrm{d}\mathbf{\bar{w}}_t, \rvx_T =\mathbf{y},   \label{eq:doobback}\\ 
& \mathrm{d}\rvx_t = [\mathbf{f}(\rvx_t, t)-\frac{1}{2}g^2(t)\nabla_{\rvx_t} \log q_t(\rvx_t|\rvx_T)+g^2(t)\mathbf{h}(\rvx_t,t,\mathbf{y},T)] \mathrm{d}t, \quad \rvx_T =\mathbf{y},\label{eq:doobbackode}
\end{align} 
 where $q_t(\rvx_t|\rvx_T)$ is the conditional probability defined by \cref{eq:doobfor}. To construct the reverse SDE and ODE, it is necessary to estimate $\nabla_{\rvx_t} \log q_t(\rvx_t|\rvx_T) \approx \mathbf{s}_{\boldsymbol{\theta}}(\rvx_t,t,\rvx_T)$ through a neural network with a score matching objective~\citep{zhou2024denoising}
\begin{align}
& \frac{1}{2}\int_0^T \mathbb{E}_{q_{t}(\rvx_t,\rvx_T)}[g^2(t)||\mathbf{s}_{\boldsymbol{\theta}}(\rvx_t,t,\rvx_T)-\nabla_{\rvx_t} \log q_{t}(\rvx_t|\rvx_T)||_2^2] \mathrm{d}t.\label{eq:bridge_sm} 
\end{align}

\section{Motivation: Information Split Problem}
\label{sec:10}
This paper raises a problem in diffusion-based representation learning with auxiliary encoders~\citep{preechakul2022diffusion,zhang2022unsupervised,wang2023infodiffusion,yang2023disdiff,yue2024exploring,Wu_Zheng_2024} introduced in \Cref{sec:2.2}. The latent variable $\rvz$ from the encoder has benefits compared to the latent endpoint $\rvx_T$, but the auxiliary encoder framework encounters an \textit{information split problem}: the information of the data is split into two latent variables $\rvz$ and $\rvx_T$. The generative process in \cref{eq:zbackode} initiates with two latent variables $\rvz$ and $\rvx_T$. If the framework only relies on the tractably inferred latent variable $\rvz$, the reconstruction outcomes depicted in \Cref{fig:1b} appear to fluctuate depending on the choice of $\rvx_T$. This implies that $\rvx_T$ encompasses crucial information necessary for reconstructing $\rvx_0$. To represent all the information of $\rvx_0$, it is necessary to infer $\rvx_T$ by solving the ODE in \cref{eq:zbackode} from input $\rvx_0$ to endpoint $\rvx_T$, enduring its computational costs. Consequently, the persisting issue within the latent variable $\rvx_T$ remains unresolved in this framework.

To learn an informative latent representation, the mutual information between the data and the latent variable needs to be maximized~\citep{alemi2018fixing}. The \textit{information split problem} hinders the maximization of the mutual information between the data $\rvx_0$ and the latent variable $\rvz$. The variational lower bound of the mutual information in the auxiliary encoder framework is
\begin{align}
\mathbb{E}_{q_{\text{data}}(\rvx_0),q_{\boldsymbol{\phi}}(\rvz|\rvx_0)}[-CE(q^{\text{ODE}}_{\boldsymbol{\theta}}(\rvx_T|\rvz,\rvx_0) || p_{\text{prior}}(\rvx_T))] + H \leq MI(\rvx_0,\rvz), \label{eq:mi_diffae}
\end{align}
where $MI(\rvx_0,\rvz):=\mathbb{E}_{q_{\boldsymbol{\phi}}(\rvx_0,\rvz)}[\log{\frac{q_{\boldsymbol{\phi}}(\rvx_0,\rvz)}{q_{\text{data}}(\rvx_0)q_{\boldsymbol{\phi}}(\rvz)}}]$ represents the mutual information, $H:=\mathcal{H}(q_{\text{data}}(\rvx_0))$ denotes the data entropy, and $CE(q^{\text{ODE}}_{\boldsymbol{\theta}}(\rvx_T|\rvz,\rvx_0) || p_{\text{prior}}(\rvx_T)):=\mathbb{E}_{q^{\text{ODE}}_{\boldsymbol{\theta}}(\rvx_T|\rvz,\rvx_0)}[-\log{p_{\text{prior}}(\rvx_T)}]$ is the cross-entropy. The cross-entropy term increases as the discrepancy between $q^{\text{ODE}}_{\boldsymbol{\theta}}(\rvx_T|\rvz,\rvx_0)$ and $p_{\text{prior}}(\rvx_T)$ increases, resulting in a looser lower bound on the mutual information. Since $q^{\text{ODE}}_{\boldsymbol{\theta}}(\rvx_T|\rvz,\rvx_0)$ inherently forms a Dirac delta distribution due to the nature of ODEs, the discrepancy between $q^{\text{ODE}}_{\boldsymbol{\theta}}(\rvx_T|\rvz,\rvx_0)$ and $p_{\text{prior}}(\rvx_T)$ is inevitable in this framework.  For more details, please refer to \Cref{sec:A1.4.1}.


\section{Method: Diffusion Bridge AutoEncoders}
\label{sec:3}
To resolve the \textit{information split problem} in auxiliary encoder models, we introduce Diffusion Bridge AutoEncoders (DBAE) featuring $\rvz$-dependent endpoint $\rvx_T$ inference using a single network propagation. The endpoint $\rvx_T$ in DBAE only depends on $\rvz$, making $\rvz$ an information bottleneck. \Cref{fig:2} illustrates the overall schematic for DBAE. \Cref{sec:3.1} explains the latent variable inference with the encoder-decoder structure and a learnable forward SDE utilizing Doob's $h$-transform. \Cref{sec:3.2} delineates the generative process from the information bottleneck $\rvz$ to data $\rvx_0$. \Cref{sec:3.10} analyzes the benefit of DBAE for mutual information maximization between $\rvx_0$ and $\rvz$. \Cref{sec:3.3} elaborates on the objective function for reconstruction, unconditional generation, and its theoretical justifications. 

\subsection{Encoding from $\rvx_0$ to $\rvx_T$ conditioned on $\rvz$}
\label{sec:3.1}
\begin{figure}[t]
    \centerline{\includegraphics[width=\columnwidth]{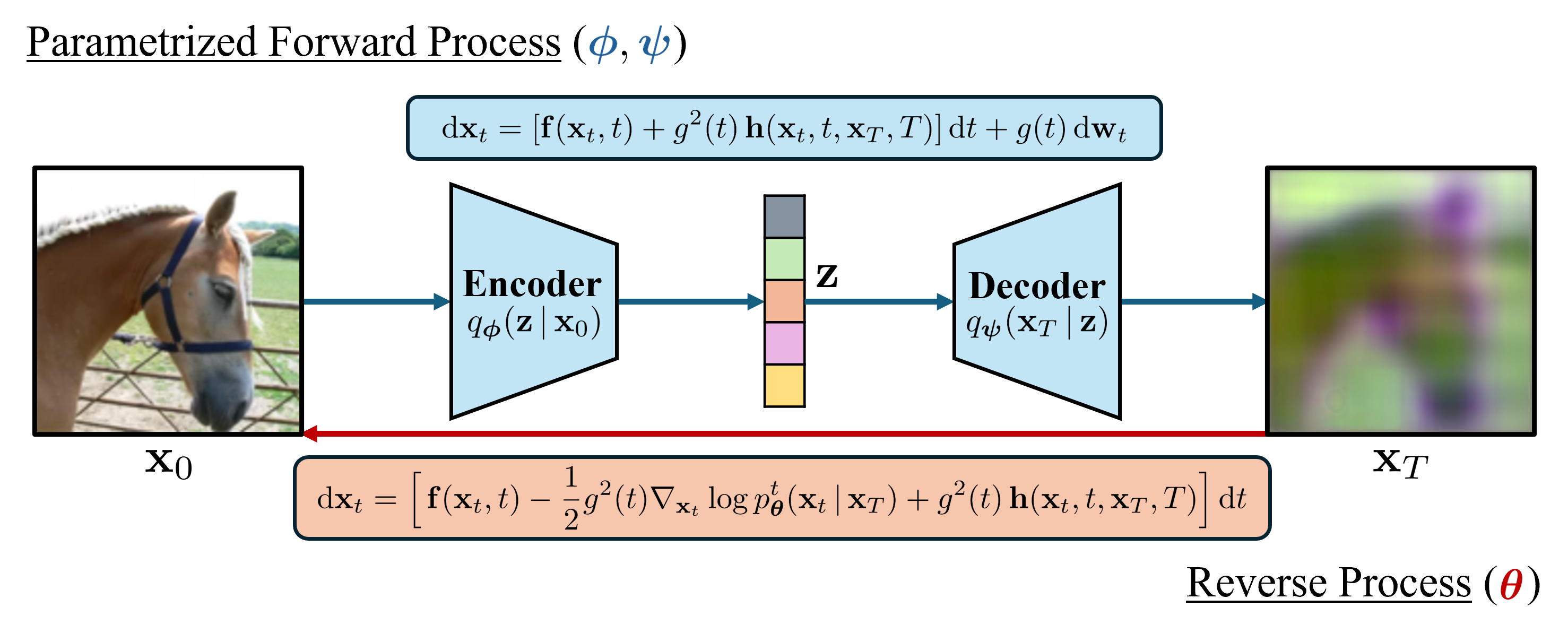}}
    \vspace{-3mm}
    \caption{A schematic for Diffusion Bridge AutoEncoders. The blue line shows the latent variable inference. DBAE infers the $\rvz$-dependent endpoint $\rvx_T$ to make $\rvx_T$ tractable and to establish $\rvz$ as an information bottleneck. The paired $\rvx_0$ and $\rvx_T$ define a new forward SDE utilizing the $h$-transform. The decoder and the red line show the generative process. The generation starts from the bottleneck latent variable $\rvz$ and decodes it to the endpoint $\rvx_T$. The reverse process generates $\rvx_0$ from $\rvx_T$.}
    \label{fig:2}
    \vspace{-5mm}
\end{figure}
We can access i.i.d. samples from $q_{\textrm{data}}(\rvx_0)$. The encoder Enc$_{\boldsymbol{\phi}}:\mathbb{R}^d\rightarrow\mathbb{R}^l$ maps data $\rvx_0$ to the latent variable $\rvz$, defining the conditional probability $q_{\boldsymbol{\phi}}(\rvz|\rvx_0)$. To condense the high-level representation of $\rvx_0$, the latent dimension $l$ is set to be lower than the data dimension $d$. The decoder Dec$_{\boldsymbol{\psi}}:\mathbb{R}^l\rightarrow\mathbb{R}^d$ maps from the latent variable $\rvz$ to the endpoint $\rvx_T$, defining the conditional probability $q_{\boldsymbol{\psi}}(\rvx_T|\rvz)$. The encoder and decoder can be deterministic (i.e., Dirac delta distribution) or stochastic (i.e., Gaussian distribution) depending on the experimental choice. Since the decoder generates the endpoint $\rvx_T$ solely based on the latent variable $\rvz$, $\rvz$ becomes a bottleneck for all the information in $\rvx_0$. The encoder-decoder structure provides the endpoint distribution $q_{\boldsymbol{\phi},\boldsymbol{\psi}}(\rvx_T|\rvx_0)=\int{q_{\boldsymbol{\psi}}(\rvx_T|\rvz)q_{\boldsymbol{\phi}}(\rvz|\rvx_0)}\mathrm{d}\rvz$ for a given starting point $\rvx_0$. We now discuss a new diffusion process $\left\{\rvx_t\right\}_{t=0}^{T}$ with a given starting point and endpoint pair. 

 To establish the relationship between the starting point and endpoint given by the encoder-decoder, we utilize Doob's $h$-transform to define a new forward SDE 
 \begin{align}
& \mathrm{d}\rvx_t = [\mathbf{f}(\rvx_t, t)+g^2(t)\mathbf{h}(\rvx_t,t,\rvx_T,T)]\mathrm{d}t + g(t)\mathrm{d}\mathbf{w}_t,\; \rvx_0 \sim q_{\textrm{data}}(\rvx_0),\, \rvx_T \sim q_{\boldsymbol{\phi},\boldsymbol{\psi}}(\rvx_T|\rvx_0)  \label{eq:doobours},
\end{align}
 where $\mathbf{h}(\rvx_t,t,\rvx_T, T):=\nabla_{\rvx_t}\log{\tilde{q}_t(\rvx_T|\rvx_t)}$ is the score function of the perturbation kernel in the original forward SDE in \cref{eq:for}. The forward SDE in \cref{eq:doobours} determines the distribution of $\rvx_{t}$, where $t\in(0,T)$. Let us denote the marginal distribution of \cref{eq:doobours} at time $t$ as $q^{t}_{\boldsymbol{\phi},\boldsymbol{\psi}}(\rvx_t)$.


\subsection{Generative Process}
\label{sec:3.2}
The generative process begins with the bottleneck latent variable $\rvz$, which can be inferred from the input data $\rvx_0$ or is randomly drawn from the prior distribution $p_{\text{prior}}(\rvz)$. The decoder Dec$_{\boldsymbol{\psi}}:\mathbb{R}^l\rightarrow\mathbb{R}^d$ maps from the latent variable $\rvz$ to the endpoint $\rvx_T$ with the probability $p_{\boldsymbol{\psi}}(\rvx_T|\rvz)$.\footnote{The two distributions $p_{\boldsymbol{\psi}}(\rvx_T|\rvz)$ and $q_{\boldsymbol{\psi}}(\rvx_T|\rvz)$ are the same. However, to distinguish between inference and generation, they are respectively denoted as $p$ and $q$.} Corresponding to a new forward SDE in \cref{eq:doobours}, there exists a reverse ODE 
\begin{align}
& \mathrm{d}\rvx_t = [\mathbf{f}(\rvx_t, t)-\frac{1}{2}g^2(t)\nabla_{\rvx_t} \log q^{t}_{\boldsymbol{\phi},\boldsymbol{\psi}}(\rvx_t|\rvx_T)+g^2(t)\mathbf{h}(\rvx_t,t,\rvx_T,T)] \mathrm{d}t,\label{eq:backodeours}
\end{align}
where the conditional probability $q^{t}_{\boldsymbol{\phi},\boldsymbol{\psi}}(\rvx_t|\rvx_T)$ is defined by \cref{eq:doobours}. However, computing the conditional probability   $q^{t}_{\boldsymbol{\phi},\boldsymbol{\psi}}(\rvx_t|\rvx_T)$ is intractable, so we parameterize our score model $\mathbf{s}_{\boldsymbol{\theta}}(\rvx_t,t,\rvx_T):=\nabla_{\rvx_t}\log p^{t}_{\boldsymbol{\theta}}(\rvx_t|\rvx_T)$ to approximate $\nabla_{\rvx_t}\log q^{t}_{\boldsymbol{\phi},\boldsymbol{\psi}}(\rvx_t|\rvx_T)$.
Our parametrized generative process becomes
\begin{align}
& \mathrm{d}\rvx_t = [\mathbf{f}(\rvx_t, t)-\frac{1}{2}g^2(t)\nabla_{\rvx_t}\log p^{t}_{\boldsymbol{\theta}}(\rvx_t|\rvx_T)+g^2(t)\mathbf{h}(\rvx_t,t,\rvx_T,T)] \mathrm{d}t.\label{eq:backodeoursmodel}
\end{align}
Stochastic sampling with an SDE is also naturally possible as shown in \Cref{sec:2.4}, but we describe only the ODE for convenience.

\begin{minipage}[t]{0.65\textwidth}
 \vspace{-6mm}
\scriptsize 
    \begin{algorithm}[H]
        \SetCustomAlgoRuledWidth{\linewidth}
        \DontPrintSemicolon
        \KwInput{data distribution $q_{\text{data}}(\rvx_0)$, drift term $\mathbf{f}$, volatility term $g$}
        \While{not converges}{%
            Sample time $t$ from $[0, T]$ \\
            $\rvx_0 \sim q_{\text{data}}(\rvx_0)$,  \\
            $\rvz =$ Enc$_{\boldsymbol{\phi}}(\rvx_0)$ and $\rvx_T =$ Dec$_{\boldsymbol{\psi}}(\rvz)$ \\
            $\rvx_t \sim \Tilde{q}_t(\rvx_t\vert\rvx_0,\rvx_T)$ \\
            $\mathcal{L}_{\textrm{AE}}\leftarrow\frac{1}{2}{g^2(t)|| \mathbf{s}_{\boldsymbol{\theta}}(\rvx_t,t,\rvx_T) - \nabla_{\rvx_t}\log \tilde{q}_{t}(\rvx_t|\rvx_0,\rvx_T)||_2^2}$ \\
            Update $\boldsymbol{\phi}$, $\boldsymbol{\psi}$, $\boldsymbol{\theta}$ by $\mathcal{L}_{\textrm{AE}}$ using the gradient descent method
        }
        \KwOutput{Enc$_{\boldsymbol{\phi}}$, Dec$_{\boldsymbol{\psi}}$, score network $\mathbf{s}_{\boldsymbol{\theta}}$}
 \caption{DBAE Training Algorithm for Reconstruction}     
 \label{alg:1}
\end{algorithm}
\end{minipage}
\begin{minipage}[t]{0.35\textwidth}
 \vspace{-6mm}
\scriptsize
\begin{algorithm}[H]
    \SetCustomAlgoRuledWidth{\linewidth}
    \DontPrintSemicolon
    \KwInput{Enc$_{\boldsymbol{\phi}}$, Dec$_{\boldsymbol{\psi}}$, score network $\mathbf{s}_{\boldsymbol{\theta}}$, sample $\rvx_0$, discretized time steps $\{t_i\}_{i=0}^N$}
    $\rvz$ = Enc$_{\boldsymbol{\phi}}(\rvx_0)$ \\
    $\rvx_T$ = Dec$_{\boldsymbol{\psi}}(\rvz)$ \\
    \For{$i=N,...,1$}{
        Update $\rvx_{t_i}$ using \cref{eq:backodeoursmodel} \\
    }
    \KwOutput{Reconstructed sample ${\hat{\rvx}_0}$}
    \caption{Reconstruction}     
    \label{alg:2}
\end{algorithm}
\end{minipage}
\vspace{-5mm}
\subsection{Mutual Information Analysis}
\label{sec:3.10}
From the definition of inference and generation of DBAE in \Cref{sec:3.1,sec:3.2}, the variational lower bound on the mutual information between $\rvx_0$ and $\rvz$ is
\begin{align}
\mathbb{E}_{q_{\boldsymbol{\phi}}(\rvx_0,\rvz)}[\mathbb{E}_{q_{\boldsymbol{\psi}}(\rvx_T|\rvz)}[\log{p_{\boldsymbol{\theta}}(\rvx_0|\rvx_T)}]-D_{KL}(q_{\boldsymbol{\psi}}(\rvx_T|\rvz)||p_{\boldsymbol{\psi}}(\rvx_T|\rvz))] + H \leq MI(\rvx_0,\rvz), \label{eq:mi_dbae}
\end{align}
where $p_{\boldsymbol{\theta}}(\rvx_0|\rvx_T)$ is defined by the generative process in \Cref{sec:3.2}. Please see \Cref{sec:A1.4.2} for a detailed derivation.  Here, the term $D_{KL}(q_{\boldsymbol{\psi}}(\rvx_T|\rvz)||p_{\boldsymbol{\psi}}(\rvx_T|\rvz))$ becomes zero because both conditional probabilities of $\rvx_T$ given $\rvz$ are the same in the inference and the generation. The remaining term $\mathbb{E}_{q_{\boldsymbol{\phi}}(\rvx_0,\rvz)}[\mathbb{E}_{q_{\boldsymbol{\psi}}(\rvx_T|\rvz)}[\log{p_{\boldsymbol{\theta}}(\rvx_0|\rvx_T)}]$ can be controlled by the optimization of $\boldsymbol{\phi}, \boldsymbol{\psi},$ and $ \boldsymbol{\theta}$. The relation between an objective function and mutual information is declared in \Cref{thm:3}.

\subsection{Objective Function}
\label{sec:3.3}
The objective function bifurcates depending on the specific tasks. The model requires a reconstruction capability for downstream inference, attribute manipulation, and interpolation. To achieve reconstruction capability, the model needs 1) an encoding capability $(\rvx_0 \rightarrow \rvz \rightarrow \rvx_T)$ and 2) a regeneration capability $(\rvx_T \rightarrow \rvx_0)$. The encoding process should infer a distinct latent variable for each data point $\rvx_0$ to ensure that the original information is preserved during reconstruction. The regeneration capability needs to estimate the reverse process by approximating $\mathbf{s}_{\boldsymbol{\theta}}(\rvx_t,t,\rvx_T) \approx \nabla_{\rvx_t} \log q^{t}_{\boldsymbol{\phi},\boldsymbol{\psi}}(\rvx_t|\rvx_T)$. For an unconditional generation, the model must possess the ability to generate random samples from the endpoint $\rvx_T$, which implies that the generative endpoint distribution $p_{\boldsymbol{\psi}}(\rvx_T)=\int{p_{\boldsymbol{\psi}}(\rvx_T|\rvz)p_{\text{prior}}(\rvz)}\mathrm{d}\rvz$ should closely match the aggregated inferred distribution $q_{\boldsymbol{\phi},\boldsymbol{\psi}}(\rvx_T)=\int{q_{\boldsymbol{\psi}}(\rvx_T|\rvz)q_{\boldsymbol{\phi}}(\rvz|\rvx_0)q_{\text{data}}(\rvx_0)}\mathrm{d}\rvx_0\mathrm{d}\rvz$.
\subsubsection{Reconstruction} 
\label{sec:3.3.1}
For successful reconstruction, the model needs to fulfill two criteria: 1) encoding the latent variable $\rvx_T$ uniquely depending on the data point $\rvx_0$, and 2) regenerating from $\rvx_T$ to $\rvx_0$. The inferred latent distribution $q_{\boldsymbol{\phi},\boldsymbol{\psi}}(\rvx_T|\rvx_0)$ should provide unique information for each $\rvx_0$. To achieve this, we aim to minimize the entropy $\mathcal{H}(q_{\boldsymbol{\phi},\boldsymbol{\psi}}(\rvx_T|\rvx_0))$ to embed $\rvx_0$-dependent $\rvx_T$ with minimum uncertainty. On the other hand, we maximize the entropy $\mathcal{H}(q_{\boldsymbol{\phi},\boldsymbol{\psi}}(\rvx_T))$ to embed different $\rvx_T$ for each $\rvx_0$. Since the posterior entropy $\mathcal{H}(q_{\boldsymbol{\phi},\boldsymbol{\psi}}(\rvx_0|\rvx_T))=\mathcal{H}(q_{\boldsymbol{\phi},\boldsymbol{\psi}}(\rvx_T|\rvx_0))-\mathcal{H}(q_{\boldsymbol{\phi},\boldsymbol{\psi}}(\rvx_T))+\mathcal{H}(q_{\text{data}}(\rvx_0))$ naturally includes the aforementioned terms, we use this term as a regularization. Minimizing the gap between \cref{eq:backodeours,eq:backodeoursmodel} is necessary for regenerating from $\rvx_T$ to $\rvx_0$. This requires alignment between the inferred score function $\nabla_{\rvx_t}\log q^{t}_{\boldsymbol{\phi},\boldsymbol{\psi}}(\rvx_t|\rvx_T)$ and the model score function $\mathbf{s}_{\boldsymbol{\theta}}(\rvx_t,t,\rvx_T)$. Similarly to \cref{eq:bridge_sm}, we propose the score-matching objective function $\mathcal{L}_{\text{SM}}$ described as 
\begin{align}
&\mathcal{L}_{\textrm{SM}}:= \frac{1}{2}\int_0^T{\mathbb{E}_{q^t_{\boldsymbol{\phi},\boldsymbol{\psi}}(\rvx_t,\rvx_T)}[g^2(t)|| \mathbf{s}_{\boldsymbol{\theta}}(\rvx_t,t,\rvx_T) - \nabla_{\rvx_t}\log q^t_{\boldsymbol{\phi},\boldsymbol{\psi}}(\rvx_t|\rvx_T)||_2^2]}\mathrm{d}t.  \label{eq:obj_ae} 
\end{align}
We train DBAE with the entropy-regularized score matching objective $\mathcal{L}_{\text{AE}}$ described as
\begin{align}
&\mathcal{L}_{\textrm{AE}}:= \mathcal{L}_{\textrm{SM}} + \mathcal{H}(q_{\boldsymbol{\phi},\boldsymbol{\psi}}(\rvx_0|\rvx_T)). \label{eq:obj_dae}
\end{align}
The detailed training and testing procedures are outlined in algorithms \ref{alg:1} and \ref{alg:2}, respectively.
\Cref{thm:1} demonstrates that the entropy-regularized score matching objective in $\mathcal{L}_{\textrm{AE}}$ becomes a tractable form of objective, and it is equivalent to the reconstruction formulation. The inference distribution $q_{\boldsymbol{\phi},\boldsymbol{\psi}}(\rvx_t,\rvx_T|\rvx_0)$ is optimized to provide the best information about $\rvx_0$ for easy reconstruction.
\begin{restatable}{theorem}{thma}
   For the objective function $\mathcal{L}_{\textrm{AE}}$, the following equality holds.
   \begin{equation}
       \mathcal{L}_{\textrm{AE}}=\frac{1}{2}\int_0^T{\mathbb{E}_{q^t_{\boldsymbol{\phi},\boldsymbol{\psi}}(\rvx_0,\rvx_t,\rvx_T)}[g^2(t)|| \mathbf{s}_{\boldsymbol{\theta}}(\rvx_t,t,\rvx_T) - \nabla_{\rvx_t}\log \tilde{q}_{t}(\rvx_t|\rvx_0,\rvx_T)||_2^2]}\mathrm{d}t
   \end{equation}
   Moreover, if \cref{eq:for} is a linear SDE.\footnote{\cref{eq:for} is a linear SDE when the drift function $\mathbf{f}$ is linear with respect to $\rvx_t$.}, there exists $\alpha(t)$, $\beta(t)$, $\gamma(t)$, $\lambda(t)$, such that
   \begin{equation}
   \mathcal{L}_{\textrm{AE}}=\frac{1}{2}\int_0^T{\mathbb{E}_{q^t_{\boldsymbol{\phi},\boldsymbol{\psi}}(\rvx_0,\rvx_t,\rvx_T)}[\lambda(t)|| \mathbf{x}^{0}_{\boldsymbol{\theta}}(\rvx_t,t,\rvx_T) - \rvx_0||_2^2]}\mathrm{d}t,
   \end{equation}
   where $\mathbf{x}^{0}_{\boldsymbol{\theta}}(\rvx_t,t,\rvx_T):=\alpha(t)\rvx_t+\beta(t)\rvx_T+\gamma(t)\mathbf{s}_{\boldsymbol{\theta}}(\rvx_t,t,\rvx_T)$, and $ q^{t}_{\boldsymbol{\phi},\boldsymbol{\psi}}(\rvx_0, \rvx_t, \rvx_T) =\int q_{\textrm{data}}(\rvx_0)q_{\boldsymbol{\phi}}(\rvz|\rvx_0)q_{\boldsymbol{\psi}}(\rvx_T|\rvz) q_t(\rvx_t|\rvx_T, \rvx_0) d\rvz$, following the graphical model in~\cref{fig:1c}.
\label{thm:1}
\end{restatable}
The assumptions and proof of \Cref{thm:1} are in \Cref{sec:A1.1}. Moreover, \Cref{thm:3} shows the objective functions $\mathcal{L}_{\textrm{AE}}$ is the upper bound of the negative mutual information between $\rvx_0$ and $\rvz$ up to a constant. Since the optimization direction of $\mathcal{L}_{\textrm{AE}}$ is aligned with maximizing the mutual information, our objective function makes the mutual information higher, which can make $\rvz$ informative. The proof of \Cref{thm:3} is in \Cref{sec:A1.5}.
\begin{restatable}{theorem}{thmc}
  $- MI(\rvx_0,\rvz) \leq \mathcal{L}_{\textrm{AE}}-H$, where $H=\mathcal{H}(q_{\text{data}}(\rvx_0))$ is a constant w.r.t. {$\boldsymbol{\phi},\boldsymbol{\psi},\boldsymbol{\theta}$}.
\label{thm:3}
\end{restatable}

\subsubsection{Generative Modeling}
\label{sec:3.3.2}
In \Cref{sec:3.3.1}, the discussion focused on the objective function for reconstruction. The distribution of $\rvx_T$ should be considered for generative modeling. This section addresses the discrepancy between the inferred distribution $q_{\boldsymbol{\phi},\boldsymbol{\psi}}(\rvx_T)$ and the generative prior distribution $p_{\boldsymbol{\psi}}(\rvx_T)$. To address this, we propose the objective $\mathcal{L}_{\text{PR}}$ related to the generative prior.  
\begin{align}
\mathcal{L}_{\textrm{PR}}:=&\mathbb{E}_{q_{\text{data}}(\rvx_0)}[D_{\text{KL}}(q_{\boldsymbol{\phi},\boldsymbol{\psi}}(\rvx_T|\rvx_0)||p_{\boldsymbol{\psi}}(\rvx_T))] \label{eq:prior}
\end{align}
\Cref{thm:2} demonstrates that the autoencoding objective $\mathcal{L}_{\text{AE}}$ and prior objective $\mathcal{L}_{\text{PR}}$ bound the  Kullback-Leibler divergence between data distribution $q_{\text{data}}(\rvx_0)$ and the generative model distribution $p_{\boldsymbol{\psi},\boldsymbol{\theta}}(\rvx_0)=\int{p_{\boldsymbol{\theta}}(\rvx_0|\rvx_T)p_{\boldsymbol{\psi}}(\rvx_T|\rvz)p_{\text{prior}}(\rvz)}\mathrm{d}\rvz\mathrm{d}\rvx_T$ up to a constant. The proof is in \Cref{sec:A1.2}.
\begin{restatable}{theorem}{thmb}  
$D_{\textrm{KL}}(q_{\text{data}}(\rvx_0)||p_{\boldsymbol{\psi},\boldsymbol{\theta}}(\rvx_0))\leq \mathcal{L}_{\textrm{AE}}+\mathcal{L}_{\textrm{PR}}-H$, where $H=\mathcal{H}(q_{\text{data}}(\rvx_0))$ is a constant w.r.t. {$\boldsymbol{\phi},\boldsymbol{\psi},\boldsymbol{\theta}$}.
\label{thm:2}
\end{restatable}
For generative modeling, we separately minimize the terms $\mathcal{L}_{\textrm{AE}}$ and $\mathcal{L}_{\textrm{PR}}$, following \citep{esser2021taming,preechakul2022diffusion,zhang2022unsupervised}. The separate training of the generative prior distribution with a powerful generative model effectively reduces the mismatch between the prior and the aggregated posterior~\citep{sinha2021d2c,aneja2021contrastive}. Initially, we optimize $\mathcal{L}_{\textrm{AE}}$ with respect to the parameters of encoder ($\boldsymbol{\phi}$), decoder ($\boldsymbol{\psi}$), and score network ($\boldsymbol{\theta}$), and fix the parameters $\boldsymbol{\theta},\boldsymbol{\phi},\boldsymbol{\psi}$. Subsequently, we newly parameterize the generative prior $p_{\text{prior}}(\rvz):=p_{\boldsymbol{\omega}}(\rvz)$ using a shallow latent diffusion models, and optimize $\mathcal{L}_{\textrm{PR}}$ w.r.t $\boldsymbol{\omega}$. See \Cref{sec:A1.3} for further details.

\section{Experiment}
\label{sec:4}
This section empirically validates the effectiveness of the intended design of the proposed model, DBAE. We utilize the U-Net architecture for the score network $(\boldsymbol{\theta})$, as shown in \cref{fig:architecture_2}. Since our score network needs to account for the additional input $\rvx_T$, we concatenate $\rvx_t$ and $\rvx_T$ as the U-Net input. We employ half of the U-Net architecture as the encoder $(\boldsymbol{\phi})$ and use a CNN-based upsampler as the decoder $(\boldsymbol{\psi})$, adopted from~\citep{liu2021towards}. The encoder and decoder architectures are detailed in \cref{fig:architecture_1}. To compare DBAE with previous diffusion-based representation learning approaches, we adopt the remaining experimental configurations (e.g., batch size, learning rate) from DiffAE~\citep{preechakul2022diffusion} as closely as possible. Detailed experimental configurations are provided in \Cref{sec:A3}. We evaluate both latent encoding capability and generation quality across various tasks. We quantitatively assess the performance of downstream inference, reconstruction, disentanglement, and unconditional generation. Additionally, we qualitatively demonstrate interpolation and attribute manipulation capabilities. Finally, we conduct experiments with two variations of the proposed model’s encoder: 1) a Gaussian stochastic encoder (DBAE) and 2) a deterministic encoder (DBAE-d) for ablation studies. We use a deterministic structure for the decoder.


\subsection{Downstream Inference}
\label{sec:4.1}
\begin{table}[h]
 \vspace{-2mm}
    \centering 
    \caption{Linear-probe attribute prediction quality comparison for models trained on CelebA and FFHQ with dim$(\rvz)=512$. `Gen' indicates the generation capability. The best and second-best results are highlighted in \textbf{bold} and \underline{underline}, respectively. We evaluate 5 times and report the average.}
     \vspace{-2mm}
    \adjustbox{max width=\textwidth}{%
    \begin{tabular}{lc ccc p{0.01\textwidth}  ccc p{0.01\textwidth}|}
        \toprule
              && \multicolumn{3}{c}{\textbf{CelebA}}  && \multicolumn{3}{c}{\textbf{FFHQ}}                       \\
             \cmidrule{3-5} \cmidrule{7-9}
             Method &Gen& AP ($\uparrow$) &Pearson's r ($\uparrow$)& MSE ($\downarrow$)      && AP ($\uparrow$)&Pearson's r ($\uparrow$)& MSE ($\downarrow$) \\
            \midrule
            \midrule
          SimCLR~\citep{chen2020simple}  & \xmark  & 0.597&0.474&0.603&&0.608& 0.481&0.638\\
          $\beta$-TCVAE~\citep{chen2018isolating}  & \cmark  & 0.450&0.378&0.573&&0.432& 0.335& 0.608\\
          IB-GAN~\citep{jeon2021ib}  & \cmark  & 0.442&0.307&0.597&&0.428& 0.260&0.644\\
        \midrule
          DiffAE~\citep{preechakul2022diffusion}  & \cmark  & 0.603&0.598&0.421&&0.605& 0.606& 0.410\\
          PDAE~\citep{zhang2022unsupervised}  & \cmark  & 0.602&0.596&0.410&&0.597& 0.603& 0.416\\
          DiTi~\citep{yue2024exploring}  & \cmark  & 0.623&0.617&\underline{0.392}&&0.614& 0.622& \underline{0.384}\\
          \rowcolor{gray!25} DBAE-d  & \cmark  & \underline{0.650}  &\underline{0.635}&0.413&&\underline{0.656}& \underline{0.638}& 0.404\\
          \rowcolor{gray!25}DBAE & \cmark  & \textbf{0.655}&\textbf{0.643}&\textbf{0.369}&&\textbf{0.664}& \textbf{0.675}& \textbf{0.332}\\
        \bottomrule
    \end{tabular}
    }
\label{tab:main1}
 \vspace{-5mm}
\end{table}

To examine the learned latent representation capability of Enc$_{\boldsymbol{\phi}}$, we perform a linear-probe attribute prediction following DiTi~\citep{yue2024exploring}. We train a linear classifier with parameters ($\mathbf{w}, b$) using data-attribute pairs ($\rvx_0, y$). The attribute prediction $\hat{y}=\mathbf{w}^{T}\rvz+b$ is based on the learned latent representation $\rvz= $ Enc$_{\boldsymbol{\phi}}$($\rvx_0$), which is fitted to predict the ground-truth label $y$. An informative latent representation allows the linear classifier to predict the ground-truth label $y$ more effectively. We evaluate Enc$_{\boldsymbol{\phi}}$($\rvx_0$) trained on CelebA~\citep{liu2015deep} and FFHQ~\citep{karras2019style}. We train a linear classifier on 1) CelebA with 40 binary labels, measuring accuracy as AP, and 2) LFW~\citep{kumar2009attribute} for attribute regression, measuring accuracy using Pearson's $r$ and MSE. \Cref{tab:main1} shows that DBAE outperforms other diffusion-based representation learning baselines. Since DiffAE, PDAE, and DiTi suffer from the \textit{information split problem}, they produce a $\rvz$ that is less informative than DBAE. \Cref{fig:3} presents statistics for 100 reconstructions of the same image with inferred $\rvz$. Because PDAE's reconstruction varies depending on the selection of $\rvx_T$, it suggests that intricate details, such as hair and facial features, are contained in $\rvx_T$, which $\rvz$ fails to capture. This observation aligns with \Cref{fig:app_0}, where significant performance gains are observed for attributes related to facial details, such as shadows and hair. A comparison between DBAE-d and DBAE reveals that the stochastic encoder performs slightly better. We conjecture that the stochastic encoder leverages a broader latent space, which benefits discriminative downstream inference.

\subsection{Reconstruction}

\label{sec:4.2}
\begin{table}[h]
 \vspace{-2mm}
    \begin{minipage}[b]{0.77\textwidth}
    \centering
    \caption{Autoencoding reconstruction quality comparison. Among tractable and 512-dimensional latent variable models, the one yielding the best performance is highlighted in \textbf{bold}, \underline{underline} for the next best performer.}
    \vspace{-1.5mm}
     \adjustbox{max width=\textwidth}{%
     {
        \begin{tabular}{lcc|ccc }
        \toprule
             Method &Tractability& Latent dim ($\downarrow$) &SSIM ($\uparrow$)& LPIPS ($\downarrow$)      & MSE ($\downarrow$) \\
            \midrule
            \midrule
          StyleGAN2 ($\mathcal{W}$)~\citep{karras2020analyzing}  & \xmark  & 512 &0.677&0.168&0.016 \\
          StyleGAN2 ($\mathcal{W+}$)~\citep{abdal2019image2stylegan}  & \xmark  & 7,168 &0.827&0.114&0.006 \\
          VQ-GAN~\citep{esser2021taming}  & \cmark  & 65,536 &0.782&0.109&3.61e-3 \\
          VQ-VAE2~\citep{razavi2019generating}  & \cmark  & 327,680 &0.947&0.012&4.87e-4 \\
          NVAE~\citep{vahdat2020nvae}  & \cmark  & 6,005,760 &0.984&0.001&4.85e-5 \\
          \midrule
          DDIM (Inferred $\mathbf{x}_T$)~\citep{song2021denoising}  & \xmark  & 49,152 &0.917&0.063&0.002 \\
          DiffAE (Inferred $\mathbf{x}_T$)~\citep{preechakul2022diffusion}  & \xmark  & 49,664 &0.991&0.011&6.07e-5 \\
          PDAE (Inferred $\mathbf{x}_T$)~\citep{zhang2022unsupervised}  & \xmark  & 49,664 &0.994&0.007&3.84e-5 \\
          \midrule
          \midrule
          DiffAE (Random $\mathbf{x}_T$)~\citep{preechakul2022diffusion}  & \cmark  & 512 &0.677&\underline{0.073}&0.007 \\
          PDAE (Random $\mathbf{x}_T$)~\citep{zhang2022unsupervised}  & \cmark  & 512 &0.689&0.098&5.01e-3 \\
           \rowcolor{gray!25} DBAE  & \cmark  & 512 &\underline{0.920}&0.094&\underline{4.81e-3} \\
           \rowcolor{gray!25} DBAE-d  & \cmark  & 512 &\textbf{0.953}&\textbf{0.072}&\textbf{2.49e-3} \\
        \bottomrule
    \end{tabular}
    }
    \label{tab:main2}
    }
    \end{minipage}
    \hfill
    \begin{minipage}[t]{0.22\textwidth}
    \setlength{\abovecaptionskip}{+3pt}
    \centering
    \raisebox{-11mm}{\includegraphics[width=0.95\textwidth]{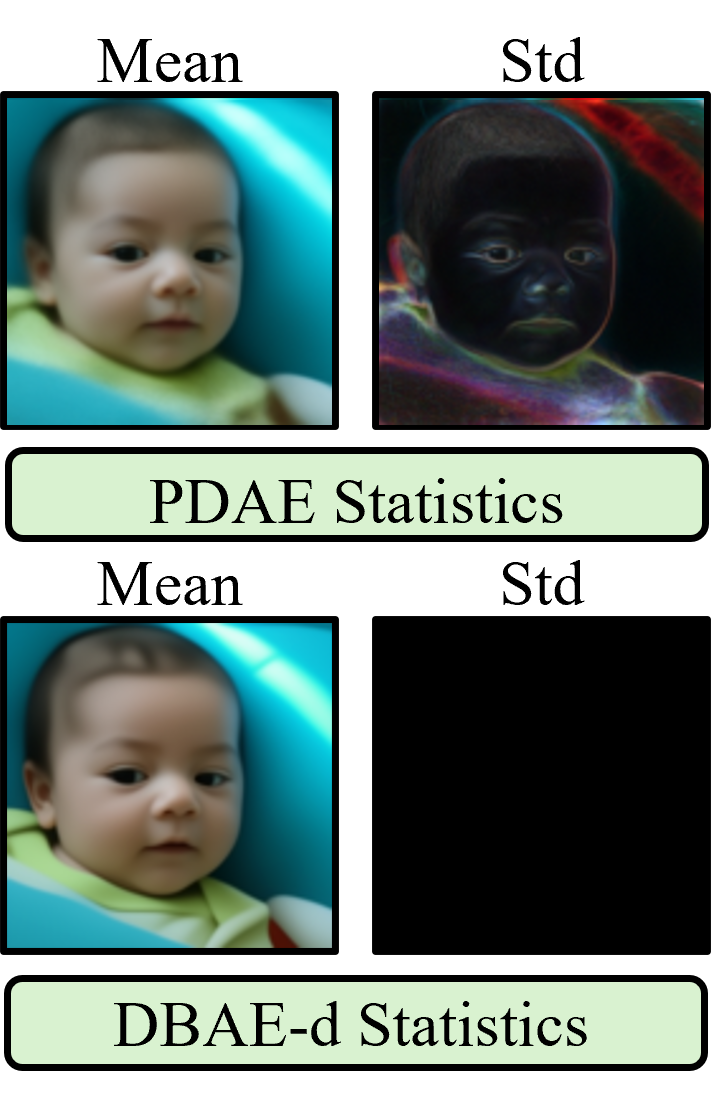}}
    \captionof{figure}{Reconstruction w/ inferred $\rvz$.}
    \label{fig:3}
    \end{minipage}
    \vspace{-5mm}
\end{table}
We examine the reconstruction quality following DiffAE~\citep{preechakul2022diffusion} to quantify information loss in the latent variable. For a test sample $\rvx_0$, the procedure in algorithm \ref{alg:2} provides a reconstructed sample $\hat{\rvx}_0$. The reconstruction error is the distance $d(\rvx_0, \hat{\rvx}_0)$, where the distance function can be SSIM~\citep{wang2003multiscale}, LPIPS~\citep{zhang2018unreasonable}, or MSE. \Cref{tab:main2} reports the averaged reconstruction error over the test dataset $\mathbb{E}_{p_{\textrm{test}}(\rvx_0)}[d(\rvx_0, \hat{\rvx}_0)]$. We trained DBAE on FFHQ and evaluated it on CelebA-HQ~\citep{karras2018progressive}. Tractability refers to the ability to perform inference on latent variables without repeated neural network evaluations. Tractability is crucial for regularizing the latent variable to achieve specific goals (e.g., disentanglement) during the training phase. The latent dimension refers to the dimension of the bottleneck latent variable during inference. A lower dimension is advantageous for applications such as downstream inference or attribute manipulation. The third block in \Cref{tab:main2} compares performance under the same qualitative conditions. DBAE-d exhibits performance that surpasses both DiffAE and PDAE. Naturally, DiffAE and PDAE exhibit worse performance because the information is split between $\rvx_T$ and $\rvz$. Unlike the downstream inference experiments in \Cref{sec:4.1}, the deterministic encoder performs better. 


\subsection{Disentanglement}
\label{sec:4.3}
\begin{table}[h]
 \vspace{-2mm}
    \begin{minipage}[h]{0.69\textwidth}
    \centering
    \caption{Disentanglment and sample quality comparisons on CelebA.}
    \vspace{-2mm}
     \adjustbox{max width=\textwidth}{%
     {\footnotesize
     \begin{tabular}{lc|ccc }
        \toprule
             Method  & Reg $\rvz$ &TAD ($\uparrow$)& ATTRS ($\uparrow$)      & FID ($\downarrow$) \\
            \midrule
            \midrule
          AE   &\xmark& 0.042$\pm{0.004}$&1.0$\pm{0.0}$&90.4$\pm{1.8}$ \\
          DiffAE~\citep{preechakul2022diffusion}&\xmark& \textbf{0.155}$\pm{0.010}$&2.0$\pm{0.0}$&22.7$\pm{2.1}$ \\
           \rowcolor{gray!25}DBAE& \xmark&0.124$\pm{0.078}$&\textbf{2.2}$\pm{1.3}$&\textbf{11.8}$\pm{0.2}$ \\
          \midrule
          \midrule
          VAE ~\citep{KingmaW13}  &\cmark& 0.000$\pm{0.000}$&0.0$\pm{0.0}$&94.3$\pm{2.8}$ \\
          $\beta$-VAE~\citep{higgins2017betavae}& \cmark&0.088$\pm{0.051}$&1.6$\pm{0.8}$&99.8$\pm{2.4}$ \\
          InfoVAE~\citep{zhao2019infovae}&\cmark &0.000$\pm{0.000}$&0.0$\pm{0.0}$&77.8$\pm{1.6}$ \\     
          InfoDiffusion~\citep{wang2023infodiffusion}  & \cmark&0.299$\pm{0.006}$&3.0$\pm{0.0}$&22.3$\pm{1.2}$ \\
          DisDiff~\citep{yang2023disdiff}&\cmark &0.305$\pm{0.010}$&-&18.3$\pm{2.1}$ \\
          \rowcolor{gray!25}DBAE+TC& \cmark&\textbf{0.362}$\pm{0.036}$&\textbf{3.8}$\pm{0.8}$&\textbf{13.4}$\pm{0.2}$ \\
        \bottomrule
    \end{tabular}
    }
    \label{tab:main3}
    }
    \end{minipage}
    \hfill
    \begin{minipage}[t]{0.30\textwidth}
    \setlength{\abovecaptionskip}{5pt}
    \centering
    \raisebox{-10mm}{\includegraphics[width=0.98\textwidth, height=2.5cm]{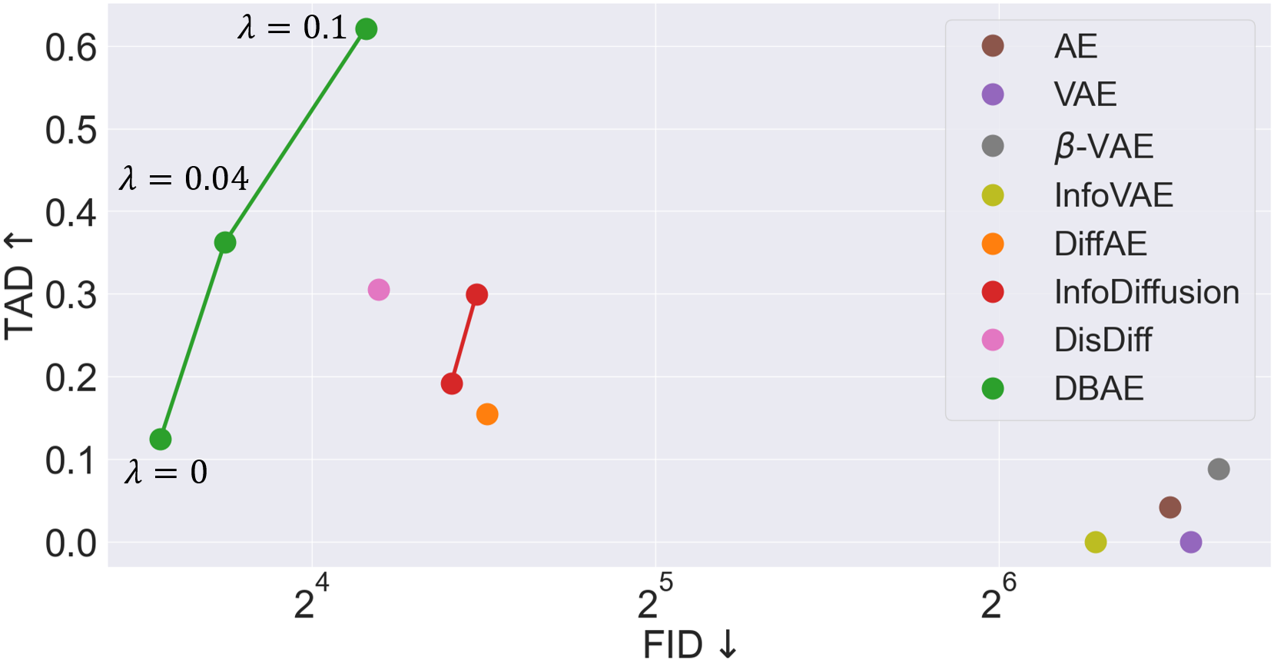}}
    \captionof{figure}{TAD-FID tradeoffs compared to the baselines.}
    \label{fig:4}
    \end{minipage}
    \vspace{-5mm}
\end{table}
Unsupervised disentanglement of the latent variable $\rvz$ is an important application of generative representation learning, as it enables controllable generation without supervision. The goal of disentanglement is to ensure that each dimension of the latent variable captures distinct information.  To achieve this, we apply regularization to minimize total correlation (TC), i.e., $D_{\text{KL}}(q_{\boldsymbol{\phi}}(\rvz)||\Pi_{i=1}^{l} q_{\boldsymbol{\phi}}(\rvz_i))$, adopted from~\citep{chen2018isolating}. TC regularization decouples the correlation between the dimensions of $\rvz$, allowing different information to be captured in each dimension. 
Following InfoDiffusion~\citep{wang2023infodiffusion}, we measure TAD and ATTRS~\citep{yeats2022nashae} to quantify disentanglement in $\rvz$. Since sample quality and disentanglement often involve a trade-off, we also measure FID~\citep{heusel2017gans} between 10k samples. \Cref{tab:main3} shows the performance comparison, where DBAE outperforms all the baselines. \Cref{fig:4} demonstrates the effects of coefficients on TC regularization, showing that DBAE envelops all the baselines. To disentangle information, a well-encoded representation must first be achieved. The informative representation capability of DBAE supports this application.


\subsection{Unconditional Generation}
\label{sec:4.4}
\begin{table}[h]
 \vspace{-3mm}
    \begin{minipage}[h]{0.67\textwidth}
    \centering
    \caption{Unconditional generation on FFHQ. `+AE' indicates the use of the inferred distribution $q_{\boldsymbol{\phi}}(\rvz)$ instead of $p_{\boldsymbol{\omega}}(\rvz)$.}
     \adjustbox{max width=\textwidth}{%
        \begin{tabular}{l|cccc }
        \toprule
             Method  & Prec ($\uparrow$) & IS ($\uparrow$)& FID 50k ($\downarrow$)      & Rec ($\uparrow$) \\
            \midrule
            \midrule
          DDIM~\citep{song2021denoising} & 0.697&3.14&11.27&0.451 \\
          DDPM~\citep{ho2020denoising} & 0.768&3.11&\textbf{9.14}&0.335 \\
          DiffAE ~\citep{preechakul2022diffusion}  & 0.762&2.98&9.40&\textbf{0.458} \\
          PDAE~\citep{zhang2022unsupervised} &0.695&2.23&47.42&0.153 \\   
          \rowcolor{gray!25}DBAE &\textbf{0.780}&\textbf{3.87}&11.25&0.392 \\
           \midrule
            \midrule
             DiffAE+AE&0.750&\textbf{3.63}&2.84&0.685 \\
            PDAE+AE  & 0.709&3.55&7.42&0.602 \\
             \rowcolor{gray!25}DBAE+AE& \textbf{0.751}&3.57&\textbf{1.77}&\textbf{0.687} \\
        \bottomrule
     \end{tabular}
    \label{tab:main4}
    }
    \end{minipage}
    \hfill
    \begin{minipage}[t]{0.32\textwidth}
    \setlength{\abovecaptionskip}{1pt}
    \centering
    \raisebox{-6mm}{\includegraphics[width=0.99\textwidth]{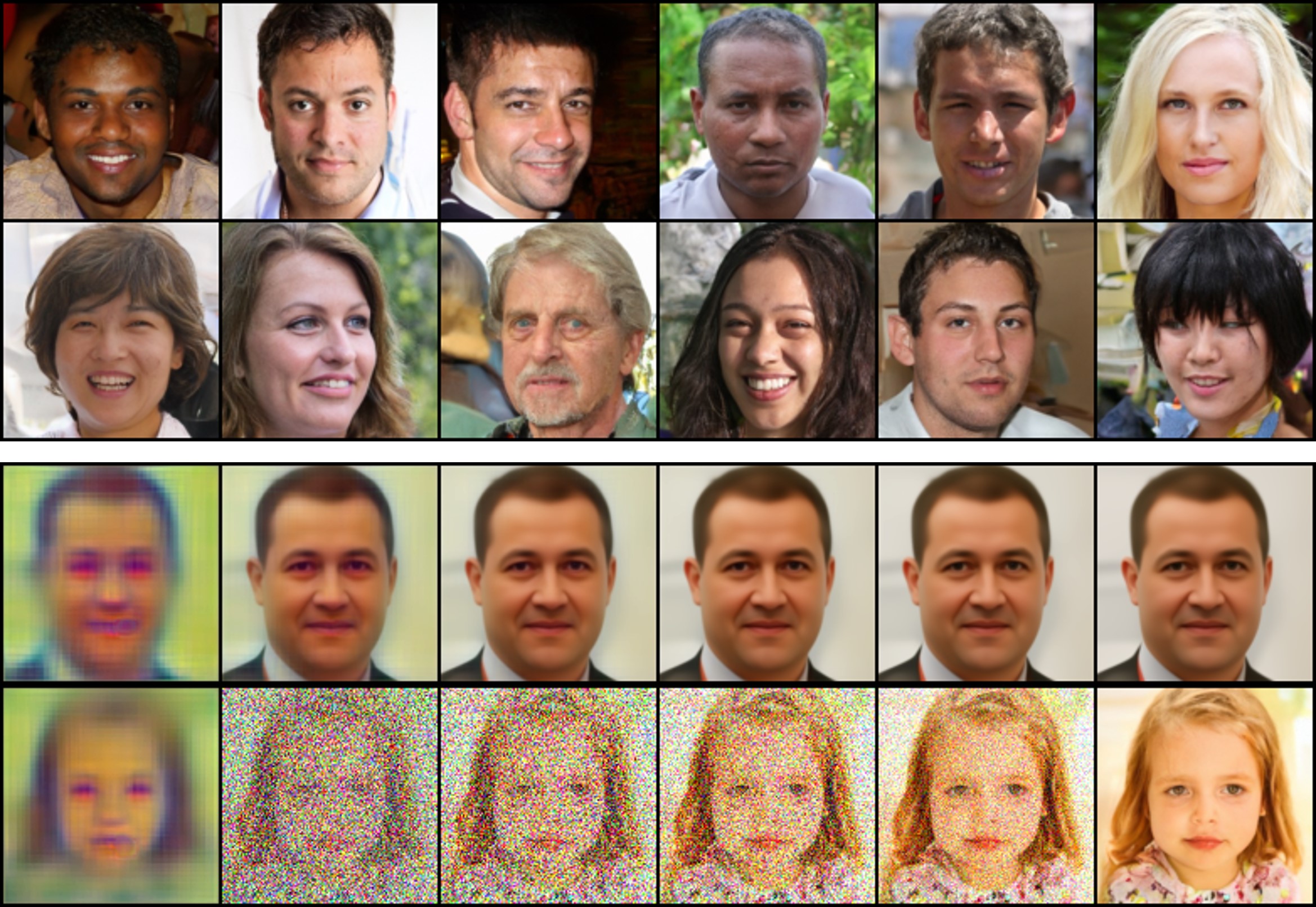}}
    \captionof{figure}{Top two rows: uncurated samples. Bottom two rows: the sampling trajectory with ODE and SDE.}
    \label{fig:5}
    \end{minipage}
    \vspace{-5mm}
\end{table}
To generate a sample unconditionally, the generation starts from the learned prior distribution $\rvz \sim p_{\boldsymbol{\omega}}(\rvz)$. The latent variable $\rvz$ is decoded into $\rvx_T=$ Dec$_{\boldsymbol{\psi}}(\rvz)$, and the sample $\rvx_0$ is finally obtained through the generative process described in \cref{eq:backodeoursmodel}. For CelebA, a comparison with DiffAE in \Cref{tab:main3} shows that DBAE surpasses DiffAE by a large margin in FID~\citep{heusel2017gans} (22.7 vs. 11.8). \Cref{tab:main4} shows the performance on FFHQ, which is known to be more diverse than CelebA. DBAE still performs the best among the baselines in terms of Precision~\citep{kynkaanniemi2019improved} and Inception Score (IS)~\citep{salimans2016improved}, both of which are highly influenced by image fidelity. However, DBAE shows slightly worse FID~\citep{heusel2017gans} and Recall~\citep{kynkaanniemi2019improved}, which are more affected by sample diversity. To analyze this, we alter the learned generative prior $p_{\boldsymbol{\omega}}(\rvz)$ to the inferred distribution $q_{\boldsymbol{\phi}}(\rvz)$ as shown in the second block of \Cref{tab:main4}. In this autoencoding case, DBAE captures both image fidelity and diversity. We speculate that it is more sensitive to the gap between $q_{\boldsymbol{\phi}}(\rvz)$ and $p_{\boldsymbol{\omega}}(\rvz)$ since the information depends solely on $\rvz$, not on the joint condition of $\rvx_T$ and $\rvz$. A complex generative prior model $\boldsymbol{\omega}$ could potentially solve this issue~\citep{esser2021taming,vahdat2021score}. \Cref{fig:5} shows the randomly generated samples and sampling trajectories on FFHQ from DBAE.

\subsection{Interpolation}
\label{sec:4.5}
For the two images $\rvx_0^1$ and $\rvx_0^2$, DBAE can mix the styles by exploring the intermediate points in the latent space. We encode images into $\rvz^1=$Enc$_{\boldsymbol{\phi}}(\rvx_0^1)$ and $\rvz^2=$Enc$_{\boldsymbol{\phi}}(\rvx_0^2)$. We then regenerate from $\rvz^{\lambda}= \lambda\rvz^1 + (1-\lambda)\rvz^2$ to data $\rvx_0$ using the generative process specified in \cref{eq:backodeoursmodel}. The unique properties of DBAE offer distinct benefits here: 1) DiffAE~\citep{preechakul2022diffusion} and PDAE~\citep{zhang2022unsupervised} need to infer $\rvx_T^1$, $\rvx_T^2$ by solving the ODE in \cref{eq:zbackode} with hundreds of score function evaluations~\citep{preechakul2022diffusion,zhang2022unsupervised}. They then geometrically interpolate between $\rvx_T^1$ and $\rvx_T^2$ to obtain $\rvx_T^{\lambda}$, regardless of the correspondence between $\rvz^{\lambda}$ and $\rvx_T^{\lambda}$. 2) DBAE directly obtains an intermediate value of $\rvx_T^{\lambda}=$Dec$_{\boldsymbol{\psi}}(\rvz^{\lambda})$. This does not require solving the ODE, and the correspondence between $\rvx_T^{\lambda}$ and $\rvz^{\lambda}$ is also naturally determined by the decoder ($\boldsymbol{\psi}$). \Cref{fig:6} shows the interpolation results on the LSUN Horse, Bedroom~\citep{yu15lsun} and FFHQ datasets. The top row shows the corresponding endpoints $\rvx_T^{\lambda}$ in the interpolation, which changes smoothly between $\rvx_T^{1}$ and $\rvx_T^{2}$. The bottom row shows the interpolation results on FFHQ, which smoothly changes semantic information such as gender, glasses, and hair color.

\subsection{Attribute Manipulation}
\label{sec:4.6}
\begin{figure}[t]
     \centering
     \begin{subfigure}[b]{\textwidth}
         \centering
         \includegraphics[width=1.0\textwidth, height=1.8in]{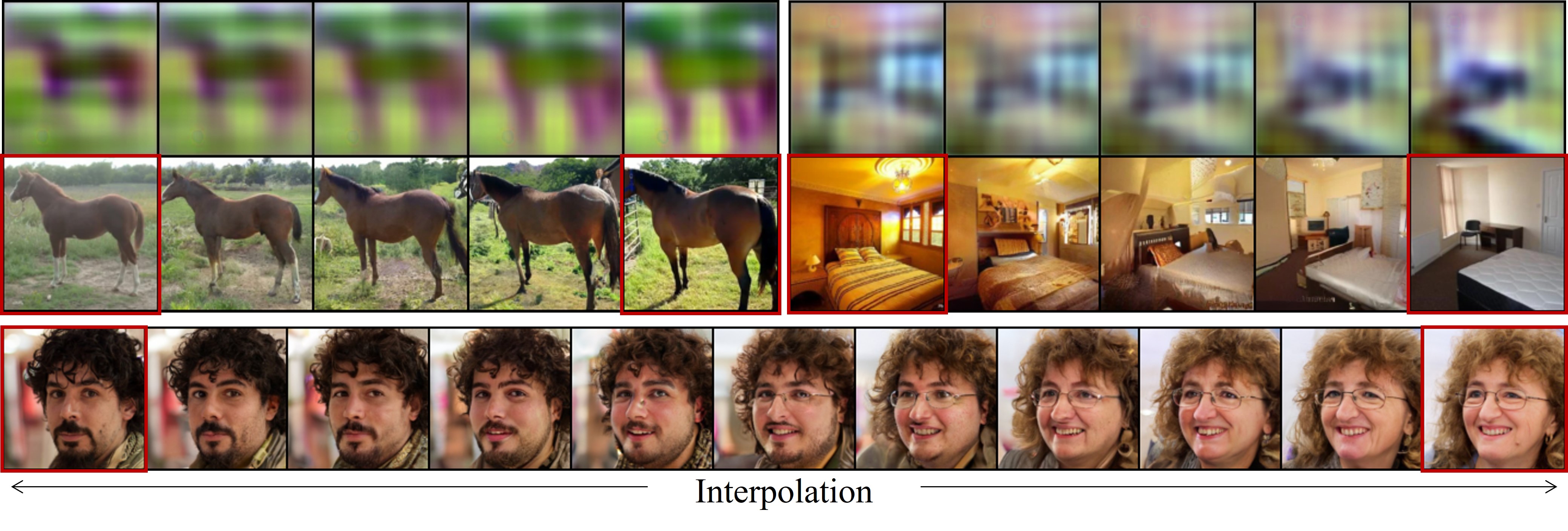}
         \vspace{-5mm}
     \end{subfigure}
     \vspace{1mm}
     \begin{subfigure}[b]{\textwidth}
         \centering
         \includegraphics[width=1.0\textwidth, height=1.4in]{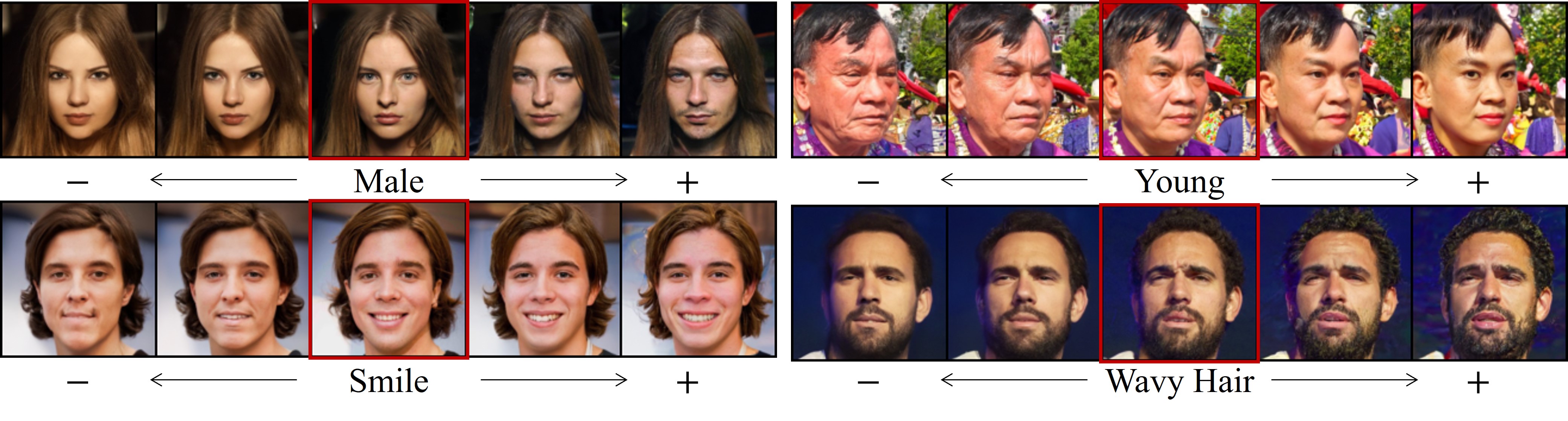}
         \vspace{-7mm}
     \end{subfigure}
    \vspace{-6mm}
    \caption{Interpolation (top) and attribute manipulation (bottom) with DBAE. (Red box: input image)}
    \label{fig:6}
\vspace{-6mm}
\end{figure}

The linear classifier used in \cref{sec:4.1} can also be utilized to identify the manipulation direction of $\rvz$. From the prediction of a linear classifier $\hat{y}=\mathbf{w}^{T}\rvz+b$, traversing in the direction $\frac{dy}{d\rvz}=\mathbf{w}$ increases or decreases the logit. For a image $\rvx_0$, this is encoded as $\rvz=$ Enc$_{\boldsymbol{\phi}}(\rvx_0)$. The encoded representation $\rvz$ is manipulated as $\rvz^{\textrm{new}} = \rvz +\lambda \mathbf{w}$. The manipulated image $\rvx^\textrm{new}_0$ is obtained by decoding $\rvx^\textrm{new}_T=$ Dec$_{\boldsymbol{\psi}}(\rvz^{\textrm{new}})$, and the reverse process in \cref{eq:backodeoursmodel}. DiffAE and PDAE additionally infer from $\rvx_0$ to $\rvx_T$ by solving \cref{eq:zbackode} with hundreds of score function evaluations, fixing $\rvx_T$ to prevent undesirable variations in $\rvx_T$. \Cref{tab:comput2} describes the long inference time for $\rvx_T$ in previous approaches. Moreover, if some information is split into $\rvx_T$, these methods cannot handle this information. On the other hand, DBAE infers $\rvx_T$ directly from manipulated $\rvz$, ensuring that the endpoint $\rvx_T$  is also controlled through the decoder ($\boldsymbol{\psi}$). \Cref{fig:6} shows the manipulation results for both CelebA-HQ images and FFHQ images with various attributes.


\section{Conclusion}
\label{sec:5}
This paper identifies the \textit{information split problem} in diffusion-based representation learning, stemming from separate inferences of the forward process and the auxiliary encoder. This issue hinders the representation capabilities of the tractable latent variable $\rvz$. The proposed method, Diffusion Bridge AutoEncoders, systematically addresses these challenges by constructing $\rvz$-dependent endpoint $\rvx_T$ inference. By transforming $\rvz$ into an information bottleneck, DBAE extracts more meaningful representations within the tractable latent space. The notable enhancements in the latent quality of DBAE improve downstream inference and image manipulation applications. This work lays a solid foundation for further exploration of effective representation in learnable diffusion inference. 

\subsubsection*{Acknowledgments}
This work was supported by Institute of Information \& Communication Technology Planning \& Evaluation (IITP) grant funded by the Korea government (MSIT) (No. RS-2024-00361319, Development of technology to generate synthetic data in exceptional situations and advance artificial intelligence prediction models). This research was supported in part by the NAVER-Intel Co-Lab. The work was conducted by KAIST and reviewed by both NAVER and Intel.

\bibliography{iclr2025_conference}
\bibliographystyle{iclr2025_conference}

\appendix
\newpage
\tableofcontents
\newpage
\section{Proofs and Mathematical Explanations}
\label{sec:A1}

In this section, we follow the assumptions from Appendix A in \citep{song2021maximum}, and we also assume that both $\mathbf{s}_{\boldsymbol{\theta}}$ and $q_{\boldsymbol{\phi},\boldsymbol{\psi}}^t$ have continuous second-order derivatives and finite second moments, which are the same assumptions of Theorems 2 and 4 in \citep{song2021maximum}.

\subsection{Proof of Theorem \ref{thm:1}}
\label{sec:A1.1}
\thma*
\begin{proof}

Note that the definitions of the objective functions are 
\begin{align}
&\mathcal{L}_{\text{SM}}:= \frac{1}{2}\int_0^T{\mathbb{E}_{q^t_{\boldsymbol{\phi},\boldsymbol{\psi}}(\rvx_t,\rvx_T)}[g^2(t)|| \mathbf{s}_{\boldsymbol{\theta}}(\rvx_t,t,\rvx_T) - \nabla_{\rvx_t}\log q^t_{\boldsymbol{\phi},\boldsymbol{\psi}}(\rvx_t|\rvx_T)||_2^2]}\mathrm{d}t,\\
&\mathcal{L}_{\text{AE}}:=\mathcal{L}_{\text{SM}} + \mathcal{H}(q_{\boldsymbol{\phi},\boldsymbol{\psi}}(\rvx_0|\rvx_T)).
\end{align}

We derive the score-matching objective $\mathcal{L}_{\text{SM}}$ with the denoising version for tractability. First, $\mathcal{L}_{\text{SM}}$ is derived as follows.
\begin{align}
\label{eq:thm1-1}
\mathcal{L}_{\text{SM}}=  \frac{1}{2}\int_0^T\mathbb{E}_{q^t_{\boldsymbol{\phi},\boldsymbol{\psi}}(\rvx_t,\rvx_T)}&[g^2(t) || \mathbf{s}_{\boldsymbol{\theta}}(\rvx_t,t,\rvx_T)||_2^2  + g^2(t) ||\nabla_{\rvx_t}\log q^t_{\boldsymbol{\phi},\boldsymbol{\psi}}(\rvx_t|\rvx_T)||_2^2 
\nonumber \\
& - 2g^2(t)  \mathbf{s}_{\boldsymbol{\theta}}(\rvx_t,t,\rvx_T)^T \nabla_{\rvx_t}\log q^t_{\boldsymbol{\phi},\boldsymbol{\psi}}(\rvx_t|\rvx_T) ]\mathrm{d}t.
\end{align}

Then, the last inner product term of \cref{eq:thm1-1} can be deduced in a similar approach to \citep{vincent2011connection}:
\begin{align}
&\mathbb{E}_{q^t_{\boldsymbol{\phi},\boldsymbol{\psi}}(\rvx_t,\rvx_T)}[ \mathbf{s}_{\boldsymbol{\theta}}(\rvx_t,t,\rvx_T)^T \nabla_{\rvx_t}\log q^t_{\boldsymbol{\phi},\boldsymbol{\psi}}(\rvx_t|\rvx_T)] \\
&= \int q^t_{\boldsymbol{\phi},\boldsymbol{\psi}}(\rvx_t,\rvx_T)\mathbf{s}_{\boldsymbol{\theta}}(\rvx_t,t,\rvx_T)^T \nabla_{\rvx_t}\log q^t_{\boldsymbol{\phi},\boldsymbol{\psi}}(\rvx_t|\rvx_T) \mathrm{d}\rvx_t \mathrm{d}\rvx_T \\
&= \int q^t_{\boldsymbol{\phi},\boldsymbol{\psi}}(\rvx_T)q^t_{\boldsymbol{\phi},\boldsymbol{\psi}}(\rvx_t|\rvx_T)\mathbf{s}_{\boldsymbol{\theta}}(\rvx_t,t,\rvx_T)^T \nabla_{\rvx_t}\log q^t_{\boldsymbol{\phi},\boldsymbol{\psi}}(\rvx_t|\rvx_T) \mathrm{d}\rvx_t \mathrm{d}\rvx_T \\
&= \mathbb{E}_{q^t_{\boldsymbol{\phi},\boldsymbol{\psi}}(\rvx_T)} \Big [ \int q^t_{\boldsymbol{\phi},\boldsymbol{\psi}}(\rvx_t|\rvx_T)\mathbf{s}_{\boldsymbol{\theta}}(\rvx_t,t,\rvx_T)^T \nabla_{\rvx_t}\log q^t_{\boldsymbol{\phi},\boldsymbol{\psi}}(\rvx_t|\rvx_T) \mathrm{d}\rvx_t \Big ] \\
&= \mathbb{E}_{q^t_{\boldsymbol{\phi},\boldsymbol{\psi}}(\rvx_T)} \Big [ \int \mathbf{s}_{\boldsymbol{\theta}}(\rvx_t,t,\rvx_T)^T \nabla_{\rvx_t}q^t_{\boldsymbol{\phi},\boldsymbol{\psi}}(\rvx_t|\rvx_T) \mathrm{d}\rvx_t \Big ] \\
&= \mathbb{E}_{q^t_{\boldsymbol{\phi},\boldsymbol{\psi}}(\rvx_T)} \Big [ \int \mathbf{s}_{\boldsymbol{\theta}}(\rvx_t,t,\rvx_T)^T \Big \{ \nabla_{\rvx_t} \int q^t_{\boldsymbol{\phi},\boldsymbol{\psi}}(\rvx_0|\rvx_T)q^t_{\boldsymbol{\phi},\boldsymbol{\psi}}(\rvx_t|\rvx_T,\rvx_0) \mathrm{d}\rvx_0 \Big \} \mathrm{d}\rvx_t \Big ] \\
&= \mathbb{E}_{q^t_{\boldsymbol{\phi},\boldsymbol{\psi}}(\rvx_T)} \Big [ \int \mathbf{s}_{\boldsymbol{\theta}}(\rvx_t,t,\rvx_T)^T \Big \{ \int q^t_{\boldsymbol{\phi},\boldsymbol{\psi}}(\rvx_0|\rvx_T)\nabla_{\rvx_t}q^t_{\boldsymbol{\phi},\boldsymbol{\psi}}(\rvx_t|\rvx_T,\rvx_0) \mathrm{d}\rvx_0 \Big \} \mathrm{d}\rvx_t \Big ] \\
&= \mathbb{E}_{q^t_{\boldsymbol{\phi},\boldsymbol{\psi}}(\rvx_T)} \Big [ \int \mathbf{s}_{\boldsymbol{\theta}}(\rvx_t,t,\rvx_T)^T \Big \{ \int q^t_{\boldsymbol{\phi},\boldsymbol{\psi}}(\rvx_0|\rvx_T)q^t_{\boldsymbol{\phi},\boldsymbol{\psi}}(\rvx_t|\rvx_T,\rvx_0)\nabla_{\rvx_t}\log q^t_{\boldsymbol{\phi},\boldsymbol{\psi}}(\rvx_t|\rvx_T,\rvx_0) \mathrm{d}\rvx_0 \Big \} \mathrm{d}\rvx_t \Big ] \\
&= \mathbb{E}_{q^t_{\boldsymbol{\phi},\boldsymbol{\psi}}(\rvx_T)} \Big [ \int \int q^t_{\boldsymbol{\phi},\boldsymbol{\psi}}(\rvx_0|\rvx_T)q^t_{\boldsymbol{\phi},\boldsymbol{\psi}}(\rvx_t|\rvx_T,\rvx_0)\mathbf{s}_{\boldsymbol{\theta}}(\rvx_t,t,\rvx_T)^T  \nabla_{\rvx_t}\log q^t_{\boldsymbol{\phi},\boldsymbol{\psi}}(\rvx_t|\rvx_T,\rvx_0) \mathrm{d}\rvx_0 \mathrm{d}\rvx_t \Big ] \\
&= \mathbb{E}_{q^t_{\boldsymbol{\phi},\boldsymbol{\psi}}(\rvx_0,\rvx_t,\rvx_T)}  [\mathbf{s}_{\boldsymbol{\theta}}(\rvx_t,t,\rvx_T)^T  \nabla_{\rvx_t}\log q^t_{\boldsymbol{\phi},\boldsymbol{\psi}}(\rvx_t|\rvx_T,\rvx_0)  ]
\end{align}


Next, we rewrite the second term of \cref{eq:thm1-1}. To begin, we express the entropy $\mathcal{H}(q_{\boldsymbol{\phi},\boldsymbol{\psi}}(\rvx_0|\rvx_T))$ with $\nabla_{\rvx_t}\log q^t_{\boldsymbol{\phi},\boldsymbol{\psi}}(\rvx_t|\rvx_T)$, which is similar to the proof of Theorem 4 in \citep{song2021maximum}.
Let $\mathcal{H}(q_{\boldsymbol{\phi},\boldsymbol{\psi}}(\rvx_t,\rvx_T)):= - \int q_{\boldsymbol{\phi},\boldsymbol{\psi}}(\rvx_t,\rvx_T) \log q_{\boldsymbol{\phi},\boldsymbol{\psi}}(\rvx_t,\rvx_T) \mathrm{d}\rvx_t\mathrm{d}\rvx_T $ be the joint entropy function of $q_{\boldsymbol{\phi},\boldsymbol{\psi}}(\rvx_t,\rvx_T)$. Note that $\mathcal{H}(q_{\boldsymbol{\phi},\boldsymbol{\psi}}(\rvx_T,\rvx_T)) = \mathcal{H}(q_{\boldsymbol{\phi},\boldsymbol{\psi}}(\rvx_T))$. Then, we have
\begin{align}
\label{eq:thm1-2}
\mathcal{H}(q_{\boldsymbol{\phi},\boldsymbol{\psi}}(\rvx_0,\rvx_T)) =&\mathcal{H}(q_{\boldsymbol{\phi},\boldsymbol{\psi}}(\rvx_T,\rvx_T))+\int_{T}^{0} \frac{\partial \mathcal{H}_t(\rvx_t,\rvx_T)}{\partial t} \mathrm{d}t.
\end{align}

We can expand the integrand of \cref{eq:thm1-2} as follows.
\begin{align}
\frac{\partial \mathcal{H}_t(\rvx_t,\rvx_T)}{\partial t} &= \frac{\partial}{\partial t}\Big [ - \int q_{\boldsymbol{\phi},\boldsymbol{\psi}}(\rvx_t,\rvx_T) \log q_{\boldsymbol{\phi},\boldsymbol{\psi}}(\rvx_t,\rvx_T) \mathrm{d}\rvx_t \mathrm{d}\rvx_T \Big] \\
&= \frac{\partial}{\partial t}\Big [ - \int q_{\boldsymbol{\phi},\boldsymbol{\psi}}(\rvx_T)q_{\boldsymbol{\phi},\boldsymbol{\psi}}(\rvx_t|\rvx_T) [ \log q_{\boldsymbol{\phi},\boldsymbol{\psi}}(\rvx_T)+\log q_{\boldsymbol{\phi},\boldsymbol{\psi}}(\rvx_t|\rvx_T)] \mathrm{d}\rvx_t \mathrm{d}\rvx_T \Big]\\
&=  - \int q_{\boldsymbol{\phi},\boldsymbol{\psi}}(\rvx_T)\frac{\partial}{\partial t} \big \{ q_{\boldsymbol{\phi},\boldsymbol{\psi}}(\rvx_t|\rvx_T) [ \log q_{\boldsymbol{\phi},\boldsymbol{\psi}}(\rvx_T)+\log q_{\boldsymbol{\phi},\boldsymbol{\psi}}(\rvx_t|\rvx_T)] \big \} \mathrm{d}\rvx_t \mathrm{d}\rvx_T\\
&=  - \mathbb{E}_{q_{\boldsymbol{\phi},\boldsymbol{\psi}}(\rvx_T)} \Big [ \int \frac{\partial}{\partial t} \big \{ q_{\boldsymbol{\phi},\boldsymbol{\psi}}(\rvx_t|\rvx_T) [ \log q_{\boldsymbol{\phi},\boldsymbol{\psi}}(\rvx_T)+\log q_{\boldsymbol{\phi},\boldsymbol{\psi}}(\rvx_t|\rvx_T)] \big \} \mathrm{d}\rvx_t \Big  ] \label{eq:thm1-3}
\end{align}

We further expand the integration in the last term as follows.
\begin{align}
&\int \frac{\partial}{\partial t} \big \{ q_{\boldsymbol{\phi},\boldsymbol{\psi}}(\rvx_t|\rvx_T) [ \log q_{\boldsymbol{\phi},\boldsymbol{\psi}}(\rvx_T)+\log q_{\boldsymbol{\phi},\boldsymbol{\psi}}(\rvx_t|\rvx_T)] \big \} \mathrm{d}\rvx_t \\
&=\int \frac{\partial}{\partial t} \big \{ q_{\boldsymbol{\phi},\boldsymbol{\psi}}(\rvx_t|\rvx_T) \big \} [ \log q_{\boldsymbol{\phi},\boldsymbol{\psi}}(\rvx_T)+\log q_{\boldsymbol{\phi},\boldsymbol{\psi}}(\rvx_t|\rvx_T)]  +  q_{\boldsymbol{\phi},\boldsymbol{\psi}}(\rvx_t|\rvx_T) \frac{\partial \log q_{\boldsymbol{\phi},\boldsymbol{\psi}}(\rvx_t|\rvx_T)}{\partial t} \mathrm{d}\rvx_t \\
&=\int \frac{\partial}{\partial t} \big \{ q_{\boldsymbol{\phi},\boldsymbol{\psi}}(\rvx_t|\rvx_T) \big \} [ \log q_{\boldsymbol{\phi},\boldsymbol{\psi}}(\rvx_T)+\log q_{\boldsymbol{\phi},\boldsymbol{\psi}}(\rvx_t|\rvx_T)]  +  \frac{\partial q_{\boldsymbol{\phi},\boldsymbol{\psi}}(\rvx_t|\rvx_T)}{\partial t} \mathrm{d}\rvx_t \\
&=\int \frac{\partial}{\partial t} \big \{ q_{\boldsymbol{\phi},\boldsymbol{\psi}}(\rvx_t|\rvx_T) \big \} [ \log q_{\boldsymbol{\phi},\boldsymbol{\psi}}(\rvx_T)+\log q_{\boldsymbol{\phi},\boldsymbol{\psi}}(\rvx_t|\rvx_T)]\mathrm{d}\rvx_t   + \frac{\partial}{\partial t} \int  q_{\boldsymbol{\phi},\boldsymbol{\psi}}(\rvx_t|\rvx_T) \mathrm{d}\rvx_t\\
&=\int \frac{\partial}{\partial t} \big \{ q_{\boldsymbol{\phi},\boldsymbol{\psi}}(\rvx_t|\rvx_T) \big \} [ \log q_{\boldsymbol{\phi},\boldsymbol{\psi}}(\rvx_T)+\log q_{\boldsymbol{\phi},\boldsymbol{\psi}}(\rvx_t|\rvx_T)]\mathrm{d}\rvx_t  \label{eq:thm1-4} \\
&=\int \frac{\partial}{\partial t} \big \{ q_{\boldsymbol{\phi},\boldsymbol{\psi}}(\rvx_t|\rvx_T) \big \} \log q_{\boldsymbol{\phi},\boldsymbol{\psi}}(\rvx_T) \mathrm{d}\rvx_t + \int \frac{\partial}{\partial t} \big \{ q_{\boldsymbol{\phi},\boldsymbol{\psi}}(\rvx_t|\rvx_T) \big \} \log q_{\boldsymbol{\phi},\boldsymbol{\psi}}(\rvx_t|\rvx_T) \mathrm{d}\rvx_t \\
&= \log q_{\boldsymbol{\phi},\boldsymbol{\psi}}(\rvx_T) \frac{\partial}{\partial t} \int   q_{\boldsymbol{\phi},\boldsymbol{\psi}}(\rvx_t|\rvx_T)  \mathrm{d}\rvx_t + \int \frac{\partial}{\partial t} \big \{ q_{\boldsymbol{\phi},\boldsymbol{\psi}}(\rvx_t|\rvx_T) \big \} \log q_{\boldsymbol{\phi},\boldsymbol{\psi}}(\rvx_t|\rvx_T) \mathrm{d}\rvx_t \\
&= \int \frac{\partial}{\partial t} \big \{ q_{\boldsymbol{\phi},\boldsymbol{\psi}}(\rvx_t|\rvx_T) \big \} \log q_{\boldsymbol{\phi},\boldsymbol{\psi}}(\rvx_t|\rvx_T) \mathrm{d}\rvx_t \label{eq:thm1-5}
\end{align}

Note that we use $\int   q_{\boldsymbol{\phi},\boldsymbol{\psi}}(\rvx_t|\rvx_T)  \mathrm{d}\rvx_t = 1$ in \cref{eq:thm1-4,eq:thm1-5}.

By eq. (51) in \citep{zhou2024denoising}, the Fokker-Plank equation for $q_{\boldsymbol{\phi},\boldsymbol{\psi}}(\rvx_t|\rvx_T)$ follows
\begin{align}
\frac{\partial}{\partial t} q_{\boldsymbol{\phi},\boldsymbol{\psi}}(\rvx_t|\rvx_T) = & - \nabla_{\rvx_t} \cdot \Big [ ( \mathbf{f} (\rvx_t, t) + g^2(t) \mathbf{h}(\rvx_t, t, \rvx_T, T) ) q_{\boldsymbol{\phi},\boldsymbol{\psi}}(\rvx_t | \rvx_T)  \Big ] \nonumber \\
& + \frac{1}{2} g^2(t) \nabla_{\rvx_t} \cdot \nabla_{\rvx_t} q_{\boldsymbol{\phi},\boldsymbol{\psi}}(\rvx_t|\rvx_T) \\
= & - \nabla_{\rvx_t} \cdot  [\tilde{\mathbf{f}}_{\boldsymbol{\phi},\boldsymbol{\psi}}(\rvx_t,t) q_{\boldsymbol{\phi},\boldsymbol{\psi}} (\rvx_t | \rvx_T)  ], \label{eq:thm1-6}
\end{align}
where $\tilde{\mathbf{f}}_{\boldsymbol{\phi},\boldsymbol{\psi}}(\rvx_t,t) := \mathbf{f} (\rvx_t, t) + g^2(t) \mathbf{h}(\rvx_t, t, \rvx_T, T) - \frac{1}{2} g^2(t) \nabla_{\rvx_t} \log q_{\boldsymbol{\phi},\boldsymbol{\psi}} (\rvx_t|\rvx_T)$.

Combining \cref{eq:thm1-3,eq:thm1-5,eq:thm1-6}, we have
\begin{align}
&\frac{\partial \mathcal{H}_t(\rvx_t,\rvx_T)}{\partial t}\\
&=  - \mathbb{E}_{q_{\boldsymbol{\phi},\boldsymbol{\psi}}(\rvx_T)} \Big [ \int - \nabla_{\rvx_t} \cdot  [\tilde{\mathbf{f}}_{\boldsymbol{\phi},\boldsymbol{\psi}}(\rvx_t,t) q_{\boldsymbol{\phi},\boldsymbol{\psi}} (\rvx_t | \rvx_T)  ] \log q_{\boldsymbol{\phi},\boldsymbol{\psi}}(\rvx_t|\rvx_T) \mathrm{d}\rvx_t \Big ] \\
&= \mathbb{E}_{q_{\boldsymbol{\phi},\boldsymbol{\psi}}(\rvx_T)} \Big [ \int \nabla_{\rvx_t} \cdot  [\tilde{\mathbf{f}}_{\boldsymbol{\phi},\boldsymbol{\psi}}(\rvx_t,t) q_{\boldsymbol{\phi},\boldsymbol{\psi}} (\rvx_t | \rvx_T)  ] \log q_{\boldsymbol{\phi},\boldsymbol{\psi}}(\rvx_t|\rvx_T) \mathrm{d}\rvx_t \Big ] \\
&= \mathbb{E}_{q_{\boldsymbol{\phi},\boldsymbol{\psi}}(\rvx_T)} \Big [ \tilde{\mathbf{f}}_{\boldsymbol{\phi},\boldsymbol{\psi}}(\rvx_t,t) q_{\boldsymbol{\phi},\boldsymbol{\psi}} (\rvx_t | \rvx_T)   \log q_{\boldsymbol{\phi},\boldsymbol{\psi}}(\rvx_t|\rvx_T) \nonumber\\
& \quad\quad\quad\quad\quad\quad - \int q_{\boldsymbol{\phi},\boldsymbol{\psi}} (\rvx_t | \rvx_T) \tilde{\mathbf{f}}_{\boldsymbol{\phi},\boldsymbol{\psi}}(\rvx_t,t)^T \nabla_{\rvx_t} \log q_{\boldsymbol{\phi},\boldsymbol{\psi}}(\rvx_t|\rvx_T)   \mathrm{d}\rvx_t \Big ] \\
&= \mathbb{E}_{q_{\boldsymbol{\phi},\boldsymbol{\psi}}(\rvx_T)} \Big [- \int q_{\boldsymbol{\phi},\boldsymbol{\psi}} (\rvx_t | \rvx_T) \tilde{\mathbf{f}}_{\boldsymbol{\phi},\boldsymbol{\psi}}(\rvx_t,t)^T  \nabla_{\rvx_t} \log q_{\boldsymbol{\phi},\boldsymbol{\psi}}(\rvx_t|\rvx_T)   \mathrm{d}\rvx_t \Big ] \\
&= \mathbb{E}_{q_{\boldsymbol{\phi},\boldsymbol{\psi}}(\rvx_T)} \Big [- \int\{ \mathbf{f} (\rvx_t, t) + g^2(t) \mathbf{h}(\rvx_t, t, \rvx_T, T) - \frac{1}{2} g^2(t) \nabla_{\rvx_t} \log q_{\boldsymbol{\phi},\boldsymbol{\psi}} (\rvx_t|\rvx_T) \}^T \nonumber\\
& \quad\quad\quad\quad\quad\quad \nabla_{\rvx_t} \log q_{\boldsymbol{\phi},\boldsymbol{\psi}}(\rvx_t|\rvx_T)  q_{\boldsymbol{\phi},\boldsymbol{\psi}} (\rvx_t | \rvx_T)  \mathrm{d}\rvx_t \Big ] \\
&= \mathbb{E}_{q_{\boldsymbol{\phi},\boldsymbol{\psi}}(\rvx_t,\rvx_T)} \Big [\{ -\mathbf{f} (\rvx_t, t) - g^2(t) \mathbf{h}(\rvx_t, t, \rvx_T, T) + \frac{1}{2} g^2(t) \nabla_{\rvx_t} \log q_{\boldsymbol{\phi},\boldsymbol{\psi}} (\rvx_t|\rvx_T) \}^T \nonumber\\
& \quad\quad\quad\quad\quad\quad\quad \nabla_{\rvx_t} \log q_{\boldsymbol{\phi},\boldsymbol{\psi}}(\rvx_t|\rvx_T)   \Big ] \\
&= \mathbb{E}_{q_{\boldsymbol{\phi},\boldsymbol{\psi}}(\rvx_t,\rvx_T)} \Big [ \frac{1}{2} g^2(t)||  \nabla_{\rvx_t} \log q_{\boldsymbol{\phi},\boldsymbol{\psi}} (\rvx_t|\rvx_T) ||_2^2 \nonumber \\
& \quad\quad\quad\quad\quad\quad\quad- \{ \mathbf{f} (\rvx_t, t) + g^2(t) \mathbf{h}(\rvx_t, t, \rvx_T, T)\}^T \nabla_{\rvx_t} \log q_{\boldsymbol{\phi},\boldsymbol{\psi}}(\rvx_t|\rvx_T)  \Big ].
\end{align}

Therefore, the joint entropy function $\mathcal{H}(q_{\boldsymbol{\phi},\boldsymbol{\psi}}(\rvx_0,\rvx_T))$ can be expressed as
\begin{align}
\mathcal{H}(q_{\boldsymbol{\phi},\boldsymbol{\psi}}&(\rvx_0,\rvx_T))= \mathcal{H}(q_{\boldsymbol{\phi},\boldsymbol{\psi}}(\rvx_T)) + \int_{T}^{0} \mathbb{E}_{q^t_{\boldsymbol{\phi},\boldsymbol{\psi}}(\rvx_t,\rvx_T)}\Big [ \frac{1}{2} g^2(t) ||\nabla_{\rvx_t}\log q^t_{\boldsymbol{\phi},\boldsymbol{\psi}}(\rvx_t|\rvx_T)||_2^2 \nonumber \\
& - \mathbf{f}(\rvx_t,t)^T \nabla_{\rvx_t}\log q^t_{\boldsymbol{\phi},\boldsymbol{\psi}}(\rvx_t|\rvx_T) - g^2(t)\mathbf{h}(\rvx_t,t,\rvx_T,T)^T \nabla_{\rvx_t}\log q^t_{\boldsymbol{\phi},\boldsymbol{\psi}}(\rvx_t|\rvx_T)  \Big ] \mathrm{d}t.
\end{align}

We can re-write the above equation as follows.
\begin{align}
&\int_{0}^{T} \mathbb{E}_{q^t_{\boldsymbol{\phi},\boldsymbol{\psi}}(\rvx_t,\rvx_T)} [ g^2(t) ||\nabla_{\rvx_t}\log q^t_{\boldsymbol{\phi},\boldsymbol{\psi}}(\rvx_t|\rvx_T)||_2^2] \mathrm{d}t\\
=& - 2 \mathcal{H}(q_{\boldsymbol{\phi},\boldsymbol{\psi}}(\rvx_0|\rvx_T))\\
&+ \int_{0}^{T} \mathbb{E}_{q^t_{\boldsymbol{\phi},\boldsymbol{\psi}}(\rvx_t,\rvx_T)}[  2\mathbf{f}(\rvx_t,t)^T \nabla_{\rvx_t}\log q^t_{\boldsymbol{\phi},\boldsymbol{\psi}}(\rvx_t|\rvx_T) + 2g^2(t)\mathbf{h}(\rvx_t,t,\rvx_T,T)^T \nabla_{\rvx_t}\log q^t_{\boldsymbol{\phi},\boldsymbol{\psi}}(\rvx_t|\rvx_T)  ] \mathrm{d}t \nonumber\\
=&- 2 \mathcal{H}(q_{\boldsymbol{\phi},\boldsymbol{\psi}}(\rvx_0|\rvx_T))+ 2\int_{0}^{T} \mathbb{E}_{q^t_{\boldsymbol{\phi},\boldsymbol{\psi}}(\rvx_t,\rvx_T)}[\nabla_{\rvx_t} \cdot \{ \mathbf{f}(\rvx_t,t) + g^2(t)\mathbf{h}(\rvx_t,t,\rvx_T,T) \} ] \mathrm{d}t \label{eq:app_ym0}
\end{align}

Similar to the process above, we can obtain the following results for the following joint entropy function $\mathcal{H}(q_{\boldsymbol{\phi},\boldsymbol{\psi}}(\rvx_0,\rvx_t,\rvx_T)):= - \int q_{\boldsymbol{\phi},\boldsymbol{\psi}}(\rvx_0,\rvx_t,\rvx_T) \log q_{\boldsymbol{\phi},\boldsymbol{\psi}}(\rvx_0,\rvx_t,\rvx_T) \mathrm{d}\rvx_0 \mathrm{d}\rvx_t\mathrm{d}\rvx_T $.
\begin{align}
\mathcal{H}(q_{\boldsymbol{\phi},\boldsymbol{\psi}}(\rvx_0,\rvx_0,\rvx_T)) =\mathcal{H}(q_{\boldsymbol{\phi},\boldsymbol{\psi}}(\rvx_0,\rvx_T,\rvx_T))+\int_{T}^{0} \frac{\partial \mathcal{H}(\rvx_0,\rvx_t,\rvx_T)}{\partial t} \mathrm{d}t
\end{align}

In the following results, we utilize the Fokker-Plank equation for $q_{\boldsymbol{\phi},\boldsymbol{\psi}}(\rvx_t|\rvx_0,\rvx_T)$, which comes from eq. (49) in \citep{zhou2024denoising}:
\begin{align}
\frac{\partial}{\partial t} q_{\boldsymbol{\phi},\boldsymbol{\psi}}(\rvx_t|\rvx_0,\rvx_T) = & - \nabla_{\rvx_t} \cdot \Big [ ( \mathbf{f} (\rvx_t, t) + g^2(t) \mathbf{h}(\rvx_t, t, \rvx_T, T) ) q_{\boldsymbol{\phi},\boldsymbol{\psi}}(\rvx_t | \rvx_0, \rvx_T)  \Big ] \nonumber \\
& + \frac{1}{2} g^2(t) \nabla_{\rvx_t} \cdot \nabla_{\rvx_t} q_{\boldsymbol{\phi},\boldsymbol{\psi}}(\rvx_t|\rvx_0,\rvx_T) \\
= & - \nabla_{\rvx_t} \cdot  [\hat{\mathbf{f}}_{\boldsymbol{\phi},\boldsymbol{\psi}}(\rvx_t,t) q_{\boldsymbol{\phi},\boldsymbol{\psi}} (\rvx_t |\rvx_0, \rvx_T)  ],
\end{align}
where $\hat{\mathbf{f}}_{\boldsymbol{\phi},\boldsymbol{\psi}}(\rvx_t,t) := \mathbf{f} (\rvx_t, t) + g^2(t) \mathbf{h}(\rvx_t, t, \rvx_T, T) - \frac{1}{2} g^2(t) \nabla_{\rvx_t} \log q_{\boldsymbol{\phi},\boldsymbol{\psi}} (\rvx_t|\rvx_0,\rvx_T)$.

Then, we have
\begin{align}
0 = \int_{T}^{0} \mathbb{E}_{q^t_{\boldsymbol{\phi},\boldsymbol{\psi}}(\rvx_0,\rvx_t,\rvx_T)}\Big [ &\frac{1}{2} g^2(t) ||\nabla_{\rvx_t}\log q^t_{\boldsymbol{\phi},\boldsymbol{\psi}}(\rvx_t|\rvx_T,\rvx_0)||_2^2- \mathbf{f}(\rvx_t,t) \nabla_{\rvx_t}\log q^t_{\boldsymbol{\phi},\boldsymbol{\psi}}(\rvx_t|\rvx_T,\rvx_0) \nonumber \\
&- g^2(t)\mathbf{h}(\rvx_t,t,\rvx_T,T) \nabla_{\rvx_t}\log q^t_{\boldsymbol{\phi},\boldsymbol{\psi}}(\rvx_t|\rvx_T,\rvx_0)  \Big ] \mathrm{d}t,
\end{align}
where the left hand side is from $0=\mathcal{H}(q_{\boldsymbol{\phi},\boldsymbol{\psi}}(\rvx_0,\rvx_0,\rvx_T)) -\mathcal{H}(q_{\boldsymbol{\phi},\boldsymbol{\psi}}(\rvx_0,\rvx_T,\rvx_T))$, and right hand side is from $\int_{T}^{0} \frac{\partial \mathcal{H}(\rvx_0,\rvx_t,\rvx_T)}{\partial t} \mathrm{d}t$. We can further derive as follows.
\begin{align}
\int_{0}^{T} \mathbb{E}_{q^t_{\boldsymbol{\phi},\boldsymbol{\psi}}(\rvx_0,\rvx_t,\rvx_T)} &[ g^2(t) ||\nabla_{\rvx_t}\log q^t_{\boldsymbol{\phi},\boldsymbol{\psi}}(\rvx_t|\rvx_T,\rvx_0)||_2^2] \mathrm{d}t\nonumber \\
&=2\int_{0}^{T} \mathbb{E}_{q^t_{\boldsymbol{\phi},\boldsymbol{\psi}}(\rvx_0,\rvx_t,\rvx_T)}[\nabla_{\rvx_t} \cdot \{ \mathbf{f}(\rvx_t,t) + g^2(t)\mathbf{h}(\rvx_t,t,\rvx_T,T) \} ] \mathrm{d}t \label{eq:app_ym1}
\end{align}

Combining \cref{eq:app_ym0,eq:app_ym1}, we have 
\begin{align}
\int_{0}^{T} &\mathbb{E}_{q^t_{\boldsymbol{\phi},\boldsymbol{\psi}}(\rvx_t,\rvx_T)} [ g^2(t) ||\nabla_{\rvx_t}\log q^t_{\boldsymbol{\phi},\boldsymbol{\psi}}(\rvx_t|\rvx_T)||_2^2] \mathrm{d}t \nonumber \\
&= - 2 \mathcal{H}(q_{\boldsymbol{\phi},\boldsymbol{\psi}}(\rvx_0|\rvx_T))+\int_{0}^{T} \mathbb{E}_{q^t_{\boldsymbol{\phi},\boldsymbol{\psi}}(\rvx_0,\rvx_t,\rvx_T)} [ g^2(t) ||\nabla_{\rvx_t}\log q^t_{\boldsymbol{\phi},\boldsymbol{\psi}}(\rvx_t|\rvx_T,\rvx_0)||_2^2]\mathrm{d}t 
\end{align}

Combining all results, the score-matching objective $\mathcal{L}_{\text{SM}}$ can be expressed as
\begin{align}
\mathcal{L}_{\text{SM}}&=  \frac{1}{2}\int_0^T\mathbb{E}_{q^t_{\boldsymbol{\phi},\boldsymbol{\psi}}(\rvx_0,\rvx_t,\rvx_T)}[g^2(t) || \mathbf{s}_{\boldsymbol{\theta}}(\rvx_t,t,\rvx_T)||_2^2  + g^2(t) ||\nabla_{\rvx_t}\log q^t_{\boldsymbol{\phi},\boldsymbol{\psi}}(\rvx_t|\rvx_T,\rvx_0)||_2^2 
\nonumber \\
& \quad\quad\quad\quad\quad\quad\quad - 2g^2(t)  \mathbf{s}_{\boldsymbol{\theta}}(\rvx_t,t,\rvx_T)^T \nabla_{\rvx_t}\log q^t_{\boldsymbol{\phi},\boldsymbol{\psi}}(\rvx_t|\rvx_T,\rvx_0) ]\mathrm{d}t- \mathcal{H}(q_{\boldsymbol{\phi},\boldsymbol{\psi}}(\rvx_0|\rvx_T))\\
& = \frac{1}{2}\int_0^T{\mathbb{E}_{q^t_{\boldsymbol{\phi},\boldsymbol{\psi}}(\rvx_0,\rvx_t,\rvx_T)}[g^2(t)|| \mathbf{s}_{\boldsymbol{\theta}}(\rvx_t,t,\rvx_T) - \nabla_{\rvx_t}\log \tilde{q}_{t}(\rvx_t|\rvx_0,\rvx_T)||_2^2]}\mathrm{d}t- \mathcal{H}(q_{\boldsymbol{\phi},\boldsymbol{\psi}}(\rvx_0|\rvx_T))
\end{align}
The last equality comes from $q^t_{\boldsymbol{\phi},\boldsymbol{\psi}}(\rvx_t|\rvx_0,\rvx_T)=\tilde{q}_t(\rvx_t|\rvx_0,\rvx_T)$, which is based on the Doob's $h$-transform~\citep{doob1984classical,rogers2000diffusions,zhou2024denoising}.
Finally, we have
\begin{align}
\mathcal{L}_{\text{AE}}&=\mathcal{L}_{\text{SM}} + \mathcal{H}(q_{\boldsymbol{\phi},\boldsymbol{\psi}}(\rvx_0|\rvx_T))\\
&= \frac{1}{2}\int_0^T{\mathbb{E}_{q^t_{\boldsymbol{\phi},\boldsymbol{\psi}}(\rvx_0,\rvx_t,\rvx_T)}[g^2(t)|| \mathbf{s}_{\boldsymbol{\theta}}(\rvx_t,t,\rvx_T) - \nabla_{\rvx_t}\log \tilde{q}_{t}(\rvx_t|\rvx_0,\rvx_T)||_2^2]}\mathrm{d}t.
\end{align}


From here, we show that the objective $\mathcal{L}_{\text{AE}}$ is equivalent to the reconstruction objective. Assume that the forward SDE in \cref{eq:for} is a linear SDE in terms of $\rvx_t$ (e.g. VP~\citep{ho2020denoising}, VE~\citep{song2021scorebased}). Then the transition kernel $\tilde{q}(\rvx_t|\rvx_0)$ becomes Gaussian distribution. Then, we can represent reparametrized form $\rvx_t=\alpha_t\rvx_0+\sigma_t\boldsymbol{\epsilon}$, where $\alpha_t$ and $\sigma_t$ are time-dependent constants determined by drift $\mathbf{f}$ and volatility $g$, and $\boldsymbol{\epsilon} \sim \mathcal{N}(\mathbf{0},\mathbf{I})$. The time-dependent constant signal-to-noise ratio $SNR(t):=\frac{\alpha_t^2}{\sigma_t^2}$ often define to discuss on diffusion process~\citep{kingma2021variational}. We define SNR ratio, $R(t):=\frac{SNR(T)}{SNR(t)}$ for convenient derivation.

\citet{zhou2024denoising} show the exact form of $\tilde{q}_{t}(\rvx_t|\rvx_0,\rvx_T):=\mathcal{N}(\hat{\mu}_t,\hat{\sigma}^2_t\mathbf{I})$, where $\hat{\mu}_t=R(t)\frac{\alpha_t}{\alpha_T}\rvx_T+\alpha_t\rvx_0(1-R(t))$ and $\hat{\sigma_t}=\sigma_t\sqrt{1-R(t)}$. This Gaussian form determines the exact analytic form of the score function $\nabla_{\rvx_t}\log \tilde{q}_{t}(\rvx_t|\rvx_0,\rvx_T)$. We plug this into our objective $\mathcal{L}_{\text{AE}}$.

\begin{align}
\mathcal{L}_{\textrm{AE}}&=\frac{1}{2}\int_0^T{\mathbb{E}_{q^t_{\boldsymbol{\phi},\boldsymbol{\psi}}(\rvx_0,\rvx_t,\rvx_T)}[g^2(t)|| \mathbf{s}_{\boldsymbol{\theta}}(\rvx_t,t,\rvx_T) - \nabla_{\rvx_t}\log \tilde{q}_{t}(\rvx_t|\rvx_0,\rvx_T)||_2^2]}\mathrm{d}t\\
&=\frac{1}{2}\int_0^T{\mathbb{E}_{q^t_{\boldsymbol{\phi},\boldsymbol{\psi}}(\rvx_0,\rvx_t,\rvx_T)}[g^2(t)|| \mathbf{s}_{\boldsymbol{\theta}}(\rvx_t,t,\rvx_T) - \frac{-\rvx_t+(R(t)\frac{\alpha_t}{\alpha_T}\rvx_T+\alpha_t\rvx_0(1-R(t)))}{\sigma^2_t (1-R(t))}||_2^2]}\mathrm{d}t\\
&=\frac{1}{2}\int_0^T{\mathbb{E}_{q^t_{\boldsymbol{\phi},\boldsymbol{\psi}}(\rvx_0,\rvx_t,\rvx_T)}[\lambda(t)|| \mathbf{x}^{0}_{\boldsymbol{\theta}}(\rvx_t,t,\rvx_T) - \rvx_0||_2^2]}\mathrm{d}t,
\end{align}
where 
\begin{align}
&\lambda(t)=\frac{\alpha_t}{\sigma^2_t} g^2(t),\\ 
&\mathbf{x}^{0}_{\boldsymbol{\theta}}(\rvx_t,t,\rvx_T):=\alpha(t)\rvx_t+\beta(t)\rvx_T+\gamma(t)\mathbf{s}_{\boldsymbol{\theta}}(\rvx_t,t,\rvx_T),\\ 
&\alpha(t)=\frac{1}{\alpha_t(1-R(t))},\quad \beta(t)=-\frac{R(t)}{\alpha_T(1-R(t))},\quad\gamma(t)=\frac{\sigma^2_t}{\alpha_t}.
\end{align}

\end{proof}
\subsection{Proof of Theorem \ref{thm:2}}

\label{sec:A1.2}
\thmb*
\begin{proof}


From the data processing inequality with our graphical model, we have the following result, similar to eq. (14) in \citep{song2021denoising}.
\begin{align}
D_{\textrm{KL}}(q_{\text{data}}(\rvx_0)||p_{\boldsymbol{\psi},\boldsymbol{\theta}}(\rvx_0))\leq D_{\textrm{KL}}(q_{\boldsymbol{\phi},\boldsymbol{\psi}}(\rvx_{0:T},\rvz)||p_{\boldsymbol{\psi},\boldsymbol{\theta}}(\rvx_{0:T},\rvz))\label{eq:thm1-7}
\end{align}

Also, the chain rule of KL divergences, we have
\begin{align}
&D_{\textrm{KL}}(q_{\boldsymbol{\phi},\boldsymbol{\psi}}(\rvx_{0:T},\rvz)||p_{\boldsymbol{\psi},\boldsymbol{\theta}}(\rvx_{0:T},\rvz)) \\
&=  D_{\textrm{KL}}(q_{\boldsymbol{\phi},\boldsymbol{\psi}}(\rvx_{T},\rvz)||p_{\boldsymbol{\psi},\boldsymbol{\theta}}(\rvx_{T},\rvz)) + \mathbb{E}_{q_{\boldsymbol{\phi},\boldsymbol{\psi}}(\rvx_T,\rvz)} [ D_{\textrm{KL}}(\mu_{\boldsymbol{\phi},\boldsymbol{\psi}}(\cdot|\rvx_{T},\rvz)||\nu_{\boldsymbol{\theta},\boldsymbol{\psi}}(\cdot|\rvx_{T},\rvz)) ],\label{eq:thm1-8}
\end{align}
where $\mu_{\boldsymbol{\phi}, \boldsymbol{\psi}}$ and $\nu_{\boldsymbol{\theta},\boldsymbol{\psi}}$ are the path measures of the SDEs in \cref{eq:mu_sde,eq:nu_sde}, respectively:
\begin{align}
 \mathrm{d}\rvx_t &= [\mathbf{f}(\rvx_t, t)+g^2(t)\mathbf{h}(\rvx_t,t,\mathbf{y},T)]\mathrm{d}t + g(t)\mathrm{d}\mathbf{w}_t, \quad \rvx_0 \sim q_{\textrm{data}}(\rvx_0), \quad \rvx_T \sim q_{\boldsymbol{\phi},\boldsymbol{\psi}} (\rvx_T|\rvx_0), \label{eq:mu_sde} \\
 \mathrm{d}\rvx_t &= [\mathbf{f}(\rvx_t, t)-g^2(t) [ \nabla_{\rvx_t} \log p_{\boldsymbol{\theta}}(\rvx_t|\rvx_T) - \mathbf{h}(\rvx_t,t,\mathbf{y},T)] ]\mathrm{d}t + g(t)\mathrm{d}\bar{\mathbf{w}}_t, \quad \rvx_T \sim p_{\boldsymbol{\psi}} (\rvx_T). \label{eq:nu_sde}
\end{align}



By our graphical modeling, $\rvz$ is independent of $\{  \rvx_t \}$ given $\rvx_T$. Therefore, we have
\begin{align}
& \mathbb{E}_{q_{\boldsymbol{\phi},\boldsymbol{\psi}}(\rvx_T,\rvz)} [ D_{\textrm{KL}}(\mu_{\boldsymbol{\phi},\boldsymbol{\psi}}(\cdot|\rvx_{T},\rvz)||\nu_{\boldsymbol{\theta}}(\cdot|\rvx_{T},\rvz)) ]=\mathbb{E}_{q_{\boldsymbol{\phi},\boldsymbol{\psi}}(\rvx_T)} [ D_{\textrm{KL}}(\mu_{\boldsymbol{\phi},\boldsymbol{\psi}}(\cdot|\rvx_{T})||\nu_{\boldsymbol{\theta}}(\cdot|\rvx_{T})) ], 
\end{align}
where $\mu_{\boldsymbol{\phi}, \boldsymbol{\psi}}(\cdot|\rvx_T)$ and $\nu_{\boldsymbol{\theta}}(\cdot|\rvx_T)$ are the path measures of the SDEs in \cref{eq:mu_sde2,eq:nu_sde2}, respectively:
\begin{align}
 \mathrm{d}\rvx_t &= [\mathbf{f}(\rvx_t, t)-g^2(t) [ \nabla_{\rvx_t} \log q_{\boldsymbol{\phi},\boldsymbol{\psi}}(\rvx_t|\rvx_T) - \mathbf{h}(\rvx_t,t,\mathbf{y},T)] ]\mathrm{d}t + g(t)\mathrm{d}\bar{\mathbf{w}}_t, \quad \rvx(T) = \rvx_T, \label{eq:mu_sde2} \\
 \mathrm{d}\rvx_t &= [\mathbf{f}(\rvx_t, t)-g^2(t) [ \nabla_{\rvx_t} \log p_{\boldsymbol{\theta}}(\rvx_t|\rvx_T) - \mathbf{h}(\rvx_t,t,\mathbf{y},T)] ]\mathrm{d}t + g(t)\mathrm{d}\bar{\mathbf{w}}_t, \quad \rvx(T) = \rvx_T \label{eq:nu_sde2}
\end{align}

Similar to eq. (17) in \citep{song2021denoising}, this KL divergence can be expressed using the Girsanov theorem~\citep{oksendal2013stochastic} and martingale property.
\begin{align}
&D_{\textrm{KL}}(\mu_{\boldsymbol{\phi},\boldsymbol{\psi}}(\cdot|\rvx_{T})||\nu_{\boldsymbol{\theta}}(\cdot|\rvx_{T})) =\frac{1}{2} \int_0^T \mathbb{E}_{q_{\boldsymbol{\phi},\boldsymbol{\psi}}(\rvx_t| \rvx_T)}[g^2(t)|| \mathbf{s}_{\boldsymbol{\theta}}(\rvx_t,t,\rvx_T) - \nabla_{\rvx_t}\log q^t_{\boldsymbol{\phi},\boldsymbol{\psi}}(\rvx_t|\rvx_T)||_2^2]\mathrm{d}t \label{eq:thm1-9}
\end{align}


From \cref{eq:thm1-7,eq:thm1-8,eq:thm1-9} and \cref{thm:1}, we have:
\begin{align}
&D_{\textrm{KL}}(q_{\text{data}}(\rvx_0)||p_{\boldsymbol{\psi},\boldsymbol{\theta}}(\rvx_0))\leq  D_{\textrm{KL}}(q_{\boldsymbol{\phi},\boldsymbol{\psi}}(\rvx_{T},\rvz)||p_{\boldsymbol{\psi},\boldsymbol{\theta}}(\rvx_{T},\rvz)) + \mathcal{L}_{\textrm{AE}} - \mathcal{H}(q_{\boldsymbol{\phi},\boldsymbol{\psi}}(\rvx_0|\rvx_T)) \label{eq:thm1-10}
\end{align}

Furthermore, the first and third terms of RHS in \cref{eq:thm1-10} can be expressed as follows.
\begin{align}
&D_{\textrm{KL}}(q_{\boldsymbol{\phi},\boldsymbol{\psi}}(\rvx_{T},\rvz)||p_{\boldsymbol{\psi},\boldsymbol{\theta}}(\rvx_{T},\rvz))- \mathcal{H}(q_{\boldsymbol{\phi},\boldsymbol{\psi}}(\rvx_0|\rvx_T)) \\
&= \int q_{\boldsymbol{\phi},\boldsymbol{\psi}}(\rvx_{T},\rvz)) \log \frac{q_{\boldsymbol{\phi},\boldsymbol{\psi}}(\rvx_{T},\rvz)}{p_{\boldsymbol{\psi},\boldsymbol{\theta}}(\rvx_{T},\rvz)}\mathrm{d}\rvx_T\mathrm{d}\rvz +\int q_{\boldsymbol{\phi},\boldsymbol{\psi}}(\rvx_0,\rvx_{T}) \log q_{\boldsymbol{\phi},\boldsymbol{\psi}}(\rvx_0|\rvx_{T}) \mathrm{d}\rvx_0 \mathrm{d}\rvx_T\\
&= \int q_{\boldsymbol{\phi},\boldsymbol{\psi}}(\rvx_0,\rvx_{T},\rvz) \Big [ \log \frac{q_{\boldsymbol{\phi},\boldsymbol{\psi}}(\rvx_{T},\rvz)}{p_{\boldsymbol{\psi},\boldsymbol{\theta}}(\rvx_{T},\rvz)} + \log q_{\boldsymbol{\phi},\boldsymbol{\psi}}(\rvx_0|\rvx_{T}) \Big ] \mathrm{d}\rvx_0 \mathrm{d}\rvx_T \mathrm{d}\rvz\\
&= \int q_{\boldsymbol{\phi},\boldsymbol{\psi}}(\rvx_0,\rvx_{T},\rvz) \Big [ \log \frac{q_{\boldsymbol{\phi},\boldsymbol{\psi}}(\rvx_{T})q_{\boldsymbol{\psi}}(\rvz|\rvx_{T})}{p_{\boldsymbol{\psi}}(\rvx_{T})p_{\boldsymbol{\psi}}(\rvz|\rvx_{T})} + \log q_{\boldsymbol{\phi},\boldsymbol{\psi}}(\rvx_0|\rvx_{T}) \Big ] \mathrm{d}\rvx_0 \mathrm{d}\rvx_T \mathrm{d}\rvz\\
&= \int q_{\boldsymbol{\phi},\boldsymbol{\psi}}(\rvx_0,\rvx_{T},\rvz) \Big [ \log \frac{q_{\boldsymbol{\phi},\boldsymbol{\psi}}(\rvx_{T})}{p_{\boldsymbol{\psi}}(\rvx_{T})} + \log q_{\boldsymbol{\phi},\boldsymbol{\psi}}(\rvx_0|\rvx_{T}) \Big ] \mathrm{d}\rvx_0 \mathrm{d}\rvx_T \mathrm{d}\rvz\\
&= \int q_{\boldsymbol{\phi},\boldsymbol{\psi}}(\rvx_0,\rvx_{T}) \Big [ \log \frac{q_{\boldsymbol{\phi},\boldsymbol{\psi}}(\rvx_{T})q_{\boldsymbol{\phi},\boldsymbol{\psi}}(\rvx_0|\rvx_{T})}{p_{\boldsymbol{\psi}}(\rvx_{T})} \big ] \mathrm{d}\rvx_0 \mathrm{d}\rvx_T \\
&= \int q_{\boldsymbol{\phi},\boldsymbol{\psi}}(\rvx_0,\rvx_{T}) \Big [ \log \frac{q_{\text{data}}(\rvx_{0})q_{\boldsymbol{\phi},\boldsymbol{\psi}}(\rvx_T|\rvx_{0})}{p_{\boldsymbol{\psi}}(\rvx_{T})} \big ] \mathrm{d}\rvx_0 \mathrm{d}\rvx_T \\
&= \int q_{\text{data}}(\rvx_0) q_{\boldsymbol{\phi},\boldsymbol{\psi}}(\rvx_T|\rvx_{0}) \Big [ \log \frac{q_{\boldsymbol{\phi},\boldsymbol{\psi}}(\rvx_T|\rvx_{0})}{p_{\boldsymbol{\psi}}(\rvx_{T})} + \log q_{\text{data}} (\rvx_0) \big ] \mathrm{d}\rvx_0 \mathrm{d}\rvx_T \\
&= \mathbb{E}_{q_{\text{data}}(\rvx_0)} [ D_{\textrm{KL}}(q_{\boldsymbol{\phi},\boldsymbol{\psi}}(\rvx_{T}|\rvx_0)||p_{\boldsymbol{\psi}}(\rvx_{T})) ] - \mathcal{H}(q_{\text{data}}(\rvx_0))\\
&= \mathcal{L}_{\textrm{PR}}- \mathcal{H}(q_{\text{data}}(\rvx_0))
\end{align}

To sum up, we have
\begin{align}
D_{\textrm{KL}}(q_{\text{data}}(\rvx_0)||p_{\boldsymbol{\psi},\boldsymbol{\theta}}(\rvx_0))\leq \mathcal{L}_{\textrm{AE}}+\mathcal{L}_{\textrm{PR}}-\mathcal{H}(q_{\text{data}}(\rvx_0)).
\end{align}

    
\end{proof}

\subsection{Prior Optimization Objective}
\label{sec:A1.3}
This section explains the details of the prior related objective function mentioned in \Cref{sec:3.3.2}. The proposed objective is $\mathcal{L}_{\textrm{PR}}$ as shown in \cref{eq:prior_app}.
\begin{align}
\mathcal{L}_{\textrm{PR}}=&\mathbb{E}_{q_{\text{data}}(\rvx_0)}[D_{\text{KL}}(q_{\boldsymbol{\phi},\boldsymbol{\psi}}(\rvx_T|\rvx_0)||p_{\boldsymbol{\psi}}(\rvx_T))] \label{eq:prior_app}
\end{align}

To optimize this term, we fix the parameters of the encoder ($\boldsymbol{\phi} \rightarrow \boldsymbol{\phi}^* $), the decoder ($\boldsymbol{\psi} \rightarrow \boldsymbol{\psi}^* $), and score network ($\boldsymbol{\theta} \rightarrow \boldsymbol{\theta}^*$), which is optimized by $\mathcal{L}_{AE}$. And we newly parameterize the generative prior $p_{\text{prior}}(\rvz) \rightarrow p_{\boldsymbol{\omega}}(\rvz)$, so the generative endpoint distribution becomes $p_{\boldsymbol{\psi}}(\rvx_T) \rightarrow p_{\boldsymbol{\psi}^*,\boldsymbol{\omega}}(\rvx_T)$. We utilize MLP-based latent diffusion models following~\citep{preechakul2022diffusion,zhang2022unsupervised}.

The objective function in \cref{eq:prior_app} with respect to $\boldsymbol{\omega}$ is described in \cref{prior:init} and extends to \cref{prior:final} with equality. \Cref{prior:ours} is derived from the same optimality condition. In other words, it reduces the problem of training an unconditional generative prior $p_{\omega}(\rvz)$ to matching the aggregated posterior distribution $q_{\boldsymbol{\phi}^*}(\rvz)$.

\begin{align}
& \arg\min_{\boldsymbol{\omega}}{\mathbb{E}_{q_{\text{data}}(\rvx_0)}[D_{\text{KL}}(q_{\boldsymbol{\phi}^*,\boldsymbol{\psi}^*}(\rvx_T|\rvx_0)||p_{\boldsymbol{\psi}^*,\omega}(\rvx_T))]} \label{prior:init} \\
\Leftrightarrow& \arg\min_{\boldsymbol{\omega}}{\int q_{\text{data}}(\rvx_0)q_{\boldsymbol{\phi}^*,\boldsymbol{\psi}^*}(\rvx_T|\rvx_0)\log{\frac{q_{\boldsymbol{\phi}^*,\boldsymbol{\psi}^*}(\rvx_T|\rvx_0)}{p_{\boldsymbol{\psi}^*,\boldsymbol{\omega}}(\rvx_T)}}} \mathrm{d}\rvx_0\mathrm{d}\rvx_T \\
\Leftrightarrow& \arg\min_{\boldsymbol{\omega}}{D_{\text{KL}}(q_{\boldsymbol{\phi}^*,\boldsymbol{\psi}^*}(\rvx_T)||p_{\boldsymbol{\psi}^*,\boldsymbol{\omega}}(\rvx_T))} + C \label{prior:final}\\ 
\Leftrightarrow& \arg\min_{\boldsymbol{\omega}}{D_{\text{KL}}(q_{\boldsymbol{\phi}^*}(\rvz)||p_{\boldsymbol{\omega}}(\rvz))} \label{prior:ours}
\end{align}

\subsection{Mutual Information Analysis}
\label{sec:A1.4}
\citet{alemi2018fixing} shows the \textit{distortion}; reconstruction error with inferred $\rvz$ is the variational bound of mutual information between $\rvx_0$ and $\rvz$ in the autoencoding framework. We explain the functional form of \textit{distortion} in both the auxiliary encoder framework (\Cref{sec:A1.4.1}) and DBAE (\Cref{sec:A1.4.2}).

\subsubsection{Auxiliary encoder framework}
\label{sec:A1.4.1}
In the auxiliary encoder framework (e.g., DiffAE~\citep{preechakul2022diffusion}), the \textit{distortion} $:= \mathbb{E}_{q_{\text{data}}(\rvx_0),q_{\boldsymbol{\phi}}(\rvz|\rvx_0)}[-\log{p_{\boldsymbol{\theta}}(\rvx_0|\rvz)}]$  and mutual information $MI(\rvx_0,\rvz):=\mathbb{E}_{q_{\boldsymbol{\phi}}(\rvx_0,\rvz)}[\log{\frac{q_{\boldsymbol{\phi}}(\rvx_0,\rvz)}{q_{\text{data}}(\rvx_0)q_{\boldsymbol{\phi}}(\rvz)}}]$ has a relation
\begin{align}
-\mathbb{E}_{q_{\text{data}}(\rvx_0),q_{\boldsymbol{\phi}}(\rvz|\rvx_0)}[-\log{p_{\boldsymbol{\theta}}(\rvx_0|\rvz)}] + \mathcal{H}(q_{\text{data}}(\rvx_0)) \leq MI(\rvx_0,\rvz), \label{eq:distortion_app} 
\end{align}
where $p_{\boldsymbol{\theta}}(\rvx_0|\rvz)=\int{p_{\text{prior}}(\rvx_T)p^{\text{ODE}}_{\boldsymbol{\theta}}(\rvx_0|\rvz,\rvx_T)}d\rvx_T$, when this framework reconstruct only with inferred $\rvz$.

We have the followings
\begin{align}
&\log{p_{\boldsymbol{\theta}}(\rvx_0|\rvz)}  \\
&=\log{\int{p_{\text{prior}}(\rvx_T)p^{\text{ODE}}_{\boldsymbol{\theta}}(\rvx_0|\rvz,\rvx_T)}d\rvx_T} \\
&=\log{\int{p_{\text{prior}}(\rvx_T)p^{\text{ODE}}_{\boldsymbol{\theta}}(\rvx_0|\rvz,\rvx_T)\frac{q^{\text{ODE}}_{\boldsymbol{\theta}}(\rvx_T|\rvz,\rvx_0)}{q^{\text{ODE}}_{\boldsymbol{\theta}}(\rvx_T|\rvz,\rvx_0)}}d\rvx_T} \\
&\geq \int{q^{\text{ODE}}_{\boldsymbol{\theta}}(\rvx_T|\rvz,\rvx_0)\log{\frac{p_{\text{prior}}(\rvx_T)p^{\text{ODE}}_{\boldsymbol{\theta}}(\rvx_0|\rvz,\rvx_T)}{q^{\text{ODE}}_{\boldsymbol{\theta}}(\rvx_T|\rvz,\rvx_0)}}}d\rvx_T \\ 
&=\mathbb{E}_{q^{\text{ODE}}_{\boldsymbol{\theta}}(\rvx_T|\rvx_0,\rvz)}[\log{p^{\text{ODE}}_{\boldsymbol{\theta}}(\rvx_0|\rvz,\rvx_T)}]-D_{KL}(q^{\text{ODE}}_{\boldsymbol{\theta}}(\rvx_T|\rvx_0,\rvz)||p_{\text{prior}}(\rvx_T)).\\
&= \int{q^{\text{ODE}}_{\boldsymbol{\theta}}(\rvx_T|\rvz,\rvx_0)\log{\frac{p_{\text{prior}}(\rvx_T)\cancel{p^{\text{ODE}}_{\boldsymbol{\theta}}(\rvx_0|\rvz,\rvx_T)}}{\cancel{q^{\text{ODE}}_{\boldsymbol{\theta}}(\rvx_T|\rvz,\rvx_0)}}}}d\rvx_T \\
&= \int{q^{\text{ODE}}_{\boldsymbol{\theta}}(\rvx_T|\rvz,\rvx_0)\log{p_{\text{prior}}(\rvx_T)}}d\rvx_T \\ 
&= -CE(q^{\text{ODE}}_{\boldsymbol{\theta}}(\rvx_T|\rvz,\rvx_0) || p_{\text{prior}}(\rvx_T))\label{eq:disortion_deri_diffae}
\end{align}

Note that $p^{\text{ODE}}_{\boldsymbol{\theta}}(\rvx_0|\rvz,\rvx_T)=q^{\text{ODE}}_{\boldsymbol{\theta}}(\rvx_T|\rvz,\rvx_0)$ because the deterministic coupling of ($\rvx_0,\rvx_T$) is given by the ODE in \cref{eq:zbackodeagain}. When the coupling ($\rvx_0$, $\rvx_T$) lies on the ODE path, both probabilities $p^{\text{ODE}}_{\boldsymbol{\theta}}(\rvx_0|\rvz,\rvx_T)$ and $q^{\text{ODE}}_{\boldsymbol{\theta}}(\rvx_T|\rvz,\rvx_0)$ become infinite. When the coupling ($\rvx_0$, $\rvx_T$) is outside the ODE path, both probabilities $p^{\text{ODE}}_{\boldsymbol{\theta}}(\rvx_0|\rvz,\rvx_T)$ and $q^{\text{ODE}}_{\boldsymbol{\theta}}(\rvx_T|\rvz,\rvx_0)$ become zero.
\begin{align}
& \mathrm{d}\rvx_t = [\mathbf{f}(\rvx_t, t)-\frac{1}{2}g^2(t) \mathbf{s}_{\boldsymbol{\theta}}(\rvx_t,\rvz,t)] \mathrm{d}t.\label{eq:zbackodeagain}
\end{align}
From \cref{eq:distortion_app} and \cref{eq:disortion_deri_diffae}, we have the following.
\begin{align}
\mathbb{E}_{q_{\text{data}}(\rvx_0),q_{\boldsymbol{\phi}}(\rvz|\rvx_0)}[-CE(q^{\text{ODE}}_{\boldsymbol{\theta}}(\rvx_T|\rvz,\rvx_0) || p_{\text{prior}}(\rvx_T))] + \mathcal{H}(q_{\text{data}}(\rvx_0))\leq MI(\rvx_0,\rvz) \label{eq:mi_diffae_final}
\end{align}
The discrepancy between $q^{\text{ODE}}_{\boldsymbol{\theta}}(\rvx_T|\rvx_0,\rvz)$ and $p_{\text{prior}}(\rvx_T)$ makes the lower bound of mutual information between $\rvx_0$ and $\rvz$ loose. This discrepancy is inevitable from the deterministic nature of $q^{\text{ODE}}_{\boldsymbol{\theta}}(\rvx_T|\rvz,\rvx_0)$.

This discrepancy is empirically observed in \Cref{tab:main2}, providing two cases of $\rvx_T$ draw (random $\rvx_T$, inferred $\rvx_T$) in the auxiliary encoder models. The reconstruction gap between (random $\rvx_T$, inferred $\rvx_T$) is significant in practice. However, the inference of $\rvx_T$ is computationally expensive and inflexible in terms of dimensionality. If we only consider $\rvz$ inference, the information leakage is inevitable due to the functional form of diffusion models with an auxiliary encoder.

\subsubsection{Diffusion Bridge AutoEncoders}
\label{sec:A1.4.2}
In the DBAE, the \textit{distortion} $:=\mathbb{E}_{q_{\text{data}}(\rvx_0),q_{\boldsymbol{\phi}}(\rvz|\rvx_0)}[-\log{p_{\boldsymbol{\theta, \psi}}(\rvx_0|\rvz)}]$ term and mutual information between $\rvx_0$ and $\rvz$ has relation in \cref{eq:distortion_dbae_app}.
\begin{align}
-\mathbb{E}_{q_{\text{data}}(\rvx_0),q_{\boldsymbol{\phi}}(\rvz|\rvx_0)}[-\log{p_{\boldsymbol{\theta, \psi}}(\rvx_0|\rvz)}]+ \mathcal{H}(q_{\text{data}}(\rvx_0)) \leq MI(\rvx_0,\rvz),\label{eq:distortion_dbae_app} 
\end{align}
where $p_{\boldsymbol{\theta, \psi}}(\rvx_0|\rvz)=\int{p_{\boldsymbol{\theta}}(\rvx_0|\rvx_T)p_{\boldsymbol{\psi}}(\rvx_T|\rvz)}d\rvx_T$. We have followings
\begin{align}
&\log{p_{\boldsymbol{\theta, \psi}}(\rvx_0|\rvz)}  \\
&=\log{\int{p_{\boldsymbol{\theta}}(\rvx_0|\rvx_T)p_{\boldsymbol{\psi}}(\rvx_T|\rvz)}d\rvx_T} \\
&=\log{\int{p_{\boldsymbol{\theta}}(\rvx_0|\rvx_T)p_{\boldsymbol{\psi}}(\rvx_T|\rvz)\frac{q_{\boldsymbol{\psi}}(\rvx_T|\rvz)}{q_{\boldsymbol{\psi}}(\rvx_T|\rvz)}}d\rvx_T} \\
&\geq \int{q_{\boldsymbol{\psi}}(\rvx_T|\rvz)\log{\frac{p_{\boldsymbol{\theta}}(\rvx_0|\rvx_T)p_{\boldsymbol{\psi}}(\rvx_T|\rvz)}{q_{\boldsymbol{\psi}}(\rvx_T|\rvz)}}}d\rvx_T \\
&=\mathbb{E}_{q_{\boldsymbol{\psi}}(\rvx_T|\rvz)}[\log{p_{\boldsymbol{\theta}}(\rvx_0|\rvx_T)}]-D_{KL}(q_{\boldsymbol{\psi}}(\rvx_T|\rvz)||p_{\boldsymbol{\psi}}(\rvx_T|\rvz))\label{eq:disortion_deri_dbae}
\end{align}
Since $D_{KL}(q_{\boldsymbol{\psi}}(\rvx_T|\rvz)||p_{\boldsymbol{\psi}}(\rvx_T|\rvz))=0$, we have followings from \cref{eq:distortion_dbae_app} and \cref{eq:disortion_deri_dbae}.
\begin{align}
\mathbb{E}_{q_{\text{data}}(\rvx_0),q_{\boldsymbol{\phi}}(\rvz|\rvx_0)}[\mathbb{E}_{q_{\boldsymbol{\psi}}(\rvx_T|\rvz)}[\log{p_{\boldsymbol{\theta}}(\rvx_0|\rvx_T)}]]+\mathcal{H}(q_{\text{data}}(\rvx_0)) \leq MI(\rvx_0,\rvz). \label{eq:mi_dbae_final}
\end{align}
Unlike in \cref{eq:mi_diffae_final}, the $\rvx_T$ related term does not hinder maximizing mutual information between $\rvx_0$ and $\rvz$. Moreover, the remaining term $\mathbb{E}_{q_{\text{data}}(\rvx_0),q_{\boldsymbol{\phi}}(\rvz|\rvx_0)}[\mathbb{E}_{q_{\boldsymbol{\psi}}(\rvx_T|\rvz)}[\log{p_{\boldsymbol{\theta}}(\rvx_0|\rvx_T)}]]$ can maximized by our training, as we explain in \Cref{thm:3}.

\subsection{Proof of Theorem \ref{thm:3}}
\label{sec:A1.5}
\thmc*
\begin{proof}
From data processing inequality similar in \cref{eq:thm1-7},
\begin{align}
\mathbb{E}_{q_{\boldsymbol{\phi},\boldsymbol{\psi}}(\rvx_T)} [ D_{\textrm{KL}}(q_{\boldsymbol{\phi},\boldsymbol{\psi}}(\rvx_0|\rvx_{T})||p_{\boldsymbol{\theta}}(\rvx_0|\rvx_{T})) ] \leq \mathbb{E}_{q_{\boldsymbol{\phi},\boldsymbol{\psi}}(\rvx_T)} [ D_{\textrm{KL}}(\mu_{\boldsymbol{\phi},\boldsymbol{\psi}}(\cdot|\rvx_{T})||\nu_{\boldsymbol{\theta}}(\cdot|\rvx_{T})) ] \label{eq:thm10-1}
\end{align}

The LHS of \cref{eq:thm10-1} becomes followings,
\begin{align}
&\mathbb{E}_{q_{\boldsymbol{\phi},\boldsymbol{\psi}}(\rvx_T)}[ D_{\textrm{KL}}(q_{\boldsymbol{\phi},\boldsymbol{\psi}}(\rvx_0|\rvx_{T})||p_{\boldsymbol{\theta}}(\rvx_0|\rvx_{T}))] = \mathbb{E}_{q_{\boldsymbol{\phi},\boldsymbol{\psi}}(\rvx_0,\rvx_T)}[-\log{p_{\theta}(\rvx_0|\rvx_T)}]-\mathcal{H}(q_{\boldsymbol{\phi},\boldsymbol{\psi}}(\rvx_0|\rvx_{T})) \label{eq:thm10-2}
\end{align}
The RHS of \cref{eq:thm10-1} becomes followings from the result of \cref{eq:thm1-9},
\begin{align}
\mathbb{E}_{q_{\boldsymbol{\phi},\boldsymbol{\psi}}(\rvx_T)} [ D_{\textrm{KL}}(\mu_{\boldsymbol{\phi},\boldsymbol{\psi}}(\cdot|\rvx_{T})||\nu_{\boldsymbol{\theta}}(\cdot|\rvx_{T})) ] = \mathcal{L}_{\textrm{SM}} = \mathcal{L}_{\textrm{AE}} - \mathcal{H}(q_{\boldsymbol{\phi},\boldsymbol{\psi}}(\rvx_0|\rvx_{T})) \label{eq:thm10-3}
\end{align}
From \cref{eq:thm10-1,eq:thm10-2,eq:thm10-3}, we have the followings
\begin{align}
\mathbb{E}_{q_{\boldsymbol{\phi},\boldsymbol{\psi}}(\rvx_0,\rvx_T)}[-\log{p_{\theta}(\rvx_0|\rvx_T)}] \leq \mathcal{L}_{\textrm{AE}} \label{eq:thm10-4}
\end{align}
We have the following to sum up \cref{eq:thm10-4} and \cref{eq:mi_dbae_final}.
\begin{align}
      - MI(\rvx_0,\rvz) \leq \mathcal{L}_{\textrm{AE}}-\mathcal{H}(q_{\text{data}}(\rvx_0))
\end{align}

\label{sec:A1.6}

\end{proof}

\section{Related Work}
\subsection{Representation Learning in Diffusion Models}

Expanding the applicability of generative models to various downstream tasks depends on exploring meaningful latent variables through representation learning. Methods within both variational autoencoders (VAEs)~\citep{KingmaW13,rezende2014stochastic,higgins2017betavae,zhao2019infovae,kim2018disentangling} and generative adversarial networks (GANs)~\citep{jeon2021ib,karras2020analyzing,abdal2019image2stylegan,abdal2020image2stylegan++,chen2016infogan} have been proposed; however, VAEs suffer from low sample quality, limiting their practical deployment in real-world scenarios. Conversely, GANs are known for their ability to produce high-quality samples with fast sampling speeds but face challenges in accessing latent variables due to their intractable model structure. This leads to computationally expensive inference methods like GAN inversion~\citep{xia2022gan,voynov2020unsupervised,zhu2016generative,karras2020analyzing,abdal2019image2stylegan}. Additionally, the adversarial training objective of GANs introduces instability during the training.

In contrast, recent research has delved into representation learning within diffusion probabilistic models (DPMs), which offer stable training and high sample quality. In early studies, the diffusion endpoint $\rvx_T$ was introduced as a latent variable~\citep{song2021denoising,song2021scorebased} with an invertible path defined by an ordinary differential equation (ODE). However, $\rvx_T$ is difficult to consider as a semantically meaningful encoding. Additionally, the dimension of $\rvx_T$ matches that of the original data $\rvx_0$, limiting the ability to learn condensed feature representation for downstream tasks (e.g., downstream inference, attribute manipulation with linear classifier). The inference of latent variables also relies on solving ODE, rendering inference intractable. This intractability not only hinders the desired regularization (e.g. disentanglment~\citep{higgins2017betavae,kim2018disentangling,chen2018isolating}) of the latent variable but also slows down the downstream applications. 

Diffusion AutoEncoder (DiffAE)~\citep{preechakul2022diffusion} introduces a new framework for learning tractable latent variables in DPMs. DiffAE learns representation in the latent variable $\rvz$ through an auxiliary encoder, with a $\rvz$-conditional score network~\citep{ronneberger2015u}. The encoder-generated latent variable $\rvz$ can learn a semantic representation with a flexible dimensionality. Pre-trained DPM AutoEncoding (PDAE)~\citep{zhang2022unsupervised} proposes a method to learn unsupervised representation from pre-trained unconditional DPMs. PDAE also employs an auxiliary encoder to define $\rvz$ and introduces a decoder to represent $\nabla_{\rvx_t}\log p(\rvz|\rvx_t)$. PDAE can parameterize the $\rvz$-conditional model score combined with a pre-trained unconditional score network, utilizing the idea of classifier guidance~\citep{dhariwal2021diffusion}. PDAE can use the pre-trained checkpoint from publicly available sources, but its complex decoder architecture slows down the sampling speed. 

Subsequent studies have imposed additional assumptions or constraints on the encoder based on specific objectives. DiTi~\citep{yue2024exploring} introduces a time-dependent latent variable on the top of PDAE to enable feature learning that depends on diffusion time. InfoDiffusion~\citep{wang2023infodiffusion} regularizes the latent space of DiffAE to foster an informative and disentangled representation of $\rvz$. It should be noted that such proposed regularization in~\citep{wang2023infodiffusion} is also applicable with DBAE, and \Cref{sec:4.3} demonstrates that the tradeoff between disentanglement and sample quality is better managed in DBAE than in DiffAE. FDAE~\citep{Wu_Zheng_2024} learns disentangled latent representation by masking image pixel content with DiffAE. DisDiff~\citep{yang2023disdiff} learns disentangled latent variable $\rvz$ by minimizing mutual information between each latent variable from different dimensions atop PDAE. LCG-DM~\citep{kim2022unsupervised} adopts a pre-trained disentangled encoder and trains DiffAE structure with fixed encoder parameters to enable unsupervised controllable generation. SODA~\citep{hudson2023soda} improves the network architectures of DiffAE and training for novel image reconstruction.

All the frameworks~\citep{preechakul2022diffusion,zhang2022unsupervised} and applications~\citep{yue2024exploring,wang2023infodiffusion,Wu_Zheng_2024,yang2023disdiff,hudson2023soda} utilize the encoder and do not consider the diffusion endpoint $\rvx_T$, leading to an \textit{information split problem}. In contrast, DBAE constructs an $\rvz$-dependent endpoint $\rvx_T$ inference with feed-forward architecture to induce $\rvz$ as an information bottleneck. Our framework makes $\rvz$ more informative, which is orthogonal to advancements in downstream applications~\citep{kim2022unsupervised,yue2024exploring,wang2023infodiffusion,Wu_Zheng_2024,yang2023disdiff,hudson2023soda}, as exemplified in \Cref{sec:4.3}.

\subsection{Parametrized Forward Diffusion}

The forward diffusion process with learnable parameters is a key technique in DBAE to resolve \textit{information split problem}. We summarize several other methods that proposed a learnable forward process. Note that DBAE has clear technical differences from those methods.

Sch$\ddot{\textrm{o}}$dinger bridge problem (SBP)~\citep{de2021diffusion,chen2022likelihood} learns the pair of SDEs that have forward and reverse dynamics relationships. SBP identifies the joint distribution in the form of a diffusion path between two given marginal distributions. The optimization is reduced to entropy-regularized optimal transport~\citep{schrodinger1932theorie,genevay2018learning}, which is often solved by Iterative Proportional Fitting~\citep{ruschendorf1995convergence}. For this optimization, samples are required at any given time $t$ from the forward SDE; however, these samples are not from a Gaussian kernel like \cref{eq:for} or \cref{eq:doobfor},  resulting in longer training times needed to solve the SDE numerically with intermediate particles. The formulation is also not suitable for our case, as we learn the given joint distribution through an encoder-decoder framework. 

Diffusion normalizing flow (DiffFlow)~\citep{zhang2021diffusion} parameterizes the drift term in \cref{eq:for} using a normalizing flow, making the endpoint of DiffFlow learnable. However, both training and endpoint inference are intractable because the parametrized forward SDE does not provide a Gaussian kernel similar to that in SBP. Implicit nonlinear diffusion model (INDM)~\citep{kim2022maximum} learns a diffusion model that is defined in the latent space of a normalizing flow, implicitly parameterizing both the drift and volatility terms in \cref{eq:for}. A unique benefit is its tractable training, allowing direct sampling from any diffusion time $t$. However, INDM merely progresses the existing diffusion process in the flow latent space, making it unsuitable for encoding due to technical issues such as dimensionality. The inference also requires solving the ODE for encoding.

Unlike other studies, we parameterize the endpoint $\rvx_T$ rather than the drift or volatility terms. The forward process is naturally influenced by the endpoint determined from Doob’s $h$-transform. Unlike other parameterized diffusions, our approach ensures tractable learning and $\rvx_T$ inference, making it particularly advantageous for encoding tasks.

\section{Implementation details}
\label{sec:A3}
\subsection{Training Configuration}
\label{sec:A3.1}
\textbf{Model Architecture} We use the score network ($\boldsymbol{\theta}$) backbone U-Net~\citep{ronneberger2015u}, which are modified for diffusion models~\citep{dhariwal2021diffusion} with time-embedding. DiffAE~\citep{preechakul2022diffusion}, PDAE~\citep{zhang2022unsupervised}, and DiTi~\citep{yue2024exploring} also utilize the same score network architecture. The only difference for DBAE is the endpoint $\rvx_T$ conditioning. We follow DDBM~\citep{zhou2024denoising} which concatenate $\rvx_t$ and $\rvx_T$ for the inputs as described in \Cref{fig:architecture_2}. This modification only increases the input channels, so the complexity increase is marginal. While the endpoint $\rvx_T$ contains all the information from $\rvz$, we design a score network also conditioning on $\rvz$ for implementation to effectively utilize the latent information in the generative process. For the encoder ($\boldsymbol{\phi}$), we utilize the same structure from DiffAE~\citep{preechakul2022diffusion}. For the decoder ($\boldsymbol{\psi}$), we adopt the upsampling structure from the generator of FastGAN~\citep{liu2021towards}, while removing the intermediate stochastic element. For the generative prior ($\boldsymbol{\omega}$), we utilize latent ddim from \citep{preechakul2022diffusion}. \Cref{tab:setting,tab:setting_latentddim} explains the network configurations for the aforementioned structures. 

\textbf{Optimization} We follow the optimization argument from DDBM~\citep{zhou2024denoising} with Variance Preserving (VP) SDE. We utilize the preconditioning and time-weighting proposed in DDBM, with the pred-x parameterization~\citep{karras2022elucidating}. \Cref{tab:setting} shows the remaining optimization hyperparameters. While DDBM does not include the encoder ($\boldsymbol{\phi}$) and the decoder ($\boldsymbol{\psi}$), we optimize jointly the parameters $\boldsymbol{\phi}$, $\boldsymbol{\psi}$, and $\boldsymbol{\theta}$ to minimize $\mathcal{L}_{\text{AE}}$.

\begin{figure}[]
     \centering
     \begin{subfigure}[b]{\textwidth}
         \centering
         \includegraphics[width=\textwidth]{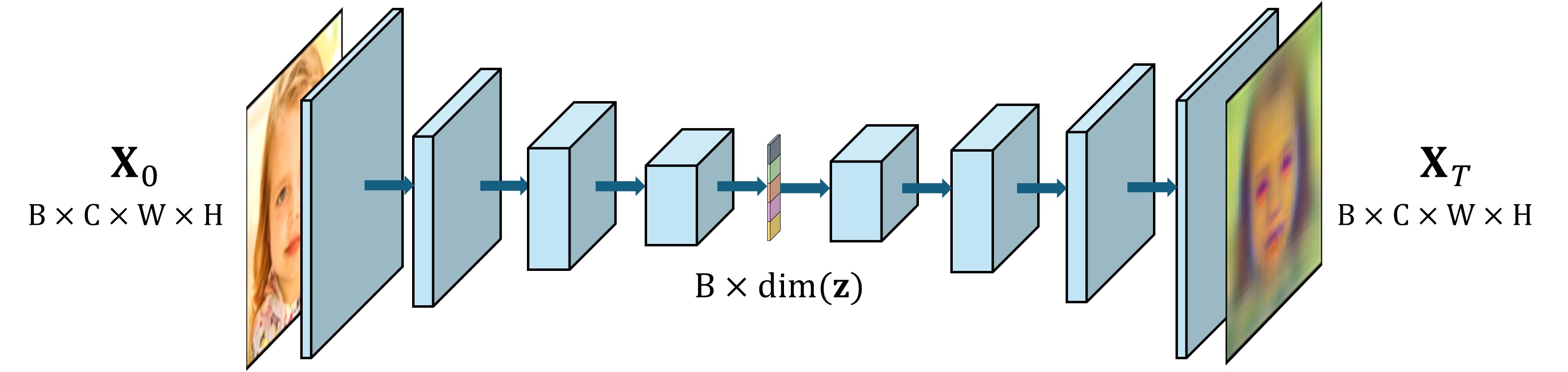}
         \caption{The encoder ($\boldsymbol{\phi}$) structure in the left, and decoder ($\boldsymbol{\psi}$) structure in the right.}
         \label{fig:architecture_1}
     \end{subfigure}
     \vspace{5mm}
     \begin{subfigure}[b]{\textwidth}
         \centering
         \includegraphics[width=\textwidth]{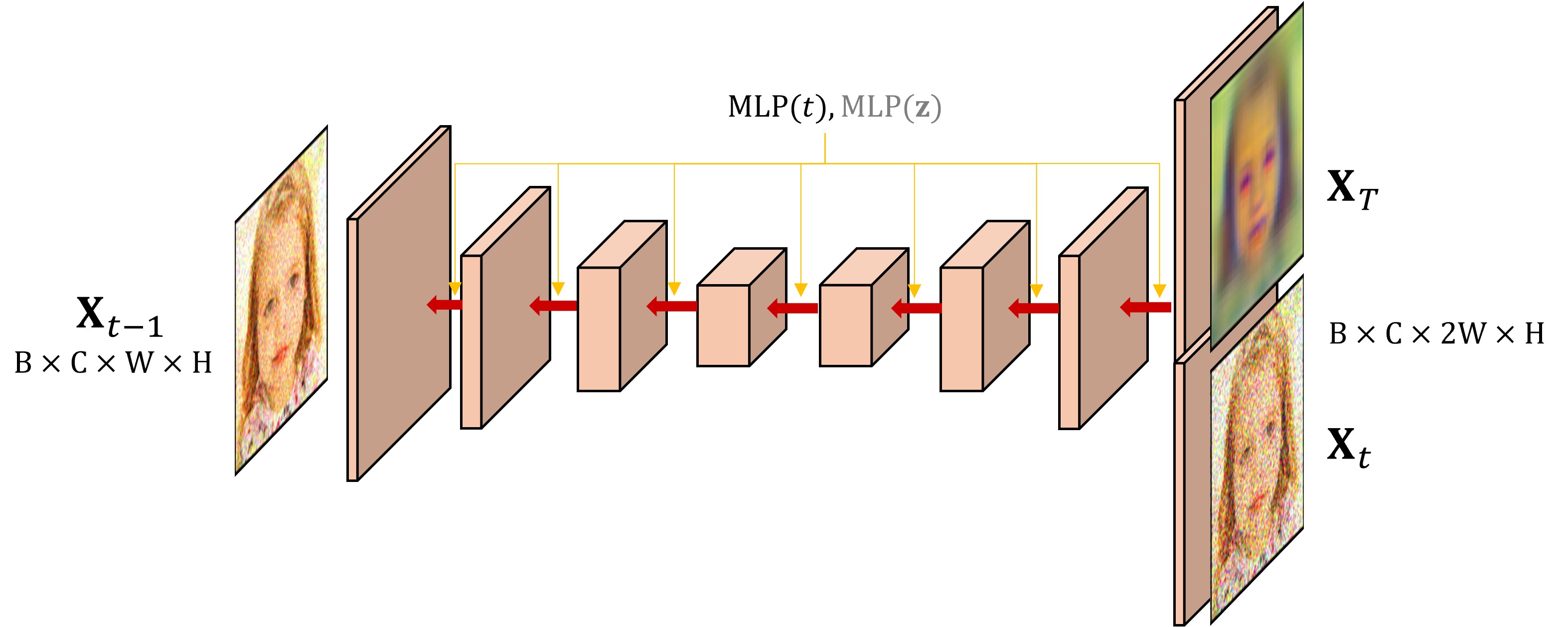}
         \caption{The score network ($\boldsymbol{\theta}$) structure. While the model output is not directly one-step denoised sample $\rvx_t$, the output is equivalent to $\rvx_{t-1}$ with time-dependent constant operation with accessible information. }
         \label{fig:architecture_2}
     \end{subfigure}
    \caption{The architecture overview of Diffusion Bridge AutoEncoder.}
    \label{fig:architecture}
\end{figure}

\begin{table}[]
\centering
\caption{Network architecture and training configuration of DBAE.}
\resizebox{\textwidth}{!}{
\begin{tabular}{l|ccccc}
\toprule
$\textbf{Parameter}$ & $\textbf{CelebA 64}$  & $\textbf{FFHQ 128}$ & $\textbf{Horse 128}$ & $\textbf{Bedroom 128}$  \\ \midrule
Base channels                       & 64                                                            & 128                                & 128                                 & 128                                                        \\
Channel multipliers                 & {[}1,2,4,8{]}                                       & {[}1,1,2,3,4{]}                    & {[}1,1,2,3,4{]}                     & {[}1,1,2,3,4{]}                                      \\
Attention resolution                & {[}16{]}                                              & {[}16{]}                           & {[}16{]}                            & {[}16{]}                                                 \\ \midrule
Encoder base ch                     & 64                                                               & 128                                & 128                                 & 128                                                         \\ 
Enc. attn. resolution               & {[}16{]}                                                & {[}16{]}                           & {[}16{]}                            & {[}16{]}                                                      \\
Encoder ch. mult.                   & {[}1,2,4,8,8{]}                                 & {[}1,1,2,3,4,4{]}                  & {[}1,1,2,3,4,4{]}                   & {[}1,1,2,3,4,4{]}                                    \\
latent variable $\rvz$ dimension      & 32, 256, 512                                                          & 512                                & 512                                 & 512                                                   \\ \midrule
Vanilla forward SDE                   & VP                                                & VP                             & VP                              & VP                                                    \\
Images trained                      & 72M, 130M                                                              & 130M                               & 130M                                & 130M                                                             \\
Batch size                          & 128                                                        & 128                                & 128                                 & 128                                                            \\
Learning rate                       & 1e-4                                                        & 1e-4                                & 1e-4                                 & 1e-4   \\ 
Optimizer                         & RAdam                                                        & RAdam                                & RAdam                                 & RAdam  \\
Weight decay                      & 0.0                                                        & 0.0                                 & 0.0                                  & 0.0   \\ 
EMA rate                           & 0.9999                                                        & 0.9999                                 & 0.9999                                  & 0.9999 \\
\bottomrule
\end{tabular}}
\label{tab:setting}
\end{table}

\begin{table}[]
\centering
\caption{Network architecture and training configuration of latent diffusion models $p_{\boldsymbol{\omega}}(\rvz)$ for an unconditional generation, following \citep{preechakul2022diffusion}.}
\begin{tabular}{l|cc}
\toprule
$\textbf{Parameter}$ & $\textbf{CelebA 64}$ & $\textbf{FFHQ 128}$  \\ \midrule
Batch size                          & 512                              & 256                                                   \\
$\rvz$ trained   & 600M                             & 600M                                                \\
MLP layers (\textit{N})             & 10, 15                               & 10                                                   \\
MLP hidden size                     & \multicolumn{2}{c}{2048}                                                                                                                \\
latent variable $\rvz$ dimension    & \multicolumn{2}{c}{512}                                                                                                                 \\
 SDE                                & \multicolumn{2}{c}{VP}  \\
$\beta$ scheduler                   & \multicolumn{2}{c}{Constant 0.008}                                                                                                      \\
Learning rate                       & \multicolumn{2}{c}{1e-4}                                                                                                                \\
Optimizer                           & \multicolumn{2}{c}{AdamW (weight decay = 0.01)}                                \\
Train Diff \textit{T}               & \multicolumn{2}{c}{1000}                                                                                                                \\
Diffusion loss                      & \multicolumn{2}{c}{L1, L2}                                                                               \\               \bottomrule
\end{tabular}

\label{tab:setting_latentddim}
\end{table}

\subsection{Evaluation Configuration and Metric}
\label{subsec:C.2.2}

\textbf{Downstream Inference} In \Cref{tab:main1}, we use Average Precision (AP), Pearson Correlation Coefficient (Pearson's r), and Mean Squared Error (MSE) as metrics for comparison. For AP measurement, we train a linear classifier $(\mathbb{R}^{l}\rightarrow [0, 1]^{40})$ to classify 40 binary attribute labels from the CelebA~\citep{liu2015deep} training dataset. The output of the encoder, Enc$_{\boldsymbol{\phi}}(\rvx_0)=\rvz$, serves as the input for a linear classifier. We examine the CelebA test dataset. Precision and recall for each attribute label are calculated by computing true positives (TP), false positives (FP), and false negatives (FN) for each threshold interval divided by predicted values. The area under the precision-recall curve is obtained as AP. For Pearson's r and MSE, we train a linear regressor $(\mathbb{R}^{l}\rightarrow\mathbb{R}^{73})$ using LFW~\citep{huang2007unsupervised,kumar2009attribute} dataset. The regressor predicts the value of 73 attributes based on the latent variable $\rvz$. Pearson's r is evaluated by calculating the variance and covariance between the ground truth and predicted values for each attribute, while MSE is assessed by measuring the differences between two values. We borrow the baseline results from the DiTi~\citep{yue2024exploring} paper and adhere to the evaluation protocol found at \url{https://github.com/yue-zhongqi/diti}.

\textbf{Reconstruction} We quantify reconstruction error in \Cref{tab:main2} though the Structural Similarity Index Measure (SSIM)~\citep{wang2003multiscale}, Learned Perceptual Image Patch Similarity (LPIPS)~\citep{zhang2018unreasonable} and Mean Squared Error (MSE). This metric measures the distance between original images in CelebA-HQ and their reconstructions across all 30K samples and averages them. SSIM compares the luminance, contrast, and structure between images to measure the differences on a scale from 0 to 1, like human visual perception. LPIPS measures the distance in the feature space of a neural network that learns the similarity between two images. We borrow the baseline results from DiffAE~\citep{preechakul2022diffusion} and PDAE~\citep{zhang2022unsupervised}. We use Heun's ODE sampler (99 NFE) to evaluate SSIM and MSE and use a stochastic sampler~\citep{zhou2024denoising} (998 NFE) to evaluate LPIPS for \Cref{tab:main2}. We also present performance metrics for various NFE values in \Cref{tab:recon-app1,tab:recon-app2}.

\textbf{Disentanglment} The metric Total AUROC Difference (TAD)~\citep{yeats2022nashae} measures how effectively the latent space is disentangled, utilizing a dataset with multiple binary ground truth labels. It calculates the correlation between attributes based on the proportion of entropy reduction given any other single attribute. Attributes that show an entropy reduction greater than 0.2 when conditioned on another attribute are considered highly correlated and therefore entangled. For each remaining attribute that is not considered entangled, we calculate the AUROC score for each dimension of the latent variable $\rvz$. To calculate the AUROC score, first determine the dimension-wise minimum and maximum values of $\rvz$. We increment the threshold from the minimum to the maximum for each dimension, converting $\rvz$ to a one-hot vector by comparing each dimension's value against the threshold. This one-hot vector is then compared to the true labels to compute the AUROC score. An attribute is considered disentangled if at least one dimension of $\rvz$ can detect it with an AUROC score of 0.75 or higher. The sub-metric ATTRS denotes the number of such captured attributes. The TAD score is calculated as the sum of the differences between the two highest AUROC scores for each captured attribute. We randomly selected 1000 samples from the CelebA training, validation, and test sets to perform the measurement following~\citep{yeats2022nashae}. We borrow the baseline results expect DisDiff from the InfoDiffusion~\citep{wang2023infodiffusion}, and we follow their setting that the dim$(\rvz)=32$. DisDiff~\citep{yang2023disdiff} utilizes the dim$(\rvz)=192$ and we borrow its performance from the original paper. We use evaluation code from \url{https://github.com/ericyeats/nashae-beamsynthesis}.

\textbf{Unconditional Generation} To measure unconditional generative modeling, we quantify Precision, Recall~\citep{kynkaanniemi2019improved}, Inception Score (IS)~\citep{salimans2016improved} and the Fr$\acute{\text{e}}$chet Inception Distance (FID)~\citep{heusel2017gans}. Precision and Recall are measured by 10k real images and 10k generated images following~\citep{dhariwal2021diffusion}.  Precision is the ratio of generated images belonging to real images' manifold. Recall is the ratio of real images belonging to the generated images' manifold. The manifold is constructed in a pre-trained feature space using the nearest neighborhoods. Precision quantifies sample fidelity, and Recall quantifies sample diversity. Both IS and FID are influenced by fidelity and diversity. IS is calculated using an Inception Network~\citep{szegedy2016rethinking} pre-trained on ImageNet~\citep{russakovsky2015imagenet}, and it computes the logits for generated samples. If an instance is predicted with high confidence for a specific class, and predictions are made for multiple classes across all samples, then the IS will be high. On the other hand, for samples generated from FFHQ or CelebA, predictions cannot be made for multiple classes, which does not allow for diversity to be reflected. Therefore, a good Inception Score (IS) can only result from high-confidence predictions based solely on sample fidelity. We measure IS for 10k generated samples. FID approximates the generated and real samples as Gaussians in the feature space of an Inception Network and measures the Wasserstein distance between them. Since it measures the distance between distributions, it emphasizes the importance of sample diversity and sample fidelity. For \Cref{tab:main4} we measure FID between 50k random samples from the FFHQ dataset and 50k randomly generated samples. For `AE', we measure the FID between 50k random samples from the FFHQ dataset and generate samples that reconstruct the other 50k random samples from FFHQ. In \Cref{tab:main3}, we measure the FID between 10k random samples from the CelebA and 10k randomly generated samples. We utilize \url{https://github.com/openai/guided-diffusion} to measure Precision, Recall and IS. We utilize \url{https://github.com/GaParmar/clean-fid} to measure FID. In \Cref{tab:main4}, we loaded checkpoints for all baselines (except the generative prior of PDAE, we train it to fill performance) and conducted evaluations in the same NFEs. \Cref{tab:Uncond_500} shows the performance under various NFEs. For CelebA training, we use a dim$(\rvz)=256$ following~\citep{wang2023infodiffusion}, while FFHQ training employs a dim$(\rvz)=512$ following ~\citep{preechakul2022diffusion,zhang2022unsupervised}.

\subsection{Algorithm}
\label{subsec:C.3.3}
This section presents the training and utilization algorithms of DBAE. Algorithm \ref{alg:1} outlines the procedure for minimizing the autoencoding objective, $\mathcal{L}_{\text{AE}}$. Algorithm \ref{alg:2} explains the method for reconstruction using the trained DBAE. Algorithm \ref{alg:3} describes the steps for training the generative prior, $p_{\boldsymbol{\omega}}$. Algorithm \ref{alg:4} explains the procedure for unconditional generation using the trained DBAE and generative prior.

\begin{algorithm}[]
    \caption{Latent DPM Training Algorithm}     
    \label{alg:3}
    \small
    \SetCustomAlgoRuledWidth{\linewidth}
    \DontPrintSemicolon
    \KwInput{Enc$_{\boldsymbol{\phi}}$, data distribution $q_{\text{data}}(\rvx_0)$, drift term $\mathbf{f}$, volatility term $g$}
    \KwOutput{Latent DPM score network $\mathbf{s}_{\boldsymbol{\omega}}$} 
    \While{not converges}{
        Sample time $t$ from $[0,T]$ \\
        $\rvx_0 \sim q_{\text{data}}(\rvx_0)$ \\
        $\rvz =$ Enc$_{\boldsymbol{\phi}}(\rvx_0)$ \\
        $\rvz_t \sim \Tilde{q}_t(\rvz_t\vert\rvz_0)$ \\
        $\mathcal{L}\leftarrow{g^2(t)}{|| \mathbf{s}_{\boldsymbol{\omega}}(\rvz_t,t) - \nabla_{\rvz_t}\log {p}_{t}(\rvz_t|\rvz)||_2^2}$ \\
        Update ${\boldsymbol{\omega}}$ by $\mathcal{L}$ using the gradient descent method
    }
    \textbf{end}
\end{algorithm}
\begin{algorithm}[]
    \caption{Unconditional Generation Algorithm}     
    \label{alg:4}
    \small
    \SetCustomAlgoRuledWidth{\linewidth}
    \DontPrintSemicolon
    \KwInput{Dec$_{\boldsymbol{\psi}}$, latent score network $\mathbf{s}_{\boldsymbol{\omega}}$, score network $\mathbf{s}_{\boldsymbol{\theta}}$, latent discretized time steps $\{t_j^{*}\}_{j=0}^{N_\rvz}$, discretized time steps $\{t_i\}_{i=0}^N$}
    $\rvz_T \sim \mathcal{N}(\mathbf{0}, \mathbf{I})$\\
    \For{$j=N_{\rvz},...,1$}{
        Update $\rvz_{t_j}$ using \cref{eq:backode} \\
    }
    $\rvx_T$ = Dec${_{\boldsymbol{\psi}}}(\rvz_0)$ \\
    \For{$i=N,...,1$}{
        Update $\rvx_{t_i}$ using \cref{eq:backodeoursmodel} \\
    }
    \KwOutput{Unconditioned sample $\rvx_0$}
\end{algorithm}

\subsection{Computational Cost}
\label{subsec:C.3.4}

\begin{table}[]
\centering
\caption{Computational cost comparison for FFHQ128. Training time is measured in milliseconds per image per NVIDIA A100 (ms/img/A100), and testing time is reported in milliseconds per one sampling step per NVIDIA A100 (ms/one sampling step/A100).}
\label{tab:comput}
\begin{tabular}{l|ccc}
\toprule
 & $\textbf{Parameter Size}$ & $\textbf{Training}$ & $\textbf{Testing}$  \\ \midrule
DDIM~\citep{song2021denoising}           & 99M                  & 9.687              &    0.997        \\
DiffAE~\citep{preechakul2022diffusion}   & 129M                 & 12.088             &    1.059        \\
PDAE~\citep{zhang2022unsupervised}       & 280M                 & 12.163             &    1.375        \\
DBAE                                    & 161M                 & 13.190             &    1.024        \\       
              \bottomrule
\end{tabular}
\end{table}

\begin{table}
\centering
{
\caption{Computing costs for $\rvx_T$ inference.}
\label{tab:comput2}
\begin{tabular}{l|ccc|c}
\toprule
 &  &NFE ($\downarrow$)& & \multirow{2}{*}[-0.5\dimexpr \aboverulesep + \belowrulesep + \cmidrulewidth]{\scriptsize \shortstack[c]{ Total\\time ($\downarrow$)\\  (ms)}} \\ 
 \cmidrule{2-4}
 Method&  Enc$_{\boldsymbol{\phi}}$& Dec$_{\boldsymbol{\psi}}$ &$\mathbf{s}_{\boldsymbol{\theta}}$   \\ 
\midrule
PDAE &1&500&500&688    \\
DiffAE &1&-&250&265 \\
\rowcolor{gray!25}DBAE &1&1&0&0.31 \\
\bottomrule
\end{tabular}
}

\end{table}


This section presents a computational cost comparison among diffusion-based representation learning baselines. \Cref{tab:comput} compares DDIM~\citep{song2021denoising}, DiffAE~\citep{preechakul2022diffusion}, PDAE~\citep{zhang2022unsupervised}, and DBAE in terms of parameter size, training time, and testing time. DDIM requires only a score network (99M), resulting in minimal parameter size. DiffAE involves a $\rvz$-conditional score network (105M) and an encoder (24M), leading to an increase in parameter size. PDAE incorporates both a heavy decoder and an encoder, further increasing the parameter size. Conversely, although DBAE also includes a decoder, it is less complex (32M), resulting in a smaller relative increase in parameter size compared to PDAE. From a training time perspective, DiffAE, PDAE, and DBAE all require longer durations compared to DDIM due to their increased model sizes. DBAE's training time is 9\% longer than that of DiffAE because of the decoder module. However, the decoder does not repeatedly affect the sampling time, making it similar to DiffAE's. In contrast, PDAE, which utilizes a decoder at every sampling step, has a longer sampling time.     

\section{Additional Experiments}
\label{sec:A4}
\subsection{Downstream Inference}

\Cref{fig:app_0} shows the attribute-wise Average Precision (AP) gap between PDAE~\citep{zhang2022unsupervised} and DBAE. As discussed in \Cref{sec:4.1}, PDAE suffers from an \textit{information split problem} that $\rvx_T$ contains facial or hair details. The resulting attribute-wise gain aligns with that analysis with \Cref{fig:3}. \Cref{fig:app_1} shows the absolute attribute-wise AP of DBAE performance across the training setting varies on the encoder (deterministic/stochastic) and training datasets (CelebA training set / FFHQ). The attribute-wise performance is similar across the training configurations. \Cref{tab:main1_DiffuseVAE} shows the comparsion to the other baseline DiffuseVAE~\citep{pandey2022diffusevae}. From the two-stage paradigm of DiffuseVAE, its latent quality is only from the latent representation capability of the VAE module. This is an aligned result from the poor performance of $\beta$-TCVAE in \Cref{tab:main1}. Note that the image crop for CelebA in DiffuseVAE is not exactly the same as our setting.

\begin{figure}[]
     \centering
     \includegraphics[width=\textwidth, height=2.0in]{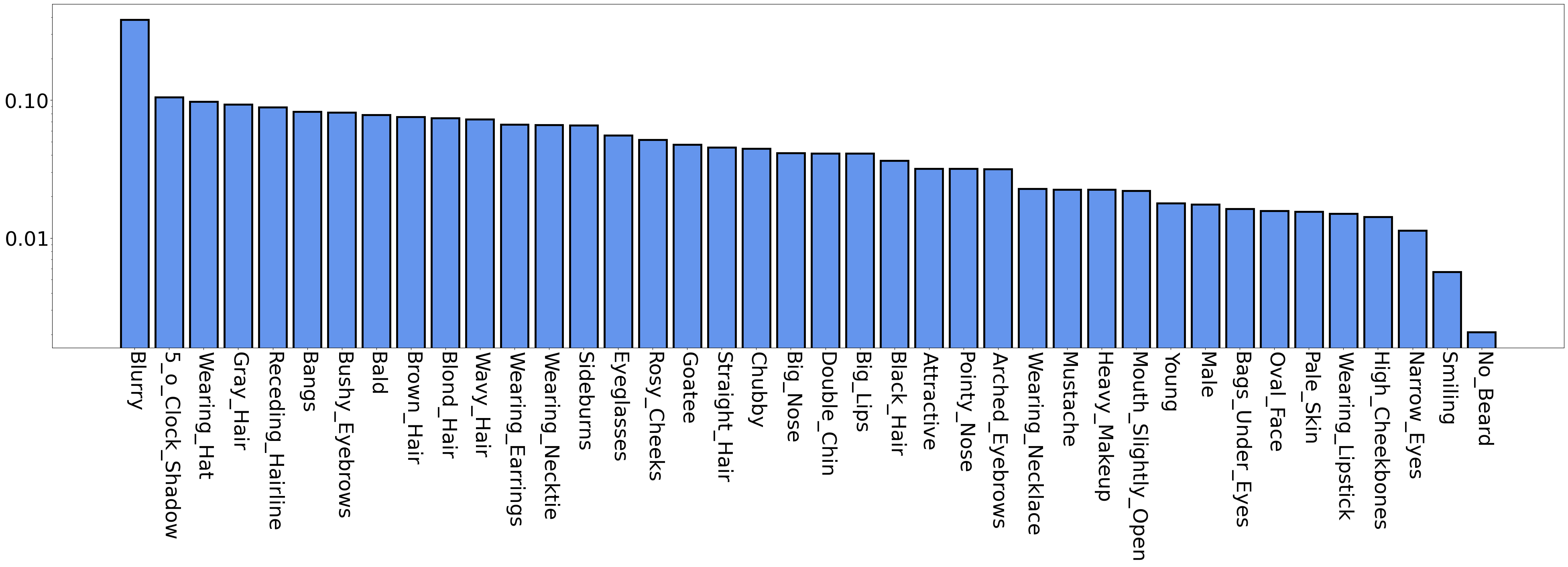}
     \caption{Attribute-wise AP gap between PDAE and DBAE-d trained on CelebA. DBAE-d performs better for all 40 attributes.}
     \label{fig:app_0}
\end{figure}

\begin{figure}[]
     \centering
     \begin{subfigure}[b]{0.49\textwidth}
         \centering
         \includegraphics[width=\textwidth, height=1.2in]{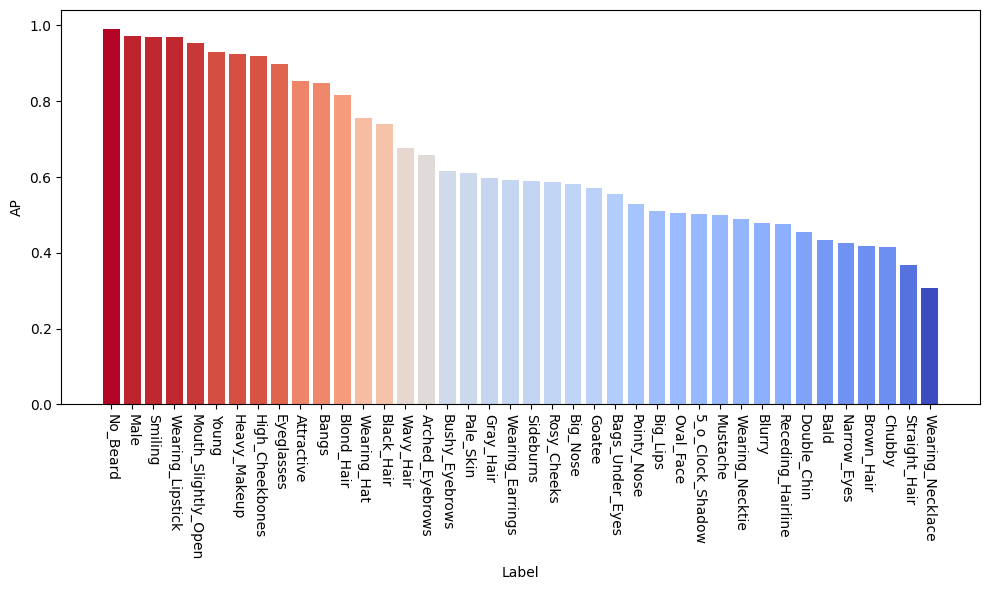}
         \caption{DBAE-d trained on CelebA}
     \end{subfigure}
     \begin{subfigure}[b]{0.49\textwidth}
         \centering
         \includegraphics[width=\textwidth, height=1.2in]{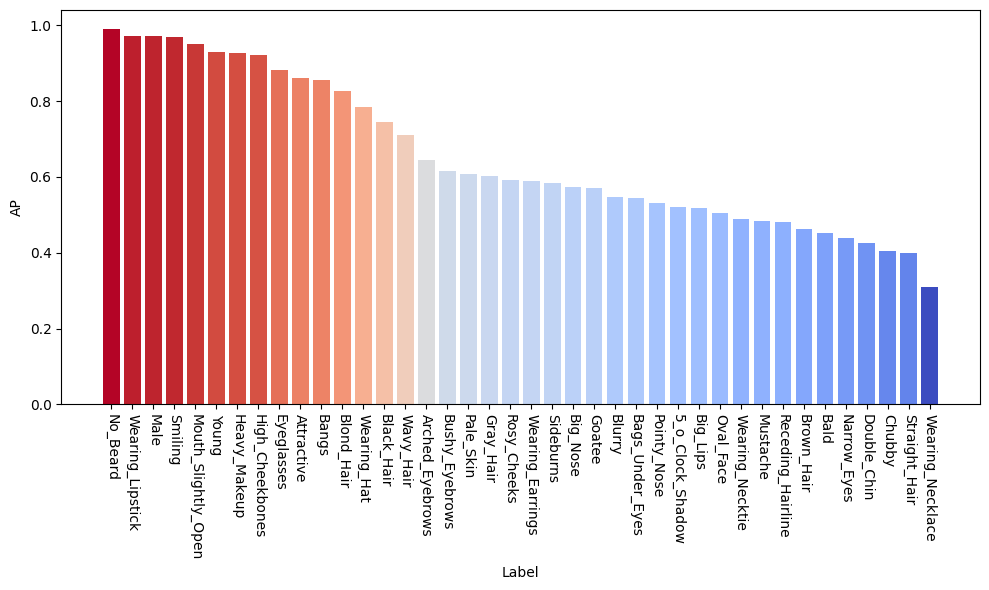}
         \caption{DBAE trained on CelebA}
     \end{subfigure}
     \hfill
     \begin{subfigure}[b]{0.49\textwidth}
         \centering
         \includegraphics[width=\textwidth, height=1.2in]{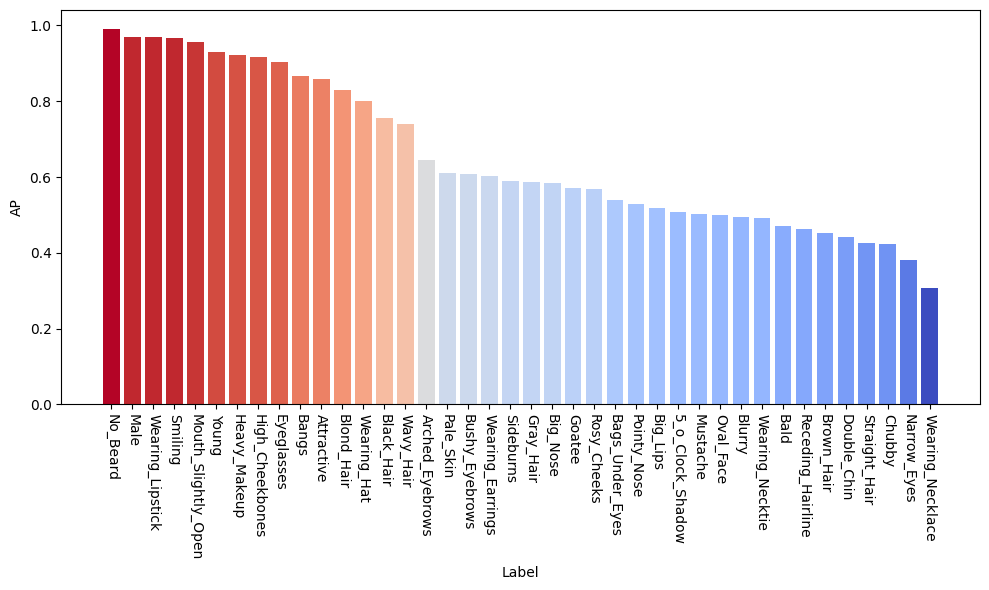}
         \caption{DBAE-d trained on FFHQ}
     \end{subfigure}
     \begin{subfigure}[b]{0.49\textwidth}
         \centering
         \includegraphics[width=\textwidth, height=1.2in]{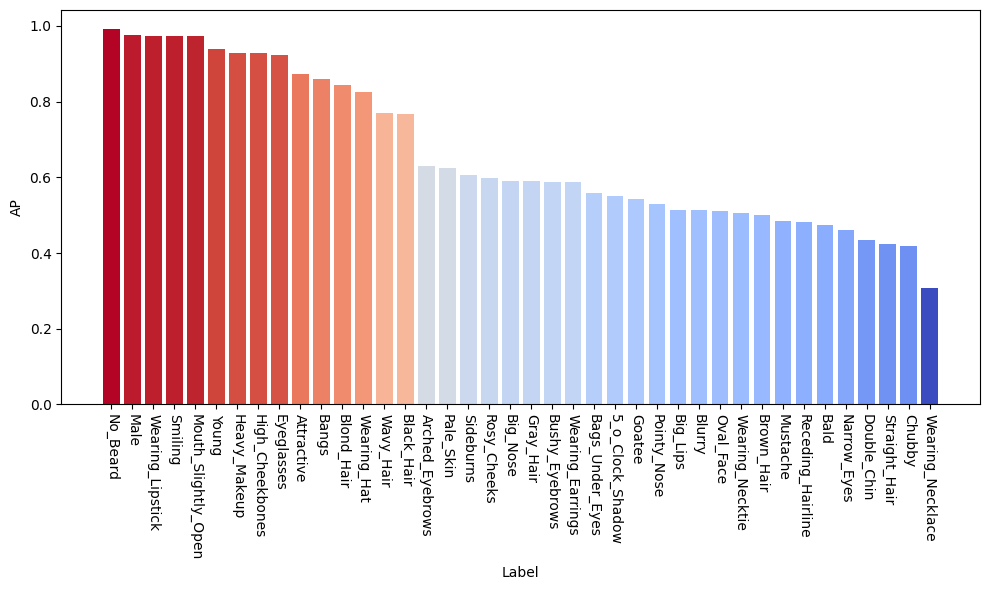}
         \caption{DBAE trained on FFHQ}
         \label{fig:app_1}
     \end{subfigure}
    \caption{Attribute-wise Average Precision across the training configuration of DBAE.}
\end{figure}
\begin{table}[h]
    \centering 
    \caption{Linear-probe attribute prediction quality comparison for models trained on CelebA and CIFAR-10 with dim$(\rvz)=512$. The best and second-best results are highlighted in \textbf{bold}. We evaluate 5 times and report the average.}
    \adjustbox{max width=\textwidth}{%
    \begin{tabular}{l ccc p{0.01\textwidth}  c p{0.01\textwidth}|}
        \toprule
              & \multicolumn{3}{c}{\textbf{CelebA}}  && \multicolumn{1}{c}{\textbf{CIFAR-10}}                       \\
             Method & AP ($\uparrow$) &Pearson's r ($\uparrow$)& MSE ($\downarrow$)      && AUROC ($\uparrow$) \\
            \midrule
            \midrule
             DiffuseVAE~\citep{pandey2022diffusevae}    & 0.395  &0.325&0.618&&0.736\\
          \rowcolor{gray!25}DBAE  & \textbf{0.655}&\textbf{0.643}&\textbf{0.369}&&\textbf{0.836}\\
        \bottomrule
    \end{tabular}
    }
\label{tab:main1_DiffuseVAE}
\end{table}

\subsection{Reconstruction}
The sampling step is important for practical applications~\citep{lu2022dpm,zheng2024dpm}. We compare the reconstruction results across various sampling steps among the baselines. \Cref{tab:recon-app1,tab:recon-app2} shows the results. The proposed model performs the best results among all NFEs in $\left (10, 20, 50, 100\right )$. We borrow the performance of DDIM, DiffAE from \citep{preechakul2022diffusion}. We manually measure for PDAE~\citep{zhang2022unsupervised} using an official checkpoint in \url{https://github.com/ckczzj/PDAE}. \Cref{fig:std} shows the reconstruction statistics for a single image with inferred $\rvz$. Due to the information split on $\rvx_T$, DiffAE shows substantial variations even when utilizing ODE sampling. When DBAE also performs stochastic sampling, information is split across the sampling path, but it has less variation compared to DiffAE (9.99 vs 6.52), and DBAE induce information can be stored solely at $\rvx_T$ through the ODE path. \Cref{tab:main2_DiffuseVAE} shows that the reconstruction quality compare to DiffuseVAE~\citep{pandey2022diffusevae}. Since DiffuseVAE also requires to sample random $\rvx_T$ for the generation, this framework also suffers from \textit{information split problem}. That is the reason for poor reconstruction quality. \Cref{tab:main3_DiffuseVAE} shows the reconstruction quality for Horse and Bedroom datasets, which surpasses the DiffAE.

\begin{table}[h]
    \centering 
    \caption{Autoencoding reconstruction quality comparison with DiffuseVAE with 512-dimensional latent variable, the one yielding the best performance is highlighted in \textbf{bold}.}
    \adjustbox{max width=\textwidth}{%
    \begin{tabular}{l ccc p{0.01\textwidth}   p{0.01\textwidth}|}
        \toprule
              & \multicolumn{3}{c}{\textbf{CelebA}}  &\\
             Method &SSIM ($\uparrow$)& LPIPS ($\downarrow$)      & MSE ($\downarrow$)      & \\
            \midrule
            \midrule
             DiffuseVAE~\citep{pandey2022diffusevae}    & 0.836  &0.134&0.018&\\
          \rowcolor{gray!25}DBAE  & \textbf{0.990}&\textbf{0.014}&\textbf{4.86e-4
}&\\
        \bottomrule
    \end{tabular}
    }
\label{tab:main2_DiffuseVAE}
\end{table}

\begin{table}[h]
    \centering 
    \caption{More results on autoencoding reconstruction quality comparison with DiffAE with 512-dimensional latent variable, the one yielding the best performance is highlighted in \textbf{bold}.}
    \adjustbox{max width=\textwidth}{%
    \begin{tabular}{l cc p{0.01\textwidth}  cc p{0.01\textwidth}|}
        \toprule
              & \multicolumn{2}{c}{\textbf{Horse}}  && \multicolumn{2}{c}{\textbf{Bedroom}} \\
             Method & SSIM ($\uparrow$) &MSE ($\downarrow$)&&  SSIM ($\uparrow$)  & MSE ($\downarrow$) \\
            \midrule
            \midrule
             DiffAE~\citep{preechakul2022diffusion}    & 0.857  &0.025&&0.910&0.017\\
          \rowcolor{gray!25}DBAE  & \textbf{0.902}&\textbf{0.012}&&\textbf{0.948}&\textbf{0.007}\\
        \bottomrule
    \end{tabular}
    }
\label{tab:main3_DiffuseVAE}
\end{table}

\begin{figure}[]
\centering
         \includegraphics[width=\textwidth]{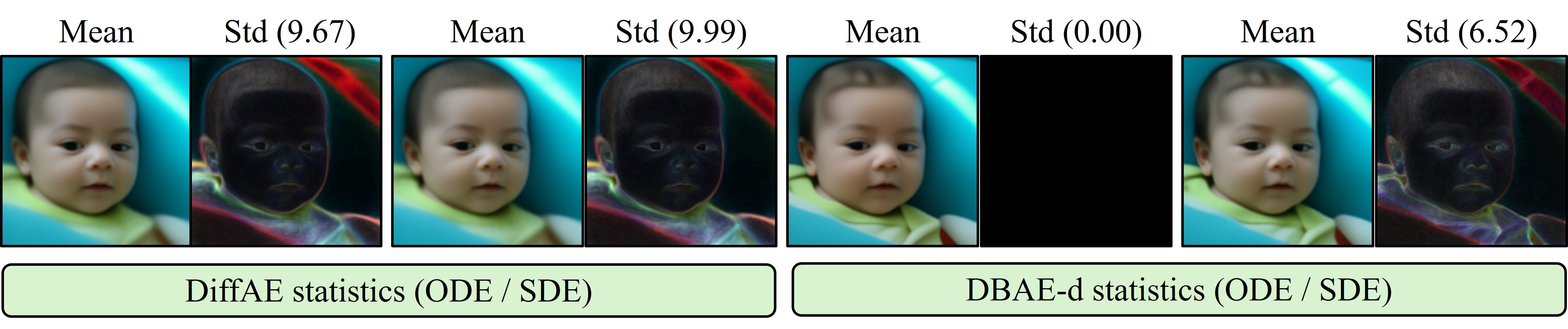}
         \caption{Reconstruction statistics with inferred $\rvz$. We quantify the mean and standard deviation of the reconstruction in the pixel space. The number in parentheses represents the dimension-wise averaged standard deviation in the pixel space.}
         \label{fig:std}
\end{figure}
\begin{table}[]
    \centering
    \caption{Autoencoding reconstruction quality comparison. All the methods are trained on the FFHQ dataset and evaluated on the 30K CelebA-HQ dataset. Among tractable and compact 512-dimensional latent variable models, the one yielding the best performance was highlighted in \textbf{bold}, followed by an \underline{underline} for the next best performer. All the metric is SSIM.}
    \adjustbox{max width=\textwidth}{%
    \begin{tabular}{lcc|cccc }
        \toprule
             Method &Tractability& Latent dim ($\downarrow$)  &NFE=10& NFE=20& NFE=50& NFE=100 \\
            \midrule
            \midrule
          DDIM (Inferred $\mathbf{x}_T$)~\citep{song2021denoising}  &\xmark  & 49,152 &0.600&0.760&0.878&0.917 \\
          DiffAE (Inferred $\mathbf{x}_T$)~\citep{preechakul2022diffusion}  & \xmark  & 49,664 &0.827&0.927&0.978&0.991 \\
          PDAE (Inferred $\mathbf{x}_T$)~\citep{zhang2022unsupervised}  & \xmark  & 49,664 &0.822&0.901&0.966&0.987 \\
          \midrule
          \midrule
          DiffAE (Random $\mathbf{x}_T$)~\citep{preechakul2022diffusion}  & \cmark  & 512 &0.707&0.695&0.683&0.677 \\
          PDAE (Random $\mathbf{x}_T$)~\citep{zhang2022unsupervised}  & \cmark  & 512 &0.728&0.713&0.697&0.689 \\
          \rowcolor{gray!25} DBAE  & \cmark  & 512 &\textbf{0.904}&\underline{0.909}&\underline{0.916}&\underline{0.920} \\
           \rowcolor{gray!25} DBAE-d   & \cmark  & 512 &\underline{0.884}&\textbf{0.920}&\textbf{0.945}&\textbf{0.954} \\
        \bottomrule
    \end{tabular}
    }
\label{tab:recon-app1}

\end{table}
\begin{table}[]
    \centering
    \caption{Autoencoding reconstruction quality comparison. All the methods are trained on the FFHQ dataset and evaluated on the 30K CelebA-HQ dataset. Among tractable and compact 512-dimensional latent variable models, the one yielding the best performance was highlighted in \textbf{bold}, followed by an \underline{underline} for the next best performer. All the metric is MSE.}
    \adjustbox{max width=\textwidth}{%
    \begin{tabular}{lcc|cccc }
        \toprule
             Method &Tractability& Latent dim ($\downarrow$)  &NFE=10& NFE=20& NFE=50& NFE=100 \\
            \midrule
            \midrule
          DDIM (Inferred $\mathbf{x}_T$)~\citep{song2021denoising}  &\xmark  & 49,152 &0.019&0.008&0.003&0.002 \\
          DiffAE (Inferred $\mathbf{x}_T$)~\citep{preechakul2022diffusion}  & \xmark  & 49,664 &0.001&0.001&0.000&0.000 \\
          PDAE (Inferred $\mathbf{x}_T$)~\citep{zhang2022unsupervised}  & \xmark  & 49,664 &0.001&0.001&0.000&0.000 \\
          \midrule
          \midrule
          DiffAE (Random $\mathbf{x}_T$)~\citep{preechakul2022diffusion}  & \cmark  & 512 &0.006&0.007&0.007&0.007 \\
          PDAE (Random $\mathbf{x}_T$)~\citep{zhang2022unsupervised}  & \cmark  & 512 &\textbf{0.004}&\underline{0.005}&\underline{0.005}&\underline{0.005} \\
           \rowcolor{gray!25} DBAE  & \cmark  & 512 &\underline{0.005}&\underline{0.005}&\underline{0.005}&\underline{0.005} \\
           \rowcolor{gray!25} DBAE-d   & \cmark  & 512 &0.006&\textbf{0.003}&\textbf{0.002}&\textbf{0.002} \\
        \bottomrule
    \end{tabular}
    }
\label{tab:recon-app2}
\end{table}

\subsection{Unconditional Generation}

The sampling step is also important for unconditional generation~\citep{lu2022dpm,zheng2024dpm}. We reduce the NFE=1000 in \Cref{tab:main4} to NFE=500 and NFE=250 in \Cref{tab:Uncond_500}. As the number of function evaluations (NFE) decreased, DDPM~\citep{ho2020denoising} showed a significant drop in performance, while DBAE and the other baselines maintained a similar performance trend.

Although DBAE improves sample fidelity which is crucial for practical uses, sample diversity remains an important virtue depending on the specific application scenarios~\citep{kim2024training,sadat2024cads}. In the area of generative models, there is a trade-off between fidelity and diversity~\citep{dhariwal2021diffusion}. Therefore, providing a balance between these two virtues is important. We offer an option based on DBAE. The $h$-transformed forward SDE we designed in \cref{eq:doobours} is governed by the determination of the endpoint distribution. If we set endpoint distribution as \cref{eq:prior_interpolation}, we can achieve smooth transitions between DiffAE and DBAE in terms of $\rvx_T$ distribution. Modeling $q_{\boldsymbol{\phi},\boldsymbol{\psi}}(\rvx_T|\rvx_0)$ as a Gaussian distribution (with learnable mean and covariance) with a certain variance or higher can also be considered as an indirect approach.
\begin{align}
\rvx_T~\sim \lambda \times q_{\boldsymbol{\phi},\boldsymbol{\psi}}(\rvx_T|\rvx_0) + (1-\lambda)\times\mathcal{N}(\mathbf{0},\mathbf{I}) \label{eq:prior_interpolation}
\end{align}

\begin{table}[h]
    \centering
    \caption{Unconditional generation with reduced NFE $\in\left\{250,500\right\}$ on FFHQ. ‘+AE’ indicates the use of the inferred distribution $q_{\boldsymbol{\phi}}(\rvz)$ instead of $p_{\boldsymbol{\omega}}(\rvz)$}
    \adjustbox{max width=\textwidth}{%
    \begin{tabular}{l|cccc|cccc }
        \toprule
                 & \multicolumn{4}{c}{NFE = 500} & \multicolumn{4}{c}{NFE = 250}\\
                 \midrule
             Method  & Prec ($\uparrow$) & IS ($\uparrow$)& FID 50k ($\downarrow$)      & Rec ($\uparrow$)& Prec ($\uparrow$) & IS ($\uparrow$)& FID 50k ($\downarrow$)      & Rec ($\uparrow$) \\
            \midrule
            \midrule
          DDIM~\citep{song2021denoising} & 0.705&3.16&11.33&0.439        & 0.706&3.16&11.48&\textbf{0.453}\\
          DDPM~\citep{ho2020denoising} & 0.589&2.92&22.10&0.251                 & 0.390&2.76&39.55&0.093\\
          DiffAE ~\citep{preechakul2022diffusion}  & 0.755&2.98&\textbf{9.71}&\textbf{0.451}      & 0.755&3.04&\textbf{10.24}&0.443\\
          PDAE~\citep{zhang2022unsupervised} &0.687&2.24&46.67&0.175                     &0.709&2.25&44.82&0.189 \\   
          \rowcolor{gray!25}DBAE &\textbf{0.774}&\textbf{3.91}&11.71&0.391                       &\textbf{0.758}&\textbf{3.90}&13.88&0.381\\
           \midrule
            \midrule
             DiffAE+AE&\textbf{0.750}&\textbf{3.61}&3.21&0.689                                     &\textbf{0.750}&\textbf{3.61}&3.87&0.666\\
            PDAE+AE  & 0.710&3.53&7.11&0.598                                                 & 0.721&3.54&6.58&0.608\\
             \rowcolor{gray!25}DBAE+AE& 0.748&3.57&\textbf{1.99}&\textbf{0.702}            & 0.731&3.58&\textbf{3.36}&\textbf{0.694} \\
        \bottomrule
    \end{tabular}
    }
\label{tab:Uncond_500}
\end{table}

\newpage
\subsection{Results with Intel Gaudi v2 hardware.}

We conducted evaluations across various infrastructures to assess experimental reproducibility. The performance of the trained model (DBAE-d) was evaluated on both the Nvidia A100 and Intel Gaudi v2 chips. The reconstruction results for both chips are presented in \Cref{tab:gaudi}. Reconstruction performance on each chip was assessed using various metrics, revealing negligible errors across all metrics. To facilitate reproducibility, we provide the code at \url{https://github.com/NAVER-INTEL-Co-Lab/gaudi-dbae} for reproducing our experiments on Intel Gaudi v2 chips.

\begin{table}[h]
    \centering 
    \caption{Regenerated results of \Cref{tab:main2} across multiple hardwares.}
    \adjustbox{max width=\textwidth}{%
    \begin{tabular}{l ccc p{0.01\textwidth}   p{0.01\textwidth}|}
        \toprule
             Hardware &SSIM ($\uparrow$)& LPIPS ($\downarrow$)      & MSE ($\downarrow$)      \\
            \midrule
            \midrule
             Nvidia A100    & 0.953  &0.072&2.49e-3\\
          Intel Gaudi v2  &0.956 &0.073&2.47e-3\\
        \bottomrule
    \end{tabular}
    }
\label{tab:gaudi}
\end{table}

\subsection{Additional Samples}

\textbf{Interpolation}
\Cref{fig:ffhq_interpol_whole,fig:horse_bedroom_interpol_whole} shows the interpolation results of DBAE trained on FFHQ, Horse, and Bedroom. The two paired rows indicate the endpoints $\rvx_T$ and generated image $\rvx_0$ each. \Cref{fig:ffhq_interpol_compare} compares the interpolation results with PDAE~\citep{zhang2022unsupervised} and DiffAE~\citep{preechakul2022diffusion} under tractable inference condition. PDAE and DiffAE result in unnatural interpolations without inferring $\rvx_T$, compared to DBAE.

\textbf{Attribute Manipulation}
\Cref{fig:app_att} shows additional manipulation results using a linear classifier, including multiple attributes editing on a single image. \Cref{fig:ffhq_classifier_vari} provides the variations in the manipulation method within DBAE. The top row utilizes the manipulated $\rvx_T$ both for the starting point of the generative process and score network condition input. The bottom row utilizes the manipulated $\rvx_T$ only for the score network condition input, while the starting point remains the original image’s $\rvx_T$. Using manipulated $\rvx_T$ both for starting and conditioning results in more dramatic editing, and we expect to be able to adjust this according to the user's desires.

\textbf{Generation Trajectory}
\Cref{fig:trajectory_whole} shows the sampling trajectory of DBAE from $\rvx_T$ to $\rvx_0$ with stochastic sampling for FFHQ, Horse, and Bedroom.

\textbf{Unconditional Generation}
\Cref{fig:uncurated_ffhq,fig:uncurated_celeba} show the randomly generated uncurated samples from DBAE for FFHQ and CelebA.

\begin{figure}[]
\centering
         \includegraphics[width=\textwidth]{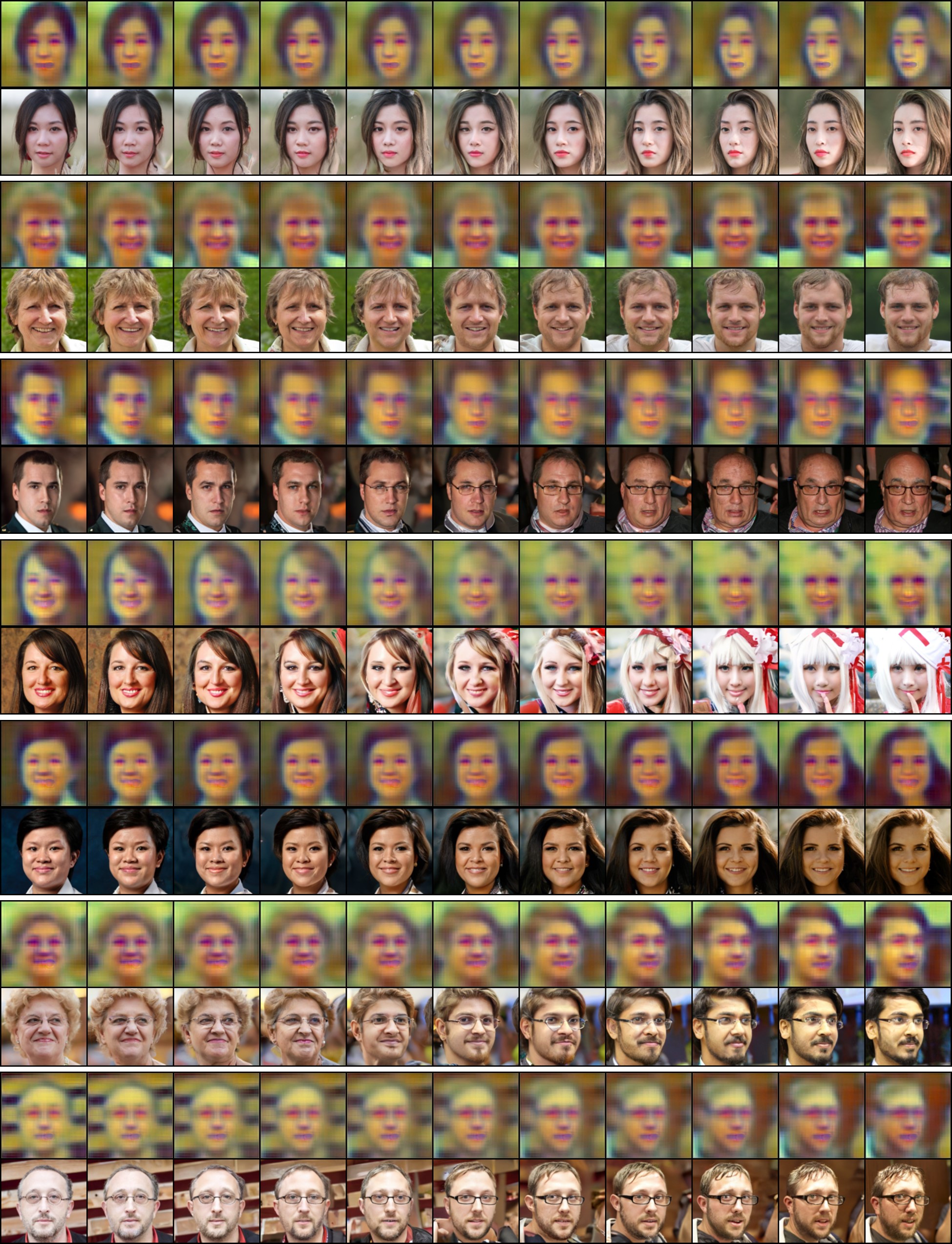}
         \caption{FFHQ interpolations results with corresponding endpoints $\rvx_T$. The leftmost and rightmost images are real images.}
         \label{fig:ffhq_interpol_whole}
\end{figure}

\begin{figure}[]
     \centering
     \begin{subfigure}[b]{\textwidth}
         \centering
         \includegraphics[width=\textwidth]{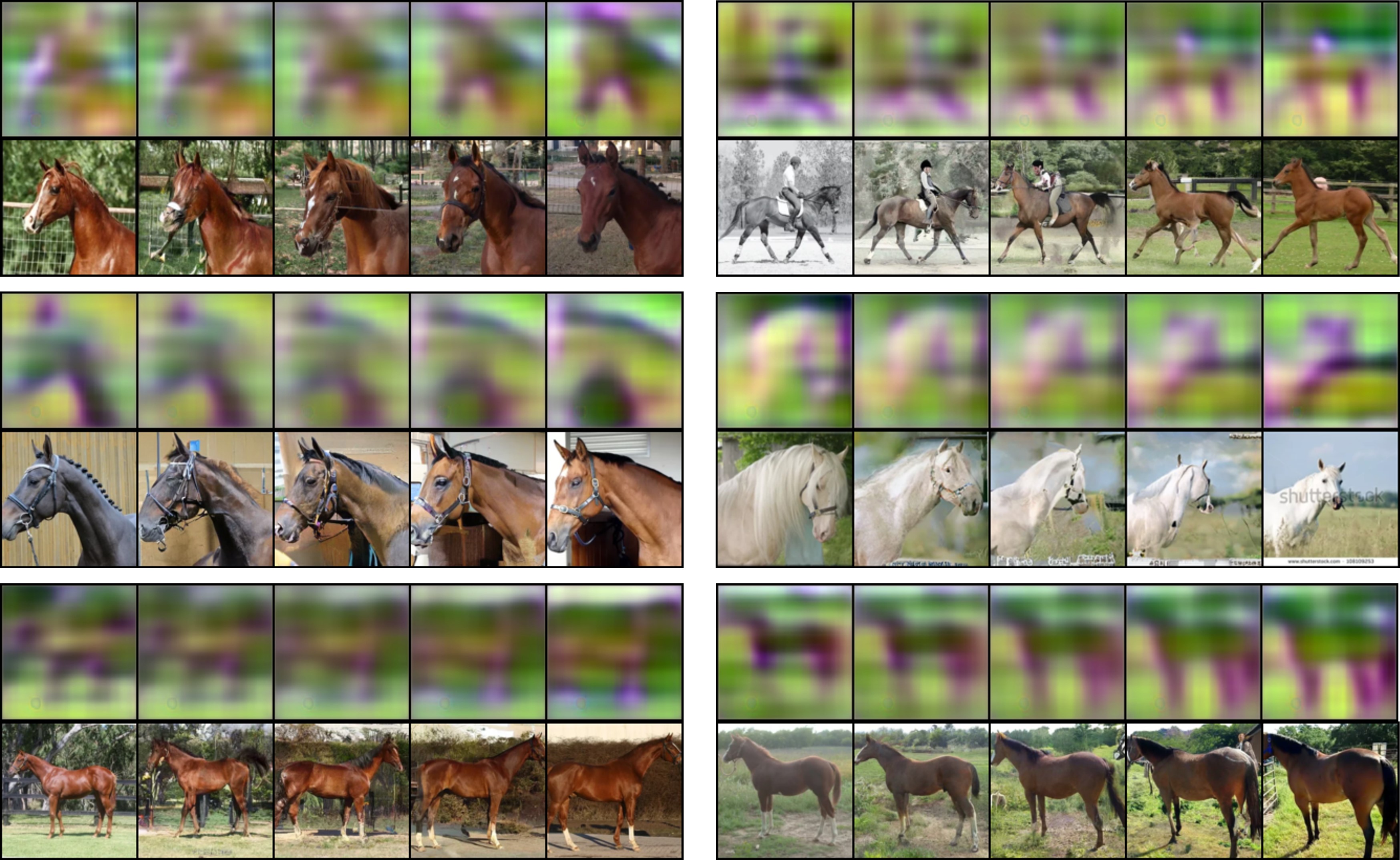}
         \vspace{5mm}
     \end{subfigure}
     \begin{subfigure}[b]{\textwidth}
         \centering
         \includegraphics[width=\textwidth]{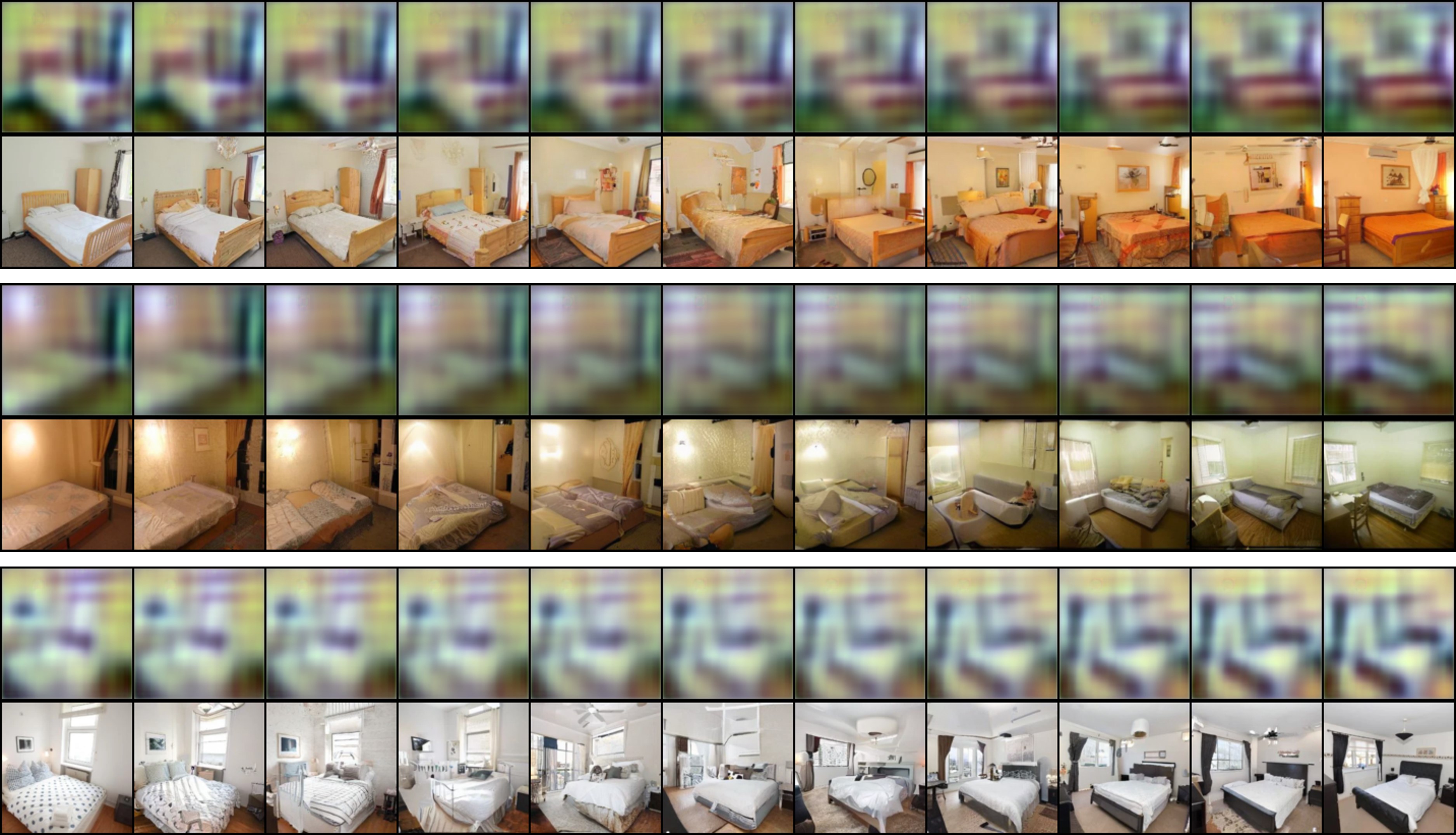}
     \end{subfigure}
    \caption{Horse and Bedroom interpolations results with corresponding endpoints $\rvx_T$. The leftmost and rightmost images are real images.}
    \label{fig:horse_bedroom_interpol_whole}
\end{figure}

\begin{figure}[]
\centering
         \includegraphics[width=\textwidth]{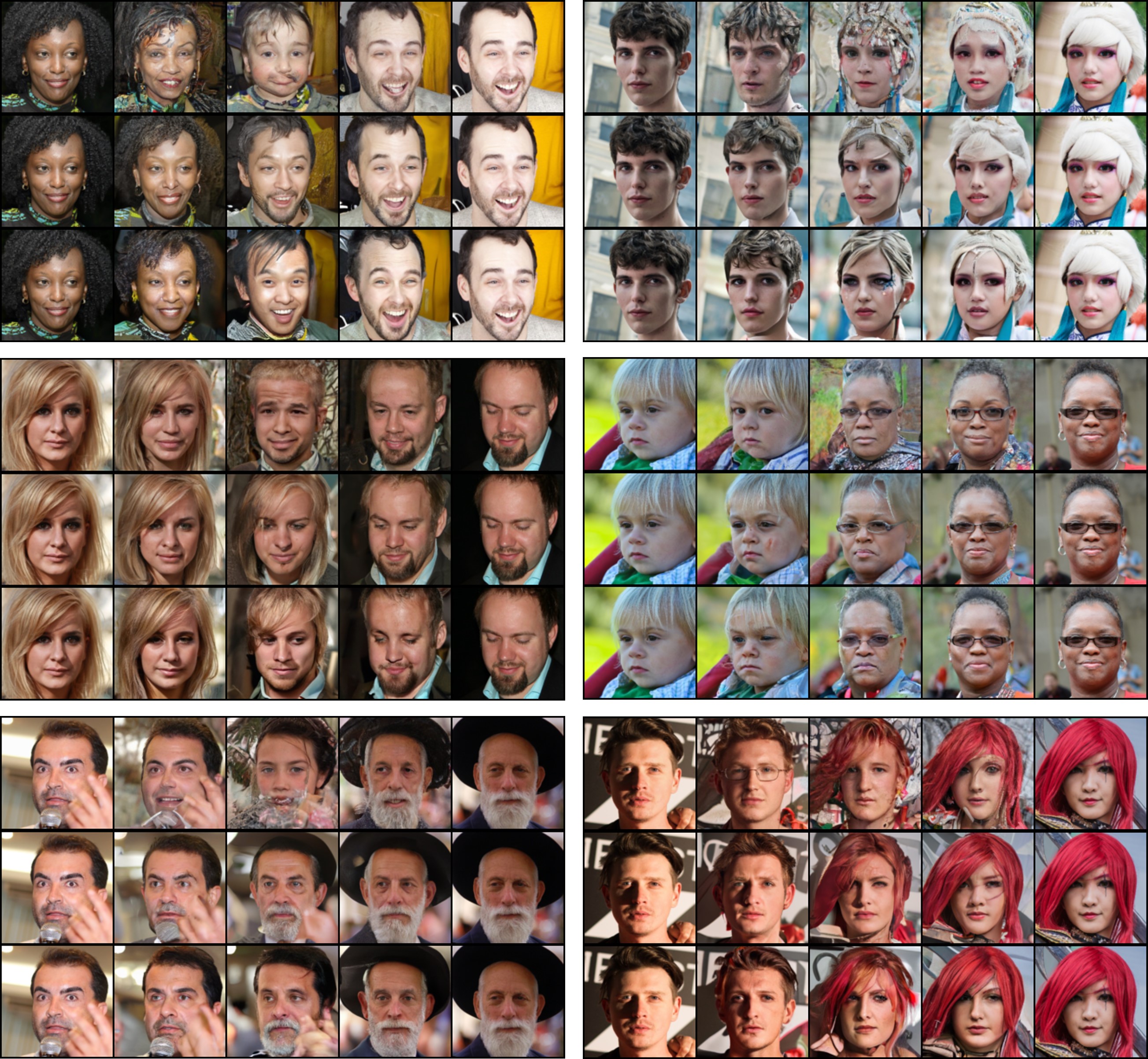}
         \caption{FFHQ interpolation comparison: PDAE~\citep{zhang2022unsupervised} (top), DiffAE~\citep{preechakul2022diffusion} (middle) and DBAE (bottom).}
         \label{fig:ffhq_interpol_compare}
\end{figure}

\begin{figure}[]
\centering
         \includegraphics[width=\textwidth]{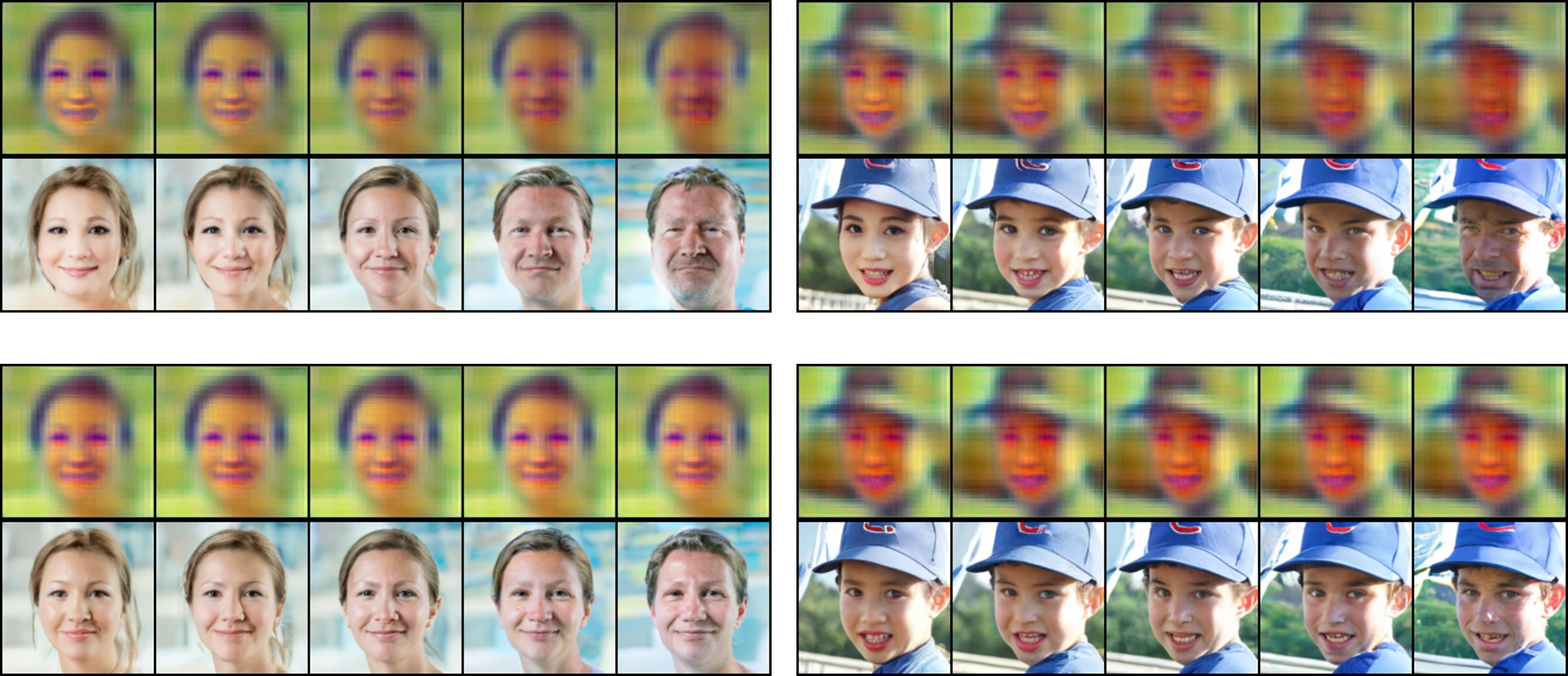}
         \caption{Attribute manipulation on FFHQ using a linear classifier and corresponding endpoints $\rvx_T$. The top results utilize the manipulated $\rvx_T$
  both as the starting point of the sampling trajectory and as a condition input to the score network. The bottom results use the manipulated $\rvx_T$ solely as the condition input and maintain the original $\rvx_T$ as the starting point of the sampling trajectory. All the middle images are the original images.}
         \label{fig:ffhq_classifier_vari}
\end{figure}

\begin{figure}[]
     \centering
     \begin{subfigure}[b]{\textwidth}
         \centering
         \includegraphics[width=\textwidth]{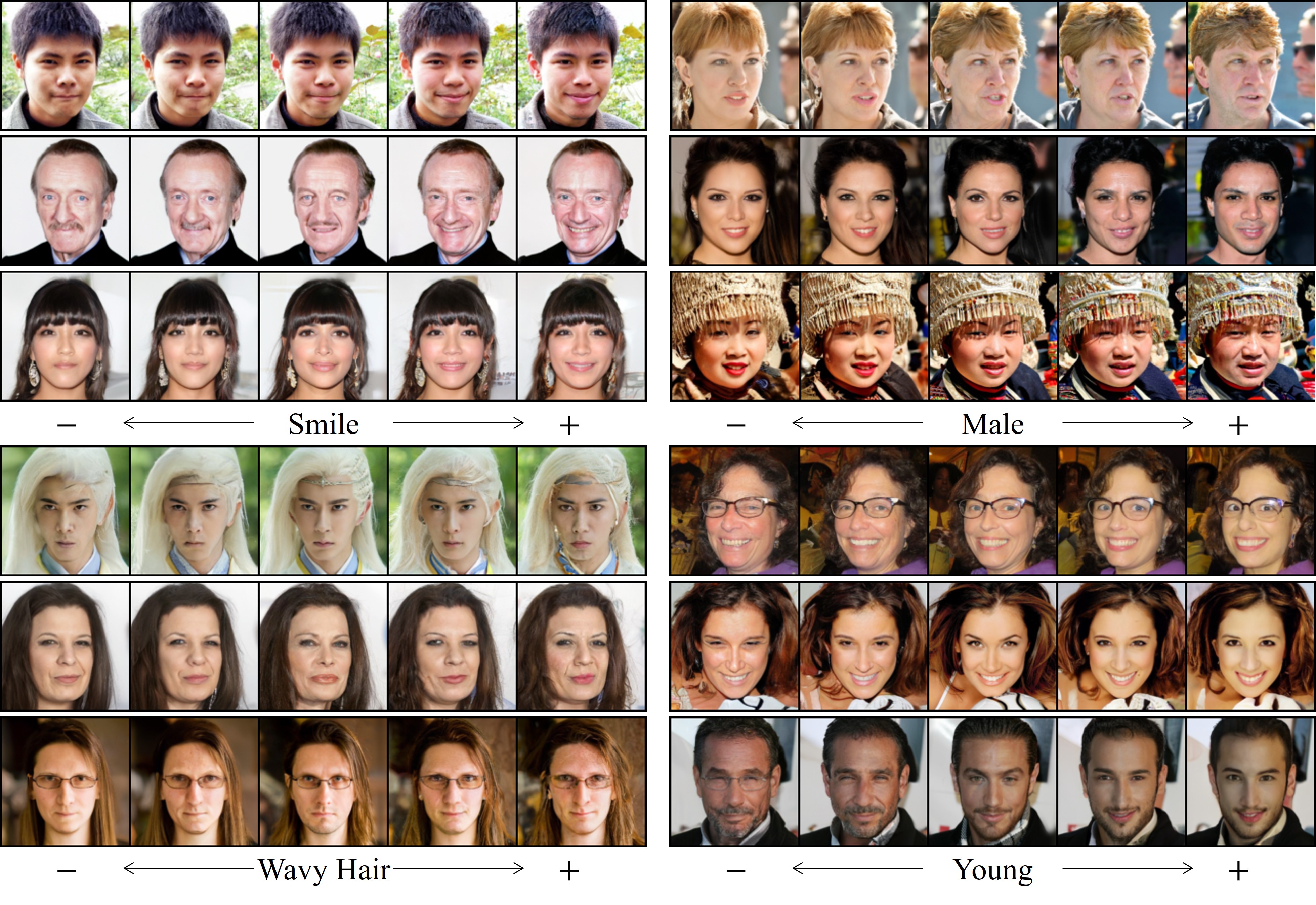}
         \caption{Smooth traversals in the direction of attribute manipulation. All the middle images are the original images.}
         \vspace{5mm}
     \end{subfigure}
     \begin{subfigure}[b]{\textwidth}
         \centering
         \includegraphics[width=\textwidth]{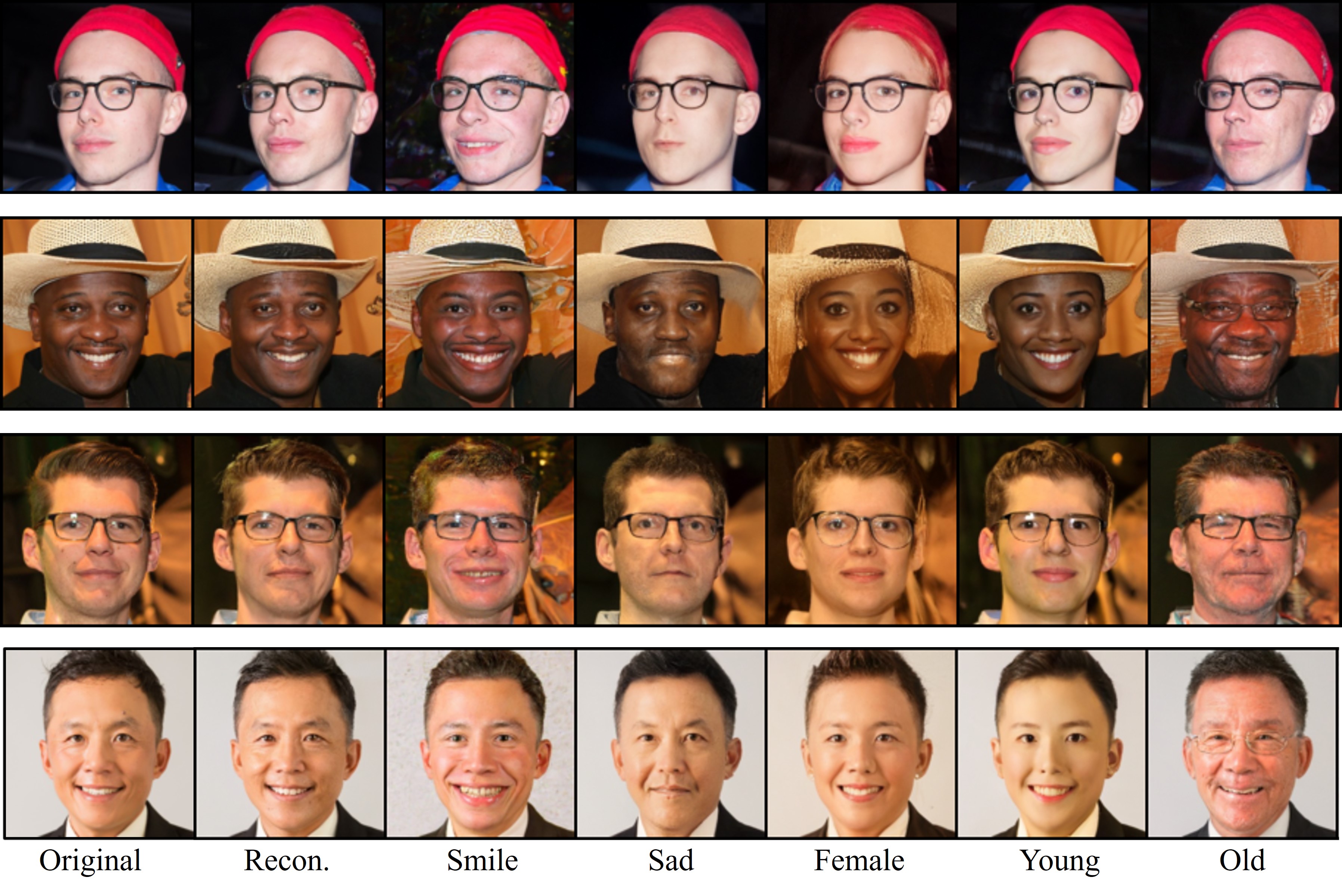}
         \caption{Multiple attribute manipulation on a single image.}
     \end{subfigure}
    \caption{Attribute manipulation using a linear classifier on FFHQ and CelebA-HQ.}
    \label{fig:app_att}
\end{figure}

\begin{figure}[]
     \centering
     \begin{subfigure}[b]{\textwidth}
         \centering
         \includegraphics[width=\textwidth]{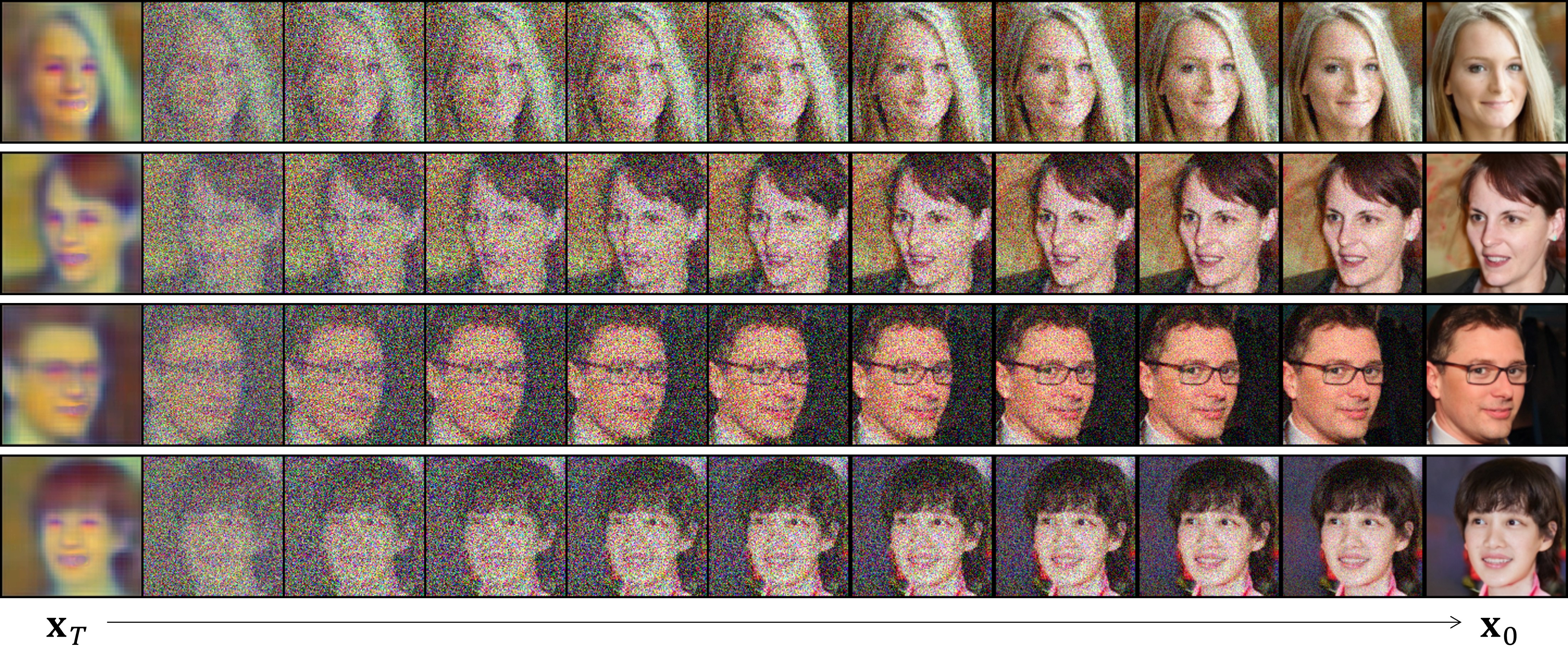}
         \caption{Sampling trajectory of DBAE trained on FFHQ.}
         \vspace{5mm}
         \label{fig:trajectory_ffhq}
     \end{subfigure}
     \begin{subfigure}[b]{\textwidth}
         \centering
         \includegraphics[width=\textwidth]{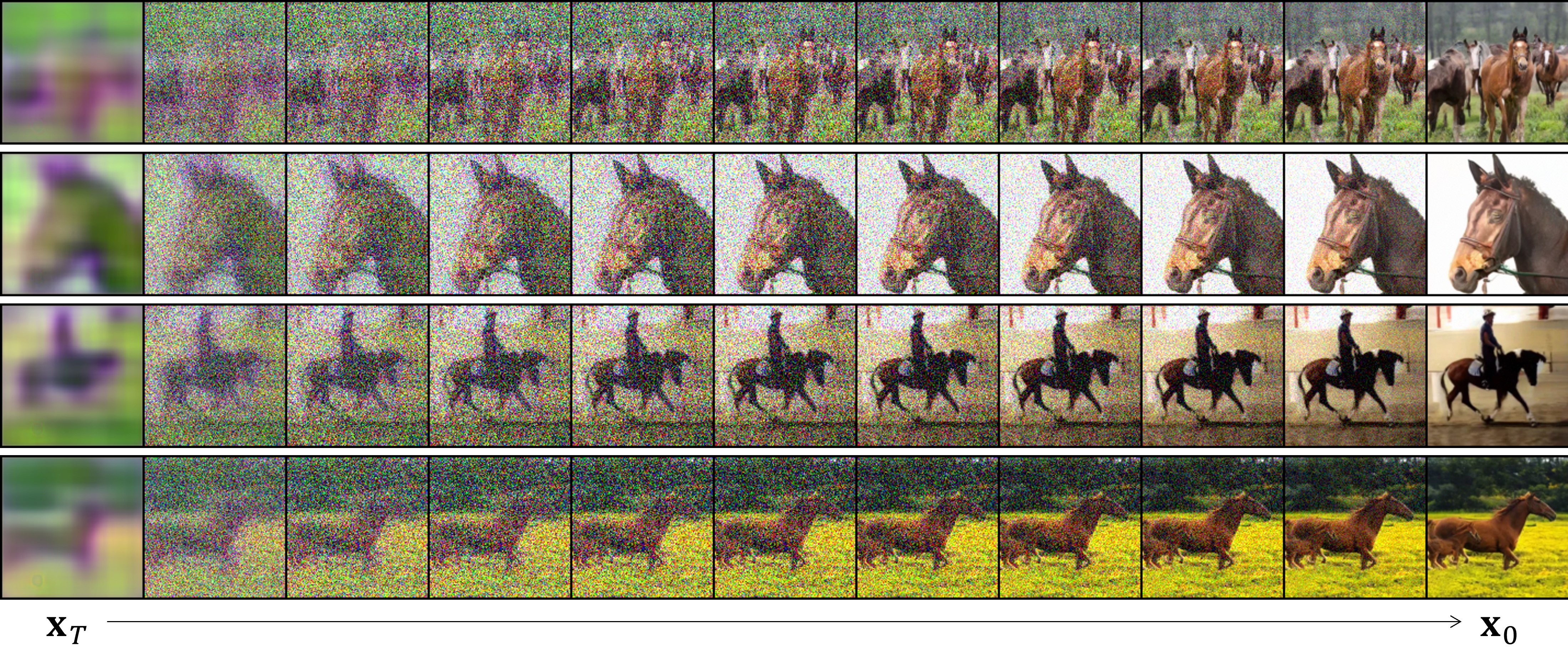}
         \caption{Sampling trajectory of DBAE trained on Horse.}
         \vspace{5mm}
         \label{fig:trajectory_horse}
     \end{subfigure}
     \vspace{5mm}
     \begin{subfigure}[b]{\textwidth}
         \centering
         \includegraphics[width=\textwidth]{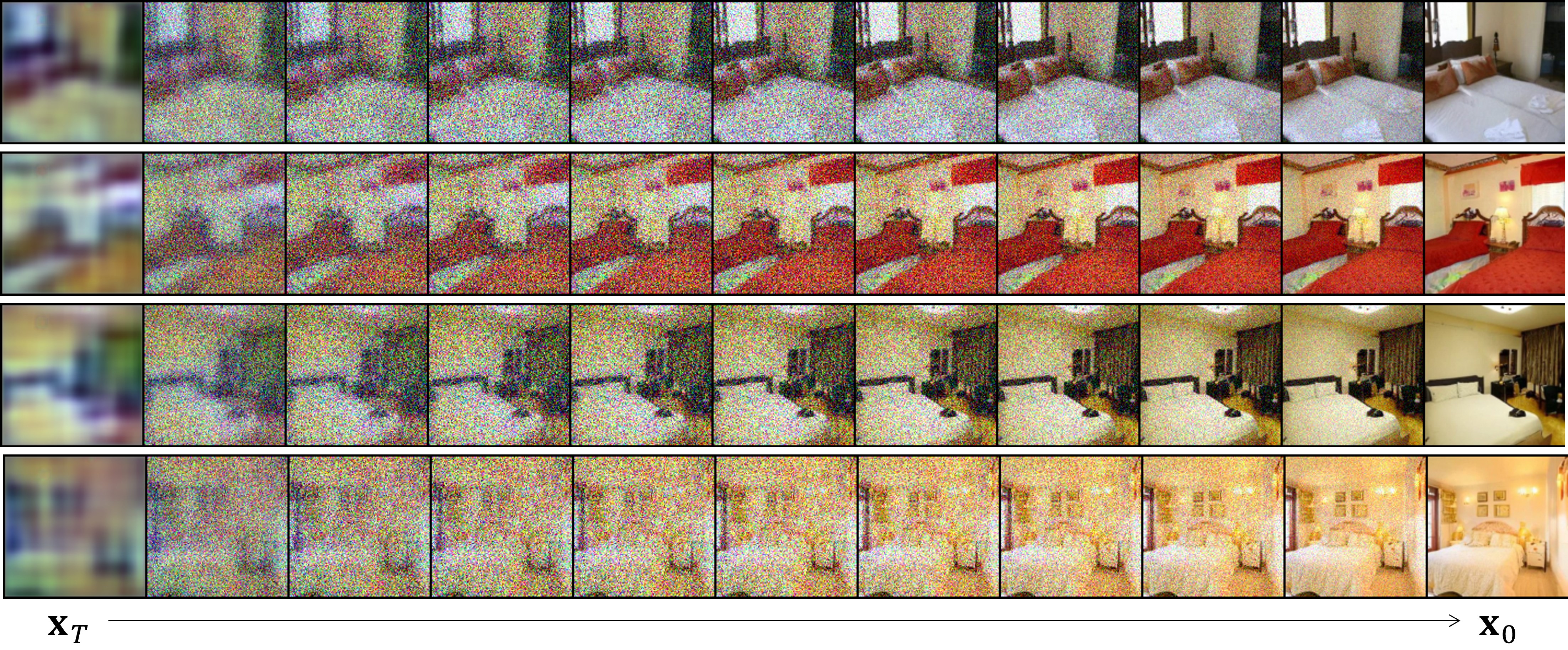}
         \caption{Sampling trajectory of DBAE trained on Bedroom.}
         \label{fig:trajectory_bedroom}
     \end{subfigure}
    \caption{Stochastic sampling trajectory of DBAE trained on various datasets.}
    \label{fig:trajectory_whole}
\end{figure}

\begin{figure}[]
     \centering
     \begin{subfigure}[b]{\textwidth}
         \centering
         \includegraphics[width=\textwidth]{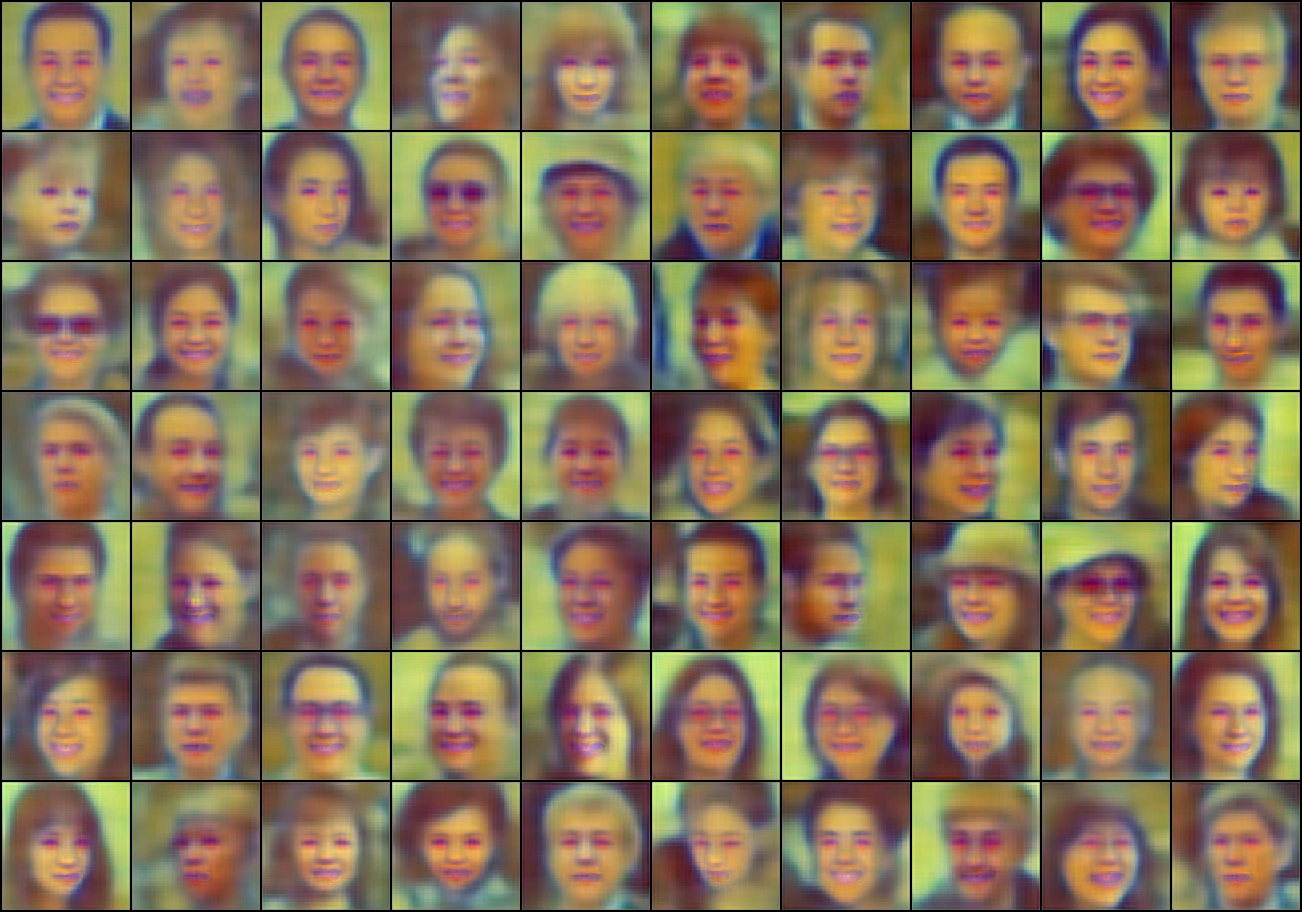}
         \caption{Generated endpoints $\rvx_T$}
     \end{subfigure}
     \begin{subfigure}[b]{\textwidth}
         \centering
         \includegraphics[width=\textwidth]{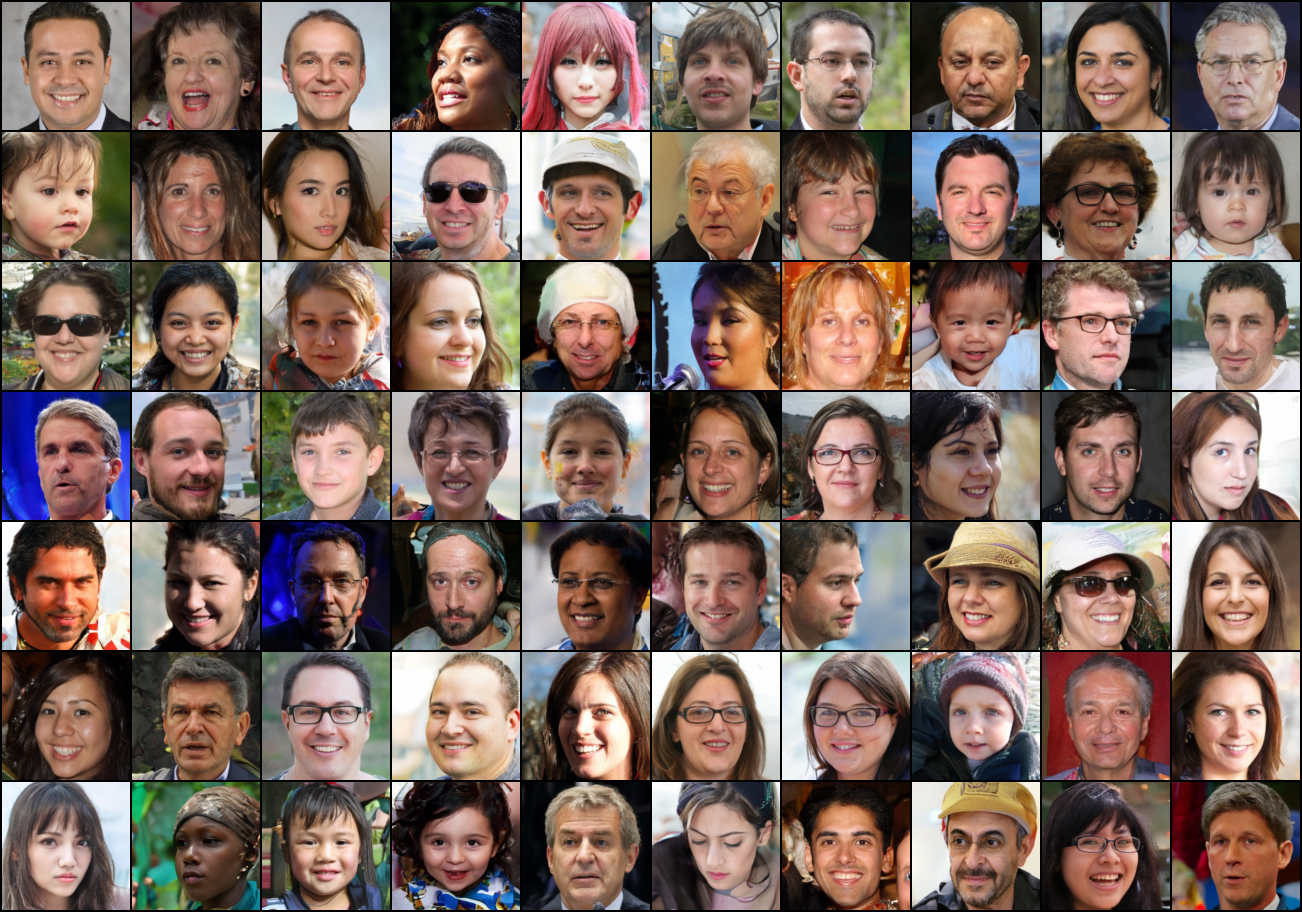}
         \caption{Generated images $\rvx_0$}
     \end{subfigure}
    \caption{Uncurated generated samples with corresponding endpoints from DBAE trained on FFHQ with unconditional generation.}
    \label{fig:uncurated_ffhq}
\end{figure}

\begin{figure}[]
     \centering
     \begin{subfigure}[b]{\textwidth}
         \centering
         \includegraphics[width=\textwidth]{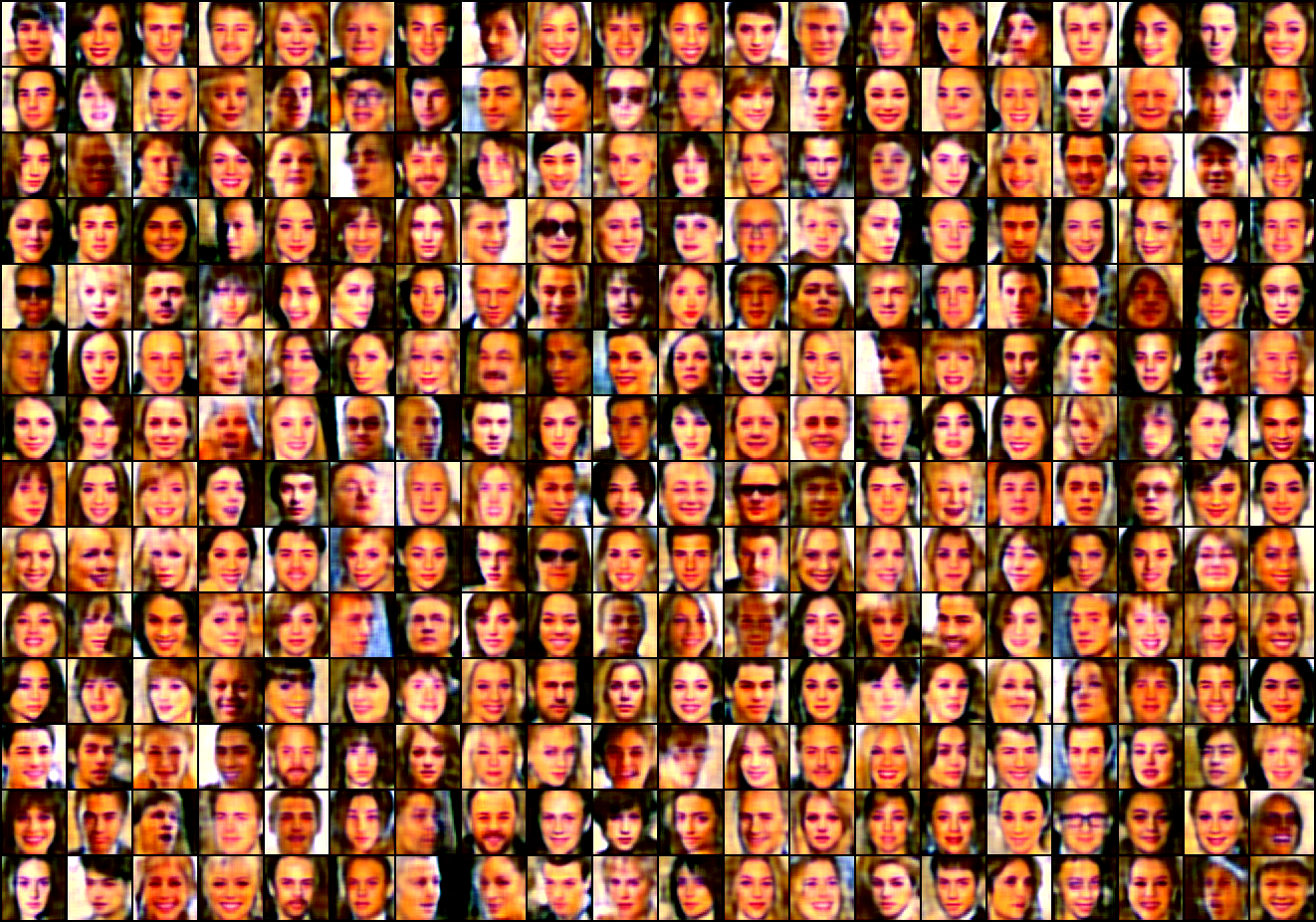}
         \caption{Generated endpoints $\rvx_T$}
     \end{subfigure}
     \begin{subfigure}[b]{\textwidth}
         \centering
         \includegraphics[width=\textwidth]{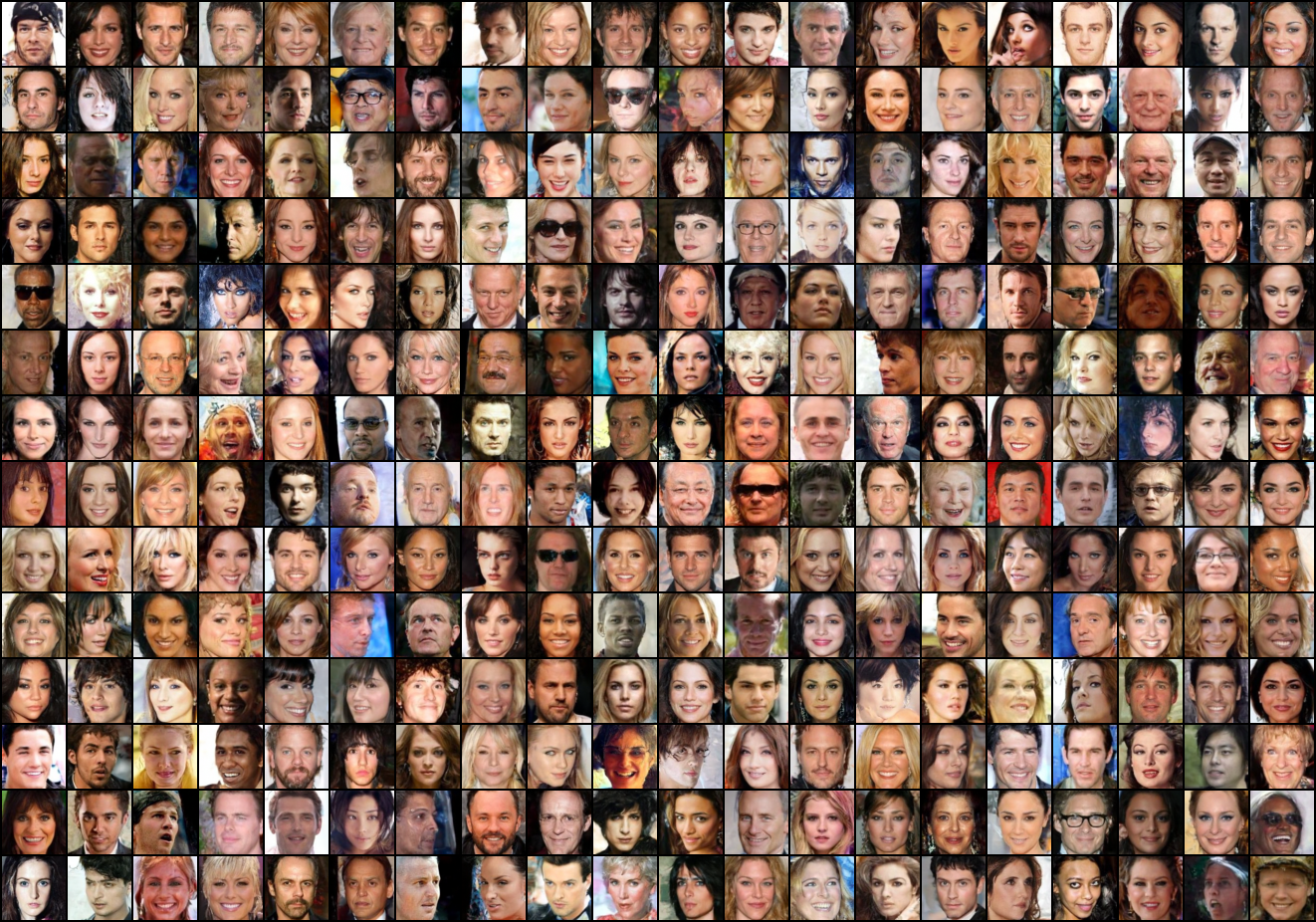}
         \caption{Generated images $\rvx_0$}
     \end{subfigure}
    \caption{Uncurated generated samples with corresponding endpoints from DBAE trained on CelebA with unconditional generation.}
    \label{fig:uncurated_celeba}
\end{figure}

\end{document}